\documentclass[11pt, letterpaper, logo, onecolumn, copyright]{template_from_google}

\usepackage[T1]{fontenc}

\usepackage{bm} %
\usepackage{nicefrac} %

\usepackage{booktabs} %
\usepackage{threeparttable}
\usepackage{enumitem} %
\usepackage{fancyhdr}
\usepackage{multirow} %
\usepackage{tabularx}
\usepackage{multicol}
\usepackage{array}
\usepackage{longtable}
\usepackage{adjustbox}

\usepackage{wrapfig} %
\usepackage{tocloft} %
\usepackage{fancybox}
\usepackage{authblk} %
\usepackage[bottom]{footmisc} %

\usepackage{algorithm}
\usepackage[noend]{algpseudocode}

\usepackage{graphicx}
\usepackage{caption}
\usepackage{subcaption}
\providecommand{\safeincludegraphics}[2][]{\IfFileExists{#2}{\includegraphics[#1]{#2}}{\fbox{\parbox{0.9\linewidth}{\centering Missing figure: \texttt{\detokenize{#2}}}}}}

\usepackage[dvipsnames]{xcolor}

\usepackage{xspace} %
\usepackage{url}

\usepackage[authoryear, sort&compress, round]{natbib}

\usepackage{hyperref}

\definecolor{darkblue}{rgb}{0.0, 0.0, 0.6}
\definecolor{darkred}{rgb}{0.7, 0.0, 0.0}
\hypersetup{
  pdffitwindow=true,
  pdfstartview={FitH},
  pdfnewwindow=true,
  colorlinks,
  linktocpage=true,
  linkcolor=darkred,
  urlcolor=darkblue,
  citecolor=darkblue
}

\usepackage{amsmath,amsfonts,bm}

\def\eqref#1{(\ref{#1})}

\def\1{\bm{1}}

\DeclareMathAlphabet{\mathsfit}{\encodingdefault}{\sfdefault}{m}{sl}
\SetMathAlphabet{\mathsfit}{bold}{\encodingdefault}{\sfdefault}{bx}{n}

\usepackage{enumitem}
\usepackage{multirow}
\usepackage{mathtools}
\usepackage{amssymb}
\usepackage{amsthm}
\usepackage{wrapfig}
\usepackage{placeins} %

\usepackage{algorithm}
\usepackage[noend]{algpseudocode}
\usepackage{listings}
\usepackage{pmboxdraw}

\usepackage{tcolorbox}
\tcbuselibrary{breakable,skins,listings}

\usepackage{mdframed}
\usepackage{xspace}

\usepackage{pifont}%
\newcommand{\xmark}{\ding{55}}%

\usepackage{tikz}
\usetikzlibrary{patterns.meta}

\definecolor{sglBlue}{RGB}{0,102,204}
\definecolor{sglRed}{RGB}{222,68,68}
\definecolor{sglGreen}{RGB}{118,205,38}
\definecolor{sglOrange}{RGB}{252,184,39}

\definecolor{SharedPrefixColor}{HTML}{8A9A5B}
\definecolor{SegmentColor}{HTML}{7393B3}
\colorlet{CachedSuffixColor}{gray!70}

\colorlet{NewSuffixAFillColor}{sglOrange!80}
\colorlet{NewSuffixBFillColor}{sglGreen!80}

\colorlet{GuideLineColor}{black!45}
\colorlet{TreeEdgeColor}{black!60}
\definecolor{BoxBorderColor}{RGB}{96,96,96}
\colorlet{NeutralArrowColor}{gray!70}
\definecolor{RejectEdgeColor}{RGB}{160,160,160}
\colorlet{PrunedMarkColor}{sglRed}
\definecolor{InactiveNodeFillColor}{RGB}{244,244,244}
\colorlet{FigureTextColor}{black}
\colorlet{FigureBackgroundColor}{white}

\def\rowh{0.45}    %
\def\barw{4.6}     %
\def\cutTwo{1.5}   %
\def\cutOne{3.2}   %
\def\rY{-1.15}     %
\def\rZ{-2.30}     %

\tikzset{
  paperfig/.style={x=1cm,y=1cm,>=stealth,font=\small,
                   line join=round,line cap=round},
  annot/.style={fill=FigureBackgroundColor,inner sep=1pt},
  panel/.style={align=center},
  maparrow/.style={->,thick,rounded corners=3pt,
                   shorten <=1pt,shorten >=1pt},
  seqbar/.style={draw=BoxBorderColor,rounded corners=1pt},
  seqlabel/.style={anchor=east,align=right},
  ticklabel/.style={anchor=south,inner sep=1.5pt},
  cutline/.style={dashed,GuideLineColor},
  prefixfill/.style={pattern={Lines[angle=15,distance=2.5pt,line width=0.5pt]},
                     pattern color=SharedPrefixColor},
  segmentfill/.style={pattern={Lines[angle=30,distance=2.5pt,line width=0.5pt]},
                      pattern color=SegmentColor},
  cachedfill/.style={pattern={Lines[angle=45,distance=2.5pt,line width=0.5pt]},
                     pattern color=CachedSuffixColor},
  newAfill/.style={fill=NewSuffixAFillColor},
  newBfill/.style={fill=NewSuffixBFillColor},
  plate/.style={fill=FigureBackgroundColor}, %
  treenode/.style={rounded corners=3pt,draw=BoxBorderColor,
                   minimum size=0.62cm,inner sep=1.5pt,align=center},
  rejectednode/.style={dashed},                %
  acc/.style={->,thick,FigureTextColor},
  rej/.style={->,thick,RejectEdgeColor,dashed},
  edgelabel/.style={text=FigureTextColor,
                    fill=FigureBackgroundColor,inner sep=1.5pt},
  kvbox/.style={draw=BoxBorderColor,rounded corners=1pt},
  kvlink/.style={TreeEdgeColor,thick},
  kvpoint/.style={circle,draw=BoxBorderColor,fill=FigureBackgroundColor,
                  minimum size=0.18cm,inner sep=0pt},
  kvleaf/.style={circle,draw=BoxBorderColor,fill=BoxBorderColor,
                 minimum size=0.14cm,inner sep=0pt}
}

\newcommand{\drawSequenceBars}{%
  \fill[plate]       (0,0) rectangle ++(\barw,\rowh);
  \fill[prefixfill]  (0,0) rectangle ++(\cutTwo,\rowh);
  \fill[segmentfill] (\cutTwo,0) rectangle ++({\cutOne-\cutTwo},\rowh);
  \fill[cachedfill]  (\cutOne,0) rectangle ++({\barw-\cutOne},\rowh);
  \draw[seqbar]      (0,0) rectangle ++(\barw,\rowh);
  \fill[plate]       (0,\rY) rectangle ++(\barw,\rowh);
  \fill[prefixfill]  (0,\rY) rectangle ++(\cutTwo,\rowh);
  \fill[segmentfill] (\cutTwo,\rY) rectangle ++({\cutOne-\cutTwo},\rowh);
  \fill[newAfill]    (\cutOne,\rY) rectangle ++({\barw-\cutOne},\rowh);
  \draw[seqbar]      (0,\rY) rectangle ++(\barw,\rowh);
  \fill[plate]       (0,\rZ) rectangle ++(\barw,\rowh);
  \fill[prefixfill]  (0,\rZ) rectangle ++(\cutTwo,\rowh);
  \fill[newBfill]    (\cutTwo,\rZ) rectangle ++({\barw-\cutTwo},\rowh);
  \draw[seqbar]      (0,\rZ) rectangle ++(\barw,\rowh);
  \draw[cutline] (\cutTwo,\rZ-0.12) -- (\cutTwo,\rowh+0.1);
  \draw[cutline] (\cutOne,\rY-0.12) -- (\cutOne,\rowh+0.1);
  \node[ticklabel] at (0,\rowh)         {$0$};
  \node[ticklabel] at (\cutTwo,\rowh)   {$c_2$};
  \node[ticklabel] at (\cutOne,\rowh)   {$c_1$};
  \node[ticklabel] at (\barw-0.1,\rowh) {$T$};
}

\newcommand{\drawSequenceReuseBrackets}{%
  \draw[thick,SharedPrefixColor]
    (0,\rZ-0.30) -- (0,\rZ-0.15) -- (\cutTwo,\rZ-0.15) -- (\cutTwo,\rZ-0.30);
  \node[annot,text=SharedPrefixColor] (shared-prefix-label)
    at ({\cutTwo/2},\rZ-0.50) {shared prefix $\mathbf{x}_{0:c_2}$};
  \draw[thick,SegmentColor]
    (\cutTwo,\rY-0.30) -- (\cutTwo,\rY-0.15)
    -- (\cutOne,\rY-0.15) -- (\cutOne,\rY-0.30);
  \node[annot,text=SegmentColor] (shared-segment-label)
    at ({(\cutTwo+\cutOne)/2+0.6},\rY-0.50) {shared segment $\mathbf{x}_{c_2+1:c_1}$};
}

\lstdefinestyle{prefetchtree}{
  basicstyle=\small\ttfamily,
  inputencoding=utf8,
  breaklines=true,
  breakindent=1em,
  breakautoindent=true,
  breakatwhitespace=false,
  postbreak=\mbox{\textcolor{GuideLineColor}{$\hookrightarrow$}\space},
  columns=fullflexible,
  keepspaces=true,
  showstringspaces=false,
  tabsize=2,
  morecomment=[l][\bfseries\color{sglBlue}]{===},
  moredelim=**[s][\color{sglGreen}]{[accept}{]},
  moredelim=**[s][\color{sglRed}]{[reject}{]},
  literate={├}{{\textSFviii}}1
           {└}{{\textSFii}}1
           {─}{{\textSFx}}1
           {│}{{\textSFxi}}1
           {→}{{\ensuremath{\rightarrow}}}1
}

\colorlet{TreeBoxFrameColor}{black!75}
\colorlet{TreeBoxBackgroundColor}{gray!7}
\colorlet{TreeBoxTitleBackgroundColor}{sglBlue!10}
\colorlet{TreeBoxTitleColor}{sglBlue!75!black}

\newcounter{prefetchbox}

\newcounter{prefetchpanel}[figure]

\makeatletter
\newcommand{\PrefetchFigureCaption}[1]{%
  \par\begingroup
    \@parboxrestore\normalsize
    \@makecaption{\figurename~\thefigure}{#1}\par
  \endgroup
}
\makeatother

\newcommand{\TreeBox}[3][]{%
  \IfFileExists{#3}{%
    \tcbinputlisting{%
      enhanced,
      breakable,
      listing only,
      boxrule=0.4pt,
      titlerule=0.4pt,
      colframe=TreeBoxFrameColor,
      colback=TreeBoxBackgroundColor,
      colbacktitle=TreeBoxTitleBackgroundColor,
      coltitle=TreeBoxTitleColor,
      fonttitle=\normalsize\ttfamily\bfseries,
      title={#2},
      title after break={#2 (continued)},
      before skip=4pt,
      after skip=4pt,
      boxsep=1pt,
      left=1pt,
      right=1pt,
      top=1pt,
      bottom=1pt,
      listing file={#3},
      listing options={style=prefetchtree,basicstyle=\fontsize{8}{9}\selectfont\ttfamily,#1}%
    }%
  }{%
    \begin{tcolorbox}[
      enhanced,
      breakable,
      boxrule=0.4pt,
      titlerule=0.4pt,
      colframe=TreeBoxFrameColor,
      colback=TreeBoxBackgroundColor,
      colbacktitle=TreeBoxTitleBackgroundColor,
      coltitle=TreeBoxTitleColor,
      fonttitle=\normalsize\ttfamily\bfseries,
      title={#2}
    ]
      \textit{File not found: \texttt{\detokenize{#3}}. Check the path and capitalization.}
    \end{tcolorbox}%
  }%
}

\newcommand{\TreeTrajectory}[3][]{%
  \IfFileExists{#3}{%
    \lstinputlisting[
      style=prefetchtree,
      title={\textbf{#2}},
      frame=single,
      rulecolor=\color{black!75},
      backgroundcolor=\color{gray!7},
      framesep=2pt,
      framexleftmargin=2pt,
      framexrightmargin=2pt,
      framextopmargin=2pt,
      framexbottommargin=2pt,
      #1
    ]{#3}%
  }{%
    \begin{tcolorbox}[
      enhanced,
      colframe=black!75,
      colback=gray!7,
      title={#2}
    ]
      \textit{File not found: \texttt{\detokenize{#3}}. Check the path and capitalization.}
    \end{tcolorbox}%
  }%
}

\mdfdefinestyle{pretty}{
  backgroundcolor=blue!5,
  linecolor=blue!50!black,
  linewidth=1pt,
  roundcorner=6pt,
  innertopmargin=6pt,
  innerbottommargin=6pt,
  innerleftmargin=8pt,
  innerrightmargin=8pt,
  skipabove=10pt,
  skipbelow=10pt
}

\newtheorem{lemma}{Lemma}
\newtheorem{condition}{Condition}
\newtheorem{proposition}{Proposition}
\newtheorem{definition}{Definition}

\newtheorem{example}{Example}

\newcommand{\framework}{{PreSTO}\xspace}                   
\newcommand{\method}{{PreSTO}-PowerMH\xspace}          %
\newcommand{\methodEnt}{{PreSTO}-EntropyCut\xspace}
\newcommand{\methodMTM}{{PreSTO}-MultiTryMH\xspace}

\usepackage{makecell}
\usepackage{adjustbox}

\let\cite\citep

\renewcommand{\headrulewidth}{0.4pt}

\reportnumber{} %

\renewcommand{\today}{}

\title{PreSTO: Predictive Subtree Prefetching for Fast LLM Power Sampling}

\author[1]{Nan Jiang}
\author[2]{Panagiotis Theodoropoulos}
\author[3]{Wei Deng}

\affil[1]{University of Texas at El Paso}
\affil[2]{Georgia Institute of Technology}
\affil[3]{Morgan Stanley}

\correspondingauthor{{njiang@utep.edu}}

\begin{abstract}

Power sampling improves large language model (LLM) reasoning at test time by sampling from a sharpened version of the base model's distribution.
Its sampler, PowerMH, is slow because its Metropolis--Hastings (MH) updates are sequential: each accept/reject decision sets the state from which the next proposal is generated.
Prior work reduces this cost by changing the sampler or the model (e.g., sequential Monte Carlo, approximate lookahead, distillation).
We propose Predictive SubTree Prefetching (PreSTO), which accelerates PowerMH by prefetching future proposals and provably leaves its MH transition kernel unchanged.
PreSTO views future MH outcomes as a tree of accept/reject decisions: in each model call, it generates the proposals of a \textit{prefetched subtree} of this tree in one batch, then follows the realized path with the original MH acceptance rule.
We also observe that MH reaches only a sparse, nearly unary part of each prefetched subtree, we further design a \textit{predictive traversal rule} that chooses the subtree to maximize the predicted number of completed transitions.
Across seven datasets and five base LLMs, PreSTO achieves median speedups of $1.71\times$, $1.91\times$, and $2.33\times$ on PowerMH and its follow-ups EntropyCut and MultiTryMH, respectively; on the first two, it makes a median of $3.9\times$ fewer model calls, each only $2.1\times$ longer.
On LCB V6 with Qwen3.5-9B, the predictive traversal rule is a further $1.25\times$--$1.37\times$ faster than the default breadth-first traversal.

\begingroup\centering
\textbf{Code:} \url{https://github.com/nan-jiang-group/PreSTO}\\
\textbf{Project Page:} \url{https://nan-jiang-group.github.io/PreSTO}\par
\vspace{-1em}
\endgroup

\end{abstract}

\begin{document}

\maketitle

\section{Introduction}

\begin{figure}[!b]
    \centering
    \includegraphics[width=\linewidth]{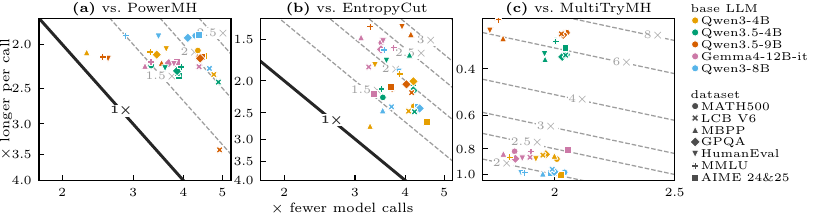}
    \vspace{-2em}
    \caption{\small Across 105 sampler--LLM--dataset settings, \framework (ours) speeds up PowerMH, EntropyCut, and MultiTryMH by a median of \(1.71\times\), \(1.91\times\), and \(2.33\times\), respectively.
    Each point is one setting, and faster settings lie toward the upper right.}
    \label{fig:teaser}
\end{figure}

Large language models (LLMs) perform strongly on reasoning tasks such as competition mathematics, code generation, and graduate-level science questions~\citep{DBLP:conf/nips/HendrycksBKABTS21,journal/science/alphacode22,conf/colm/rein2024gpqa}.
These capabilities are commonly improved by reinforcement learning (RL) post-training with task-specific rewards~\citep{guo2025deepseekr1,hu2025openreasonerzero}.
Recent work suggests that RL post-training largely sharpens the base model's distribution~\citep{he2025rewarding,shao2026spurious,DBLP:conf/nips/ChenLZWYSH25}, and that sampling directly from a sharpened distribution, known as \emph{power sampling}, can perform comparably to RL post-training without additional training~\citep{karan2026reasoning}.

However, PowerMH, the Metropolis--Hastings (MH) sampler for power sampling, is expensive~\citep{ji2026scalable}: on MATH500, it generates up to $8.8\times$ as many tokens as standard sampling~\citep{karan2026reasoning}, and on LCB V6 with Qwen3.5-9B, it takes $6.5\times$--$18.4\times$ as long per prompt as single-sample low-temperature sampling (Case Study~8, Appendix~\ref{apx:scaling-baselines}).
Follow-up work improves it in four directions:
(1) \emph{Alternative samplers} replace serial MH with sequential Monte Carlo (SMC) or replica exchange~\citep{DBLP:arxiv/power-smc,he2026recipes}.
(2) \emph{Efficient approximations} of the power target reduce inference cost~\citep{ji2026scalable,nguyen2026futurevalue}.
(3) \emph{Proposal refinements} choose cut indices by token-level entropy to improve mixing~\citep{entropycut,powermh-mtm}.
(4) \emph{Objective or training changes} target the answer marginal~\citep{arzhantsev2026selfconsistency} or fine-tune the model on power-sampled responses~\citep{tomihari2026power}.

The main bottleneck of PowerMH is the sequential dependence between MH transitions: in a standard implementation, each proposal must be accepted or rejected before the next is generated from the resulting state.
We draw on prefetching in classical MCMC~\citep{Brockwell01032006,DBLP:conf/uai/AngelinoKWSA14}: future MH transitions form a tree of accept/reject outcomes, and each sequential run traces one root-to-leaf path.
If the proposals on this path are available in advance, several transitions complete without another model call.

We present Predictive SubTree Prefetching (\framework), \textbf{a framework for accelerating MH-based sharpened sampling by completing several MH transitions per model call}.
In each round, \framework builds a \emph{prefetched subtree}, generates its feasible proposals in one batched model call, and follows the realized path with the original MH acceptance rule.
We prove that \method leaves the MH transition kernel of PowerMH unchanged (Proposition~\ref{thm:spec_mh_main}).
We further observe that, under the power-sharpened target, MH reaches only a sparse, nearly unary part of each prefetched subtree (Figure~\ref{fig:epsilon-certainty-datasets} and Appendix~\ref{subsec:edge-certainty}), so a breadth-first traversal spends much of the budget on nodes that MH rarely visits.
Motivated by this, we propose a improved \emph{predictive traversal rule} that allocates the prefetch budget to maximize the expected number of MH transitions completed per call.
\framework trades fewer serial calls for a costlier batched call; with prefix-cache reuse across requests, total latency falls whenever the saved calls outweigh the extra cost per batch.

In experiments, we apply \framework to PowerMH, EntropyCut, and MultiTryMH on seven datasets and five base LLMs.
\framework is never slower than its baseline in all 105 settings, with median speedups of $1.71\times$ over PowerMH, $1.91\times$ over EntropyCut, and $2.33\times$ over MultiTryMH, and maximum speedups of $2.29\times$ and $2.85\times$ over the first two (Figure~\ref{fig:teaser}).
Our case studies show that
\textbf{(1)} the speedup comes from fewer model calls rather than cheaper ones: over PowerMH, \method makes \(1.94\times\)--\(4.92\times\) fewer calls, each taking \(1.90\times\)--\(3.42\times\) as long (Figure~\ref{fig:teaser});
\textbf{(2)} the log-likelihood and confidence of \framework outputs pass paired equivalence tests against the baselines in 29 of 38 comparisons, which supports that \framework leaves the MH transition kernel unchanged (Figure~\ref{fig:likelihood-confidence});
and
\textbf{(3)} the predictive traversal rule is \(1.25\times\)--\(1.37\times\) faster than the default breadth-first search (BFS) traversal at all five prefetch budgets tested on LCB V6 with Qwen3.5-9B (Figure~\ref{fig:traversal-rules}).

\section{Preliminaries}

\medskip
\noindent\textbf{Metropolis--Hastings (MH)} is an MCMC algorithm for sampling from a distribution known only up to a normalizing constant~\citep{metropolis1953equation,hastings1970monte}.
MH constructs a Markov chain $\mathbf{x}^{(0)}\to\cdots\to\mathbf{x}^{(k)}\to\cdots\to\mathbf{x}^{(K)}$ to sample from a target distribution $\pi(\mathbf{x})$.
Given the current state $\mathbf{x}^{(k)}$, a candidate $\mathbf{y}$ is drawn from a proposal distribution $q(\cdot \mid \mathbf{x}^{(k)})$ and is accepted with probability
\begin{equation}\label{eq:accept-ratio}
  A(\mathbf{y}, \mathbf{x}^{(k)})
  = \min\!\left\{1,
      \frac{\pi(\mathbf{y})\, q(\mathbf{x}^{(k)} \mid \mathbf{y})}
           {\pi(\mathbf{x}^{(k)})\, q(\mathbf{y}\mid \mathbf{x}^{(k)})}
    \right\}.
\end{equation}
If the proposal is accepted, the chain moves to $\mathbf{x}^{(k+1)}\gets\mathbf{y}$; otherwise,  $\mathbf{x}^{(k+1)}\gets \mathbf{x}^{(k)}$.

\noindent\textbf{Power Sharpening for LLMs.} Let $\mathbf{x}_0$ denote a given prompt and $\mathbf{x}_{0:n} = (\mathbf{x}_0, x_1, \dots, x_n)$ denote the prompt followed by $n$ generated tokens.
Let $p_0$ denote the distribution of the base LLM.
Recently, \citet{karan2026reasoning} proposed to sample from a \emph{power-sharpened} distribution:
\begin{align} \label{eq:pi-alpha}
  \pi_\alpha(\mathbf{x}_{1:n}\mid \mathbf{x}_0)
  \;\propto\; p_0(\mathbf{x}_{1:n}\mid \mathbf{x}_0)^{\alpha}.
\end{align}
When $\alpha > 1$, the exponent sharpens the base distribution to amplify high-likelihood continuations and suppress low-likelihood ones~\citep{li2025blinkofaneyetheory}.
Empirical results~\citep{karan2026reasoning,ji2026scalable} show that power sharpening can achieve performance competitive with RL-based post-training methods, such as Group Relative Policy Optimization (GRPO)~\citep{depseekmath-grpo}.

Exact sampling from $\pi_\alpha$ in~\eqref{eq:pi-alpha} is intractable, as it requires summing over future continuations.
\citet{karan2026reasoning} propose PowerMH, an MH-based sampler that proceeds through a sequence of intermediate targets block by block:
\begin{equation}\label{eq:stage-sequence}
  \pi_\alpha(\mathbf{x}_{1:B}\mid\mathbf{x}_0)
  \rightarrow
  \cdots
  \rightarrow
  \pi_\alpha(\mathbf{x}_{1:\ell B}\mid\mathbf{x}_0)
  \rightarrow
  \cdots
  \rightarrow
  \pi_\alpha(\mathbf{x}_{1:n}\mid\mathbf{x}_0),
\end{equation}
where $B$ is the block size and the stage-$\ell$ target raises the base model's distribution over the first $\ell B$ generated tokens to the power $\alpha$,
\begin{equation}\label{eq:stage-target}
  \pi_\alpha(\mathbf{x}_{1:\ell B}\mid\mathbf{x}_0)
  \;\propto\;
  p_0(\mathbf{x}_{1:\ell B}\mid\mathbf{x}_0)^{\alpha},
\end{equation}
which equals~\eqref{eq:pi-alpha} at the last stage, $\ell B=n$.
At the $\ell$-th stage, PowerMH runs MH steps targeting the distribution $\pi_\alpha(\mathbf{x}_{1:\ell B}\mid\mathbf{x}_0)$ by first extending the sample $\mathbf{x}_{0:(\ell-1)B}$ from the previous stage by one block to produce $\mathbf{x}_{1:\ell B}$, and then refining it with $K$ MH steps.

\section{Methodology}

\begin{figure}[!t]
    \centering
    \includegraphics[width=\linewidth]{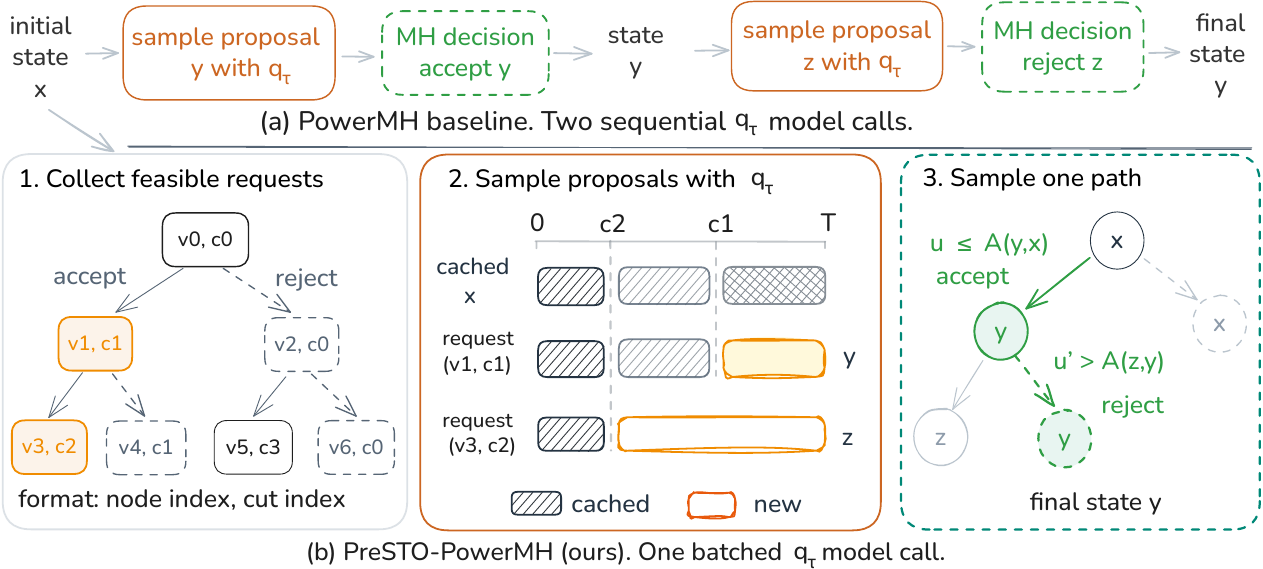}
    \caption{Our subtree prefetching (\framework) accelerates PowerMH for transitions $\mathbf{x}\to\mathbf{y}\to\mathbf{y}$. \textbf{(a)}  PowerMH performs two sequential rounds of proposal generation with $q_\tau$ followed by an MH accept/reject decision.
    \textbf{(b)} Our method collects feasible requests by traversing the prefetching tree, generates proposals $(\mathbf{y},\mathbf{z})$ in one batched call, and then samples a path over the subtree. 
    }
    \label{fig:pipeline}
\end{figure}

\subsection{Motivation}

Prior studies~\citep{ji2026scalable,DBLP:arxiv/power-smc} show that PowerMH incurs substantial inference-time latency.
On the MATH500 dataset, \citet{karan2026reasoning} report that PowerMH generates up to $8.8\times$ as many tokens as standard sampling.
In our measurements on LCB V6 with Qwen3.5-9B (Case Study~8, Appendix~\ref{apx:scaling-baselines}), PowerMH with $K=10$, $20$, and $30$ MH steps takes $6.5\times$, $12.2\times$, and $18.4\times$ as long per prompt as low-temperature sampling, which draws a single sample.
This latency arises from the \textit{inherently sequential nature of MH}: the proposal at step $k{+}1$ can be generated only after the accept/reject outcome at step $k$ determines the next state.

We first describe the suffix-resampling procedure in the PowerMH sampler.
For notational clarity, let $L=(\ell-1)B$ and $T=\ell B$ be the sequence lengths before and after the $\ell$-th stage.
Given the sequence $\mathbf{x}_{0:L}$ from the previous stage, the sampler first extends it to $\mathbf{x}_{0:T}$ by sampling a continuation from the proposal LLM, and this extended sequence is the initial MH state $\mathbf{x}^{(0)}=\mathbf{x}_{0:T}$.
At step $k$ with MH state $\mathbf{x}^{(k)}$, the sampler draws a cut index $c\sim\mathrm{Unif}\{0,\ldots,T-1\}$, and resamples the suffix from the proposal LLM to construct a candidate $\mathbf{y}=\bigl(\mathbf{x}^{(k)}_{0:c},\mathbf{y}_{c+1:T}\bigr)$:
\begin{equation}\label{eq:suffix-q}
\begin{aligned}
y_t \sim q_{\tau}\!\left( \cdot \mid \mathbf{x}^{(k)}_{0:c},\mathbf{y}_{c+1:t-1} \right),
& \qquad\text{ for }  t= c+1,\ldots, T,
\end{aligned}
\end{equation}
where the proposal distribution $q_{\tau}$ is the base LLM with temperature $\tau$.
The next state $\mathbf{x}^{(k+1)}$ is then determined by applying the acceptance rule in~\eqref{eq:accept-ratio} to the candidate $\mathbf{y}$.

\noindent\textbf{MCMC Prefetching} is a form of \emph{speculative computation} in MCMC that accelerates an inherently sequential process by precomputing multiple possible future MH transitions before the corresponding accept/reject decisions are realized~\citep{Brockwell01032006,DBLP:conf/uai/AngelinoKWSA14,DBLP:journals/tai/YeL24}.
In our setting, prefetching is best understood through a depth-$K$ tree, which enumerates all possible accept/reject outcomes over the next $K$ MH steps.
Each tree node $v_i$ is associated with a cut index $c_i$ and a proposed candidate.
Its left branch corresponds to accepting a newly proposed candidate, whereas its right branch corresponds to a rejection decision.
Every realization of $K$ MH transitions corresponds to a root-to-leaf path in the tree.
Example~\ref{example:k2-traj} explains this construction in detail.

\begin{example}
\label{example:k2-traj}
Figure~\ref{fig:pipeline}(a) shows the two-step PowerMH trajectory $\mathbf{x}\to\mathbf{y}\to\mathbf{y}$.
In round~1, starting at $\mathbf{x}$, the sampler draws cut index $c_1$, resamples the suffix according to~\eqref{eq:suffix-q} to generate $\mathbf{y}$, and accepts the proposal, yielding $\mathbf{x}^{(1)}=\mathbf{y}$.
  In round~2, it draws a new cut index $c_2$ and resamples the suffix of $\mathbf{y}$ to generate $\mathbf{z}$, but rejects this proposal, leaving $\mathbf{x}^{(2)}=\mathbf{y}$.
  These accept/reject decisions select the path $\mathbf{x}\to\mathbf{y}\to\mathbf{y}$ in the prefetched subtree in step~3 of Figure~\ref{fig:pipeline}(b).
\end{example}

To accelerate PowerMH and its follow-up variants, we propose to construct a \textit{prefetched subtree} that represents multiple possible future MH states and then select one root-to-leaf path using~\eqref{eq:accept-ratio}.

\subsection{Subtree Prefetching for MH-Based Power Sharpening}
\label{subsec:method-llm}

Our method is composed of (1) collecting feasible suffix-resampling requests from the subtree, (2) sampling proposal sequences and computing acceptance probabilities, and (3) performing MH transitions over the prefetched subtree by moving to the left or right child recursively.

\noindent\textbf{(1) Collecting feasible requests.}
We use breadth-first search (BFS) to traverse the prefetching tree and collect at most $N_{\mathrm{pf}}$ feasible suffix-resampling requests. In the BFS construction, we initialize a queue at the subtree root and traverse nodes in breadth-first order.
Each request is a tuple $(v_i,c_i)$, where $v_i$ identifies the node and $c_i$ denotes its cut index.
Expanding a node creates two children: an acceptance (left) child and a rejection (right) child.
For each acceptance child $v_i$, we draw a new cut index $c_i\sim\mathrm{Unif}\{0,\ldots,T-1\}$ and form the request $(v_i,c_i)$.
If this request satisfies Condition~\ref{cond:prefetch}, we add it to the batch and enqueue $v_i$ for further expansion.
Otherwise, we defer the request while retaining its original cut index; if its branch is reached, we evaluate it in a later batch.
We enqueue the rejection child but do not add a corresponding request to the batch, because the child retains its parent's state and therefore requires no new proposal.
Our BFS implementation enqueues the rejection child even when the node's own request is deferred. The requests that BFS then places below this child are generated and scored, but the MH walk of this batch stops at the deferred node and never reaches them; at \(N_{\mathrm{pf}}=10\) on LCB V6 with Qwen3.5-9B, they are \(19\%\) of the issued requests (Appendix~\ref{apx:acceptance-predictors}).
The traversal stops when the batch contains $N_{\mathrm{pf}}$ requests or the queue is empty. Section~\ref{subsec:efficiency} introduces our predictive traversal rule, which improves on the default BFS traversal by maximizing a predicted expected-transition objective, informed by an empirical property of prefetching trees observed across all datasets in Figure~\ref{fig:epsilon-certainty-datasets}.

\begin{condition}
\label{cond:prefetch}
Let $v$ be a node in the prefetching tree with cut index $c$ and $v'$ be its acceptance child with cut index $c'$.
The request associated with $v'$ is feasible for the current prefetch batch if $c'\le c$.
\end{condition}

Condition~\ref{cond:prefetch} follows from the suffix-resampling step in~\eqref{eq:suffix-q}.
Fix a state $\mathbf{x}=\mathbf{x}_{0:T}$ and a cut $0\le c\le T-1$, and let $\mathbf{y}=(\mathbf{x}_{0:c},\mathbf{y}_{c+1:T})$ be a candidate obtained by suffix resampling at $c$.
For any $0\le c'\le c$, the preserved prefix satisfies $\mathbf{y}_{0:c'}=\mathbf{x}_{0:c'}$.
Hence, for every suffix $\mathbf{z}_{c'+1:T}$, autoregressive factorization gives
\begin{equation*}
  q_\tau(\mathbf{z}_{c'+1:T} |\mathbf{y}_{0:c'})=\prod_{t=c'+1}^{T}
  q_\tau(z_t |\mathbf{y}_{0:c'},\mathbf{z}_{c'+1:t-1})=\prod_{t=c'+1}^{T}
  q_\tau(z_t| \mathbf{x}_{0:c'},\mathbf{z}_{c'+1:t-1})=q_\tau(\mathbf{z}_{c'+1:T} |\mathbf{x}_{0:c'}).
\end{equation*}
Consequently, the proposal at cut $c'$ can be generated in the same batch from the already available prefix $\mathbf{x}_{0:c'}$, before $\mathbf{y}$ is available, without changing its conditional distribution.
Example~\ref{ex:compose-cuts} visualizes how BFS collects suffix-resampling requests and checks Condition~\ref{cond:prefetch}.

\begin{example}\label{ex:compose-cuts}
To collect a batch of $N_{\mathrm{pf}}=2$ requests from the tree in step~1 of Figure~\ref{fig:pipeline}(b), we start BFS at the root node $v_0$, which represents the current sequence $\mathbf{x}_{0:T}$.
We assign it the cut index $c_0=T$ for consistency with Condition~\ref{cond:prefetch}.
We expand the tree in the order $v_0, v_1, v_2, \ldots$.
Whenever we visit an acceptance (left) child, we independently sample a new cut index for it.
If each sampled cut index satisfies Condition~\ref{cond:prefetch}, we add the request $(v_1,c_1)$, followed by the request $(v_3,c_2)$.

As shown in step~2 of Figure~\ref{fig:pipeline}(b), resampling the suffix at $v_1$ after cut index $c_1$ leaves the prefix $\mathbf{x}_{0:c_2}$ unchanged because $c_2\le c_1$.
Then, the proposal at $v_3$ is conditioned only on the \textcolor{SharedPrefixColor}{shared prefix $\mathbf{x}_{0:c_2}$}, which is already available at the root node $v_0$.
The two suffix-resampling requests $(v_1,c_1)$ and $(v_3,c_2)$ can therefore be issued in the same model call.
\end{example}

\textbf{(2) Batched proposal sampling.} The collected feasible suffix-resampling requests are submitted in a single batched call to the serving engine of the proposal LLM (e.g., vLLM~\citep{DBLP:conf/sosp/KwonLZ0ZY0ZS23} or SGLang~\citep{DBLP:conf/nips/ZhengYXS0YCKSGB24}), which returns the resampled suffixes together with their log-probabilities.
These log-probabilities are used to compute the MH log-acceptance probability.
Because many requests share long prefixes, the KV cache for each shared prefix can be computed once and reused by all requests that extend it.
Modern LLM serving engines support this reuse through \emph{prefix caching}.
The following Example~\ref{ex:acc-ratio-request} explains the detailed calculation procedure.
Example~\ref{ex:prefix-sharing} (Appendix~\ref{subsec:prefetch-examples}) visualizes how KV-cache reuse reduces repeated prefix computation across batched requests.

\begin{example}\label{ex:acc-ratio-request}
  In steps~2 and~3 of Figure~\ref{fig:pipeline}(b), sequence $\mathbf{x}$ is associated with the root $v_0$, while $\mathbf{y}$ and $\mathbf{z}$ are the newly resampled sequences associated with nodes $v_1$ and $v_3$, respectively, in the tree shown in step~3.
  For the request $(v_1,c_1)$, $\mathbf{y}$ is obtained by resampling the suffix from the proposal LLM $q_{\tau}$ conditioned on the prefix $\mathbf{x}_{0:c_1}$.
  Thus, $\mathbf{y} = \bigl( \mathbf{x}_{0:c_1}, \mathbf{y}_{c_1+1:T} \bigr)$.
  By~\eqref{eq:accept-ratio}, the log-acceptance probability for the MH transition from $\mathbf{x}$ to $\mathbf{y}$ is
  \begin{equation}
  \begin{aligned}
  \label{eq:log-acc-ratio}
  \log A(\mathbf{y}, \mathbf{x})
  =\min\{
  0, &\log q_{\tau}(\mathbf{x}_{c_1+1:T}\mid \mathbf{x}_{0:c_1})
  -
  \log q_{\tau}(\mathbf{y}_{c_1+1:T}\mid \mathbf{x}_{0:c_1}) \\
  &\qquad+\alpha \left(\log p_0(\mathbf{y}_{c_1+1:T}\mid \mathbf{x}_{0:c_1}) - \log p_0(\mathbf{x}_{c_1+1:T}\mid \mathbf{x}_{0:c_1}) \right)   \}.
  \end{aligned}
  \end{equation}
  The calculations are performed in the log domain for numerical stability.
  Each sequence-level term is computed by summing token-level log-probabilities, such as $\log q_{\tau}(\mathbf{y}_{c_1+1:T}\mid \mathbf{x}_{0:c_1}) = \sum_{t=c_1+1}^{T}\log q_{\tau}(y_t \mid \mathbf{x}_{0:c_1},\mathbf{y}_{c_1+1:t-1})$.
  Appendix~\ref{subsec:log-acc-derivation} provides the derivation of~\eqref{eq:log-acc-ratio}.
  \end{example}

\noindent\textbf{(3) MH transitions over the prefetched subtree.}
Starting from the subtree root, representing the current MH state, we simulate the MH accept/reject decisions sequentially along a single root-to-leaf path.
At each node $v$, we follow the left edge with probability $A(v^{+},v)$ and the right edge otherwise.
The resulting root-to-leaf path determines a sequence of accept/reject decisions and, consequently, the final MH state.
We update the current state to the sequence associated with the realized leaf and retain its token-level log-probabilities under the proposal LLM $q_\tau$ and the base model $p_0$ for computing  MH acceptance probabilities.
Example~\ref{ex:random-walk} (Appendix~\ref{subsec:prefetch-examples}) visualizes the above process with the calculated acceptance probabilities from~\eqref{eq:log-acc-ratio}.

Algorithms~\ref{alg:powermh} and~\ref{alg:subtree-mh} place the original PowerMH sampler and \method side by side.
Both share the same blockwise outer loop and the same acceptance rule; they differ only in how the $K$ MH steps of a block are executed.
PowerMH issues one model call per MH step and must wait for each accept/reject outcome before forming the next proposal.
\method instead issues one \emph{batched} call per prefetch round and realizes $k_\Delta\ge 1$ MH steps from it.
In Algorithm~\ref{alg:subtree-mh}, \textsc{PrefetchAndScore} builds the prefetching tree rooted at the current state $\mathbf{x}$, collects at most $N_{\mathrm{pf}}$ feasible requests under the tree traversal rule, and generates and scores them in one batched call.
It returns the scored prefetched subtree $\mathcal{T}$; deferred requests keep their cut indices across rounds.
\textsc{MHTransition} then follows the realized accept/reject path through $\mathcal{T}$ for at most $K-k$ transitions, the number that remain in the block, and returns the reached state $\mathbf{x}_\Delta$ together with the number $k_\Delta\le K-k$ of completed transitions; \textsc{PrefetchAndScore} likewise caps the tree at depth $K-k$ (Appendix~\ref{app:realized-prefetched-path}), so the block never exceeds $K$ transitions.
Proposition~\ref{thm:spec_mh_main} further establishes that \method preserves the original MH transition kernel.

\begin{figure}[!t]
  \refstepcounter{algorithm}\label{alg:powermh}
  \edef\powerMHAlgorithmNumber{\thealgorithm}
  \refstepcounter{algorithm}\label{alg:subtree-mh}
  \edef\subtreeMHAlgorithmNumber{\thealgorithm}
  \hrule height.8pt depth0pt
  \kern2pt
  \begin{minipage}[t]{0.5\textwidth}
  \textbf{Algorithm~\powerMHAlgorithmNumber} PowerMH~\citep{karan2026reasoning}
  \end{minipage}
  \begin{minipage}[t]{0.48\textwidth}
  \textbf{Algorithm~\subtreeMHAlgorithmNumber} \method (ours)
  \end{minipage}
  \kern2pt
  \hrule
  \kern2pt

  \textit{\textcolor{sglOrange}{Common settings:}} the base-model distribution $p_0$, proposal-LLM distribution $q_\tau$, power-sharpening coefficient $\alpha$, response length $n$, block size $B$, and MCMC steps per block $K$. \\
  \textit{\textcolor{sglBlue}{\method-only settings:}} tree traversal rule, prefetch budget $N_{\mathrm{pf}}$.
  
  \textbf{Output:} the final state $\mathbf{x}=\mathbf{x}_{0:n}$ after $K$ MH steps per block, approximately distributed as $\pi_\alpha$.
  \kern2pt
  \vspace{6pt}
  \small

  \begin{minipage}[t]{0.47\textwidth}
  
  \begin{algorithmic}[1]
  
  \For{$\ell \gets 1,\ldots, n/B$}
    \State Extend $\mathbf{x}_{0:L}$ to $\mathbf{x}_{0:T}$ using $q_{\tau}$.
    \For{$k \gets 1,\ldots, K$}
      \State \textcolor{sglOrange}{$\mathbf{y} \gets \Call{DrawOneCandidate}{\mathbf{x}}$.}
      \State \textcolor{sglOrange}{$A \gets \Call{ComputeAcceptance}{\mathbf{y}, \mathbf{x}}$.}
      \State $u \sim \mathrm{Unif}(0,1)$.
      \State \textbf{if} $u \le A$ \textbf{then} $\mathbf{x} \gets \mathbf{y}$.
    \EndFor
    
  \EndFor
  \State \Return $\mathbf{x}$.
  \end{algorithmic}
  \kern2pt
  \hrule
  \end{minipage}
  \hfill
  \begin{minipage}[t]{0.51\textwidth}
  \begin{algorithmic}[1]
  \For{$\ell \gets 1,\ldots, n/B$}
    \State Extend $\mathbf{x}_{0:L}$ to $\mathbf{x}_{0:T}$ using $q_{\tau}$.
    \State $k \gets 0$.
    \While{$k < K$}
      
      \State \textcolor{sglBlue}{$\mathcal{T} \gets \Call{PrefetchAndScore}{\mathbf{x}, N_{\mathrm{pf}}, K-k}$.}
  
      \State \textcolor{sglBlue}{$(\mathbf{x}_\Delta, k_{\Delta}) \gets \Call{MHTransition}{\mathcal{T}, \mathbf{x}, K-k}$.}
      \State $k \gets k + k_{\Delta}$; $\mathbf{x}\gets \mathbf{x}_\Delta$.
    \EndWhile
    
  \EndFor
  \State \Return $\mathbf{x}.$
  
  \end{algorithmic}
  \kern2pt
  \hrule
  \end{minipage}
\end{figure}

\subsection{Predictive Subtree Allocation for Efficient Prefetching}
\label{subsec:efficiency}

Prefetching efficiency depends on the total workload to generate a batch
and on the number of MH transitions that the batch completes. We show
that MH reaches only a sparse, nearly unary part of the prefetched
tree, and use this observation to decide which subtree to prefetch
under a budget.

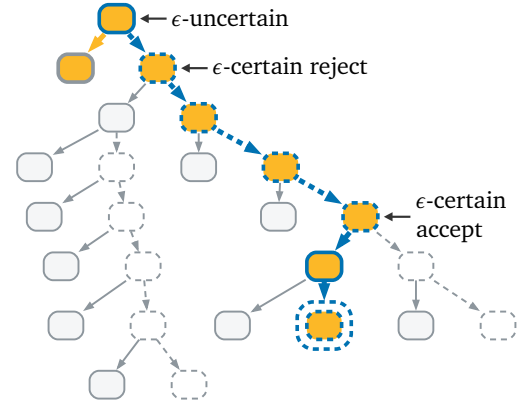
\begin{wrapfigure}{R}{0.41\textwidth}
  \centering
  \centering
  \begin{tikzpicture}
    \node[anchor=south west,inner sep=0] (tree) at (0,0)
      {\includegraphics[width=\linewidth]{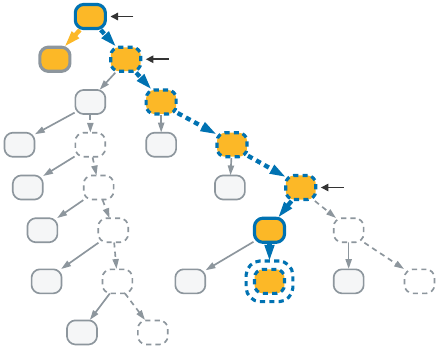}};
    \begin{scope}[x={(tree.south east)},y={(tree.north west)},
                  every node/.style={anchor=west,inner sep=0,font=\footnotesize}]
      \node at (0.3098,0.9526) {$\epsilon$-uncertain};
      \node at (0.3896,0.8297) {$\epsilon$-certain reject};
      \node[align=left] at (0.7835,0.4611) {$\epsilon$-certain\\accept};
    \end{scope}
  \end{tikzpicture}
  \caption{\small The $\epsilon$-rounded subtree (orange) is sparse and mostly
  has one child per node within the full prefetched subtree. Blue edges trace a
  sampled MH path.}
  \label{fig:epsilon-rounded-tree-theory}
\end{wrapfigure}
\textbf{Our Observation.}
Fix $\epsilon\in(0,\tfrac12)$ (we use $\epsilon=0.01$), and let $v$ be an expanded
node with acceptance child $v^{+}$ and rejection child $v^{-}$, reached with
probabilities $A(v^{+},v)$ and $1-A(v^{+},v)$, respectively. The decision at
$v$ is an $\epsilon$-certain accept if $A(v^{+},v)\geq1-\epsilon$, an
$\epsilon$-certain reject if $A(v^{+},v)\leq\epsilon$, and $\epsilon$-uncertain
otherwise (see Figure~\ref{fig:epsilon-rounded-tree-theory} for an example). Rounding the acceptance probability of every $\epsilon$-certain
decision to $1$ or $0$ leaves it with a single reachable child, whereas an
$\epsilon$-uncertain decision keeps both. We call the nodes that remain
reachable from the root after this rounding the $\epsilon$-rounded subtree.

Figure~\ref{fig:traversal-rules}(d) and Appendix~\ref{subsec:tree-visualization} show prefetched
subtrees in full, with their $\epsilon$-rounded subtrees in orange and the
sampled paths in blue. Most nodes lie outside the
$\epsilon$-rounded subtree, so MH is unlikely to visit them. The default BFS
traversal, however, expands decisions regardless of how likely MH is to reach
them, and therefore spends much of the budget on nodes that MH rarely visits.

\begin{figure}[!t]
  \centering
  \includegraphics[width=\linewidth]{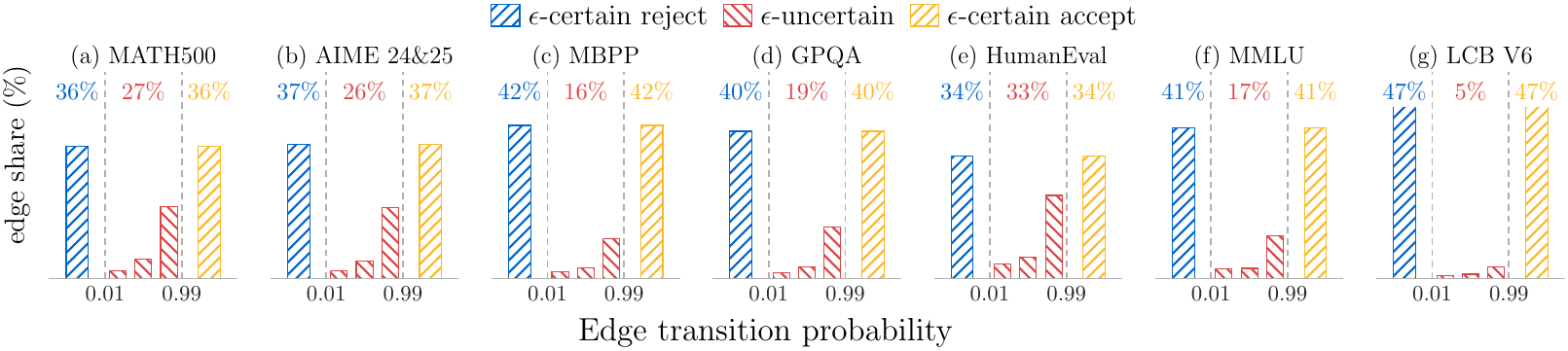}
  \vspace{-2em}
  \caption{\textbf{Most edges in prefetched subtrees are \(\epsilon\)-certain.}
  For \method with Qwen3.5-9B and the default BFS traversal, \(67\%\)--\(95\%\) of edges are \(\epsilon\)-certain across all seven datasets, with \(\epsilon=0.01\) and \(N_{\mathrm{pf}}=20\).}
  \label{fig:epsilon-certainty-datasets}
\end{figure}
The $\epsilon$-rounded subtree is predominantly unary because the target
distribution $\pi_\alpha$ in~\eqref{eq:pi-alpha} amplifies
high-likelihood continuations and suppresses low-likelihood ones, which pushes
most acceptance probabilities toward $0$ or $1$.
Figure~\ref{fig:epsilon-certainty-datasets} confirms that most edges are
$\epsilon$-certain at $\epsilon=0.01$: pooling the acceptance and rejection
edges of all decisions, about $95\%$ of the edges are $\epsilon$-certain on
LCB V6 (panel~(g)). Appendix~\ref{subsec:edge-certainty} reports additional
acceptance distributions.

\textbf{Predictive traversal rule.} These observations motivate allocating
prefetch requests selectively rather than uniformly. We therefore design a
traversal rule that, for the same budget, maximizes the predicted number of
completed MH transitions. Before generation, a three-class logistic regression
on the cut fraction $c_j/T$ of each candidate request predicts the
$\epsilon$-certainty class probabilities of the decision, and we convert
them into a branch probability. We then select at most $N_{\mathrm{pf}}$
requests (one per expanded node) that maximize the predicted number of completed
transitions (Lemma~\ref{lem:prefetch-expectation}), computing this selection
by dynamic programming over the acceptance and rejection subtrees
(Appendices~\ref{subsec:predictive-dp} and~\ref{subsec:dp-allocation}). We call this rule the
\emph{predictive traversal rule}, in contrast to the default BFS traversal of
Section~\ref{subsec:method-llm}. Case Study~5 in
Section~\ref{subsec:case-study} compares the two rules empirically, with
extended results in Appendix~\ref{apx:traversal-rules}.

\section{Related Work}

We review related work most relevant to \framework here; see Appendix~\ref{apx:related} for an extended discussion.

\textbf{MCMC prefetching.}
Like prior MCMC-prefetching methods~\citep{Brockwell01032006,DBLP:conf/uai/AngelinoKWSA14,DBLP:journals/tai/YeL24},
\framework precomputes candidates for future accept/reject outcomes
before the realized path is known. The methods differ in how they
allocate this speculative work. \citet{Brockwell01032006} and
\citet{DBLP:conf/uai/AngelinoKWSA14} treat proposal generation as cheap
and therefore prefetch many more candidates than the chain eventually
uses. LLM proposals are expensive, so \framework confines prefetching to
a budgeted subtree.
\citet{DBLP:journals/tai/YeL24} expand a regular multiproposal tree with a
fixed number of proposals and a fixed prefetch depth, whereas \framework
expands an irregular subtree: it uses Condition~\ref{cond:prefetch} to
identify feasible suffix requests and prioritizes them by predicted reach probability.

\textbf{Tree-based speculative decoding
(TreeSD)}~\citep{DBLP:conf/asplos/MiaoOZCWZWZYSSC24,DBLP:conf/icml/CaiLGPLCD24,DBLP:conf/icml/LiW0024,DBLP:conf/emnlp/LiW0024}
also batches evaluations over a tree to reduce serial model calls, but
its tree encodes different objects. In TreeSD, branches are alternative
\textit{draft-token continuations}, and target-model verification
commits an accepted token prefix, as in single-chain speculative
decoding~\citep{DBLP:conf/nips/SternSU18,DBLP:conf/icml/LeviathanKM23,chen2023accelerating}.
In \framework, branches are the accept/reject outcomes of
\textit{complete-sequence MH proposals}, and following the realized MH
decisions reproduces the sequential PowerMH trajectory. The shared tree
structure admits similar budget-allocation arguments, but the unit of
progress differs: each visited expanded node in \framework advances the
chain by one MH transition, whether the proposal is accepted or
rejected.

\textbf{Power-sharpening methods.}
Other work targets the power distribution through SMC-based
samplers~\citep{DBLP:arxiv/power-smc,he2026recipes}, efficient
approximations~\citep{ji2026scalable,nguyen2026futurevalue}, refined
proposals~\citep{entropycut,powermh-mtm}, or modified objectives and
training procedures~\citep{arzhantsev2026selfconsistency,tomihari2026power}.
\framework takes a different route: it accelerates PowerMH itself,
keeping its target, suffix proposal, and MH acceptance
rule intact, and changes only the schedule by
which proposals are generated and scored.

\section{Experiments}\label{sec:exp}
We first show that \framework speeds up all three sequential samplers across base LLMs and datasets (Table~\ref{tab:paired-speedup}).
Five case studies (Section~\ref{subsec:case-study}) then examine where the speedup comes from, how the prefetch budget affects it, whether the outputs are preserved, how often MH decisions are \(\epsilon\)-certain, and how the traversal rule matters.
Appendix~\ref{apx:exp-set} describes the baselines, base LLMs, datasets, and metrics, and Appendix~\ref{apx:extra-exp} reports the full results.

\begin{table}[!ht]
    \centering
    \caption{Speedup of \framework over sequential samplers: the mean cumulative model-call time of the sampler divided by that of its \framework variant~\eqref{eq:paired-call-speedup}, on the first 20 prompts with \(K=100\) and \(N_{\mathrm{pf}}=20\).}
    \label{tab:paired-speedup}
    \begin{tabular}{@{}lccccccc@{}}
        \multicolumn{8}{c}{\textbf{(a)} PowerMH vs. \method}\\
        \toprule
        Base LLM & MATH500 & AIME 24\&25 & MBPP & GPQA & HumanEval & MMLU & LCB V6 \\
        \midrule
        Qwen3-4B    & \(2.11\times\) & \(2.12\times\) & \(1.10\times\) & \(1.63\times\) & \(1.77\times\) & \(1.59\times\) & \(2.04\times\) \\
        Qwen3-8B    & \(2.23\times\) & \(2.29\times\) & \(1.00\times\) & \(2.12\times\) & \(1.81\times\) & \(1.50\times\) & \(2.05\times\) \\
        Qwen3.5-4B  & \(1.49\times\) & \(1.66\times\) & \(1.67\times\) & \(1.69\times\) & \(1.59\times\) & \(1.76\times\) & \(2.02\times\) \\
        Qwen3.5-9B  & \(2.06\times\) & \(2.08\times\) & \(1.61\times\) & \(2.05\times\) & \(1.21\times\) & \(1.18\times\) & \(1.44\times\) \\ 
        Gemma4-12B-it & \(1.74\times\) & \(1.76\times\) & \(1.44\times\) & \(1.67\times\) & \(1.71\times\) & \(1.52\times\) & \(1.95\times\) \\
        \bottomrule
        \addlinespace[2pt]
        \multicolumn{8}{c}{\textbf{(b)} EntropyCut vs. \methodEnt}\\
        \toprule
        Base LLM & MATH500 & AIME 24\&25 & MBPP & GPQA & HumanEval & MMLU & LCB V6 \\
        \midrule
        Qwen3-4B    & \(1.76\times\) & \(1.69\times\) & \(1.19\times\) & \(2.09\times\) & \(1.88\times\) & \(2.03\times\) & \(1.65\times\) \\
        Qwen3-8B    & \(1.75\times\) & \(1.80\times\) & \(1.43\times\) & \(1.79\times\) & \(2.05\times\) & \(2.20\times\) & \(1.54\times\) \\
        Qwen3.5-4B  & \(1.57\times\) & \(2.06\times\) & \(2.46\times\) & \(1.96\times\) & \(2.06\times\) & \(1.65\times\) & \(1.70\times\) \\
        Qwen3.5-9B  & \(1.76\times\) & \(1.76\times\) & \(2.85\times\) & \(1.96\times\) & \(2.85\times\) & \(2.38\times\) & \(1.91\times\) \\
        Gemma4-12B-it & \(2.27\times\) & \(1.52\times\) & \(2.10\times\) & \(2.16\times\) & \(2.32\times\) & \(2.35\times\) & \(1.79\times\) \\
        \bottomrule
        \addlinespace[2pt]
        \multicolumn{8}{c}{\textbf{(c)} MultiTryMH vs. \methodMTM}\\
        \toprule
        Base LLM & MATH500 & AIME 24\&25 & MBPP & GPQA & HumanEval & MMLU & LCB V6 \\
        \midrule
        Qwen3-4B     & \(2.30\times\) & \(2.02\times\) & \(2.35\times\) & \(2.34\times\) & \(2.21\times\) & \(2.14\times\) & \(2.32\times\) \\
        Qwen3-8B     & \(2.08\times\) & \(2.02\times\) & \(1.97\times\) & \(2.05\times\) & \(1.91\times\) & \(1.92\times\) & \(2.08\times\) \\
        Qwen3.5-4B   & \(5.80\times\) & \(6.12\times\) & \(5.35\times\) & \(5.64\times\) & \(5.59\times\) & \(6.36\times\) & \(5.76\times\) \\
        Qwen3.5-9B   & \(7.01\times\) & \(6.90\times\) & \(5.93\times\) & \(6.83\times\) & \(6.18\times\) & \(6.92\times\) & \(6.87\times\) \\
        Gemma4-12B-it& \(2.27\times\) & \(2.54\times\) & \(2.18\times\) & \(2.13\times\) & \(2.17\times\) & \(2.33\times\) & \(2.22\times\) \\
        \bottomrule
    \end{tabular}%
\end{table}

\subsection{Acceleration Across Base LLMs and Datasets}

We pair each sequential sampler (PowerMH~\citep{karan2026reasoning}, EntropyCut~\citep{entropycut}, and MultiTryMH~\citep{powermh-mtm}) with its \framework variant.
Both samplers in a pair run \(K=100\) realized MH transitions on the first 20 prompts of each dataset and generate the full \(1{,}024\)-token response as a single block on the same vLLM backend, with the base LLM, power exponent \(\alpha\), proposal rule, initialization, and prompts held fixed (Appendix~\ref{apx:hyperparameters}).

Table~\ref{tab:paired-speedup} reports the speedup for every base LLM--dataset pair (full metrics in Appendix~\ref{apx:steps-time-compare}). We find that
\textbf{(1)} \framework speeds up all three baselines: it is never slower in the 105 settings, reaches at least \(1.5\times\) in 96, and achieves median speedups of \(1.71\times\) over PowerMH, \(1.91\times\) over EntropyCut, and \(2.33\times\) over MultiTryMH (ranges \(1.00\times\)--\(2.29\times\), \(1.19\times\)--\(2.85\times\), and \(1.91\times\)--\(7.01\times\), respectively).
Over MultiTryMH, part of the speedup comes from the implementation rather than from prefetching: sequential MultiTryMH rescores every candidate at all three temperatures, whereas \methodMTM reuses the sampled log-probabilities (Appendix~\ref{apx:steps-time-compare}).
Each \framework call completes several MH transitions, which cuts the number of model calls by \(1.77\times\)--\(4.92\times\).
\textbf{(2)} The gain is largest when the sequential step is expensive: on MBPP with Qwen3.5-9B, an EntropyCut step (\(7.62\)\,s) is slower than a PowerMH step (\(4.97\)\,s), so a \methodEnt call costs only \(1.45\times\) a step and the speedup reaches \(2.85\times\), versus \(1.61\times\) for \method.
\textbf{(3)} The gain is smallest when few prefetched proposals are realized relative to the cost of a call: of the five settings below \(1.25\times\), the four of \method (MBPP with both Qwen3 models, and HumanEval and MMLU with Qwen3.5-9B) realize only \(1.96\)--\(2.72\) transitions per call, each costing about \(2\times\) a PowerMH step, and \methodEnt on MBPP with Qwen3-4B realizes \(3.51\) per call, each costing \(2.7\times\) an EntropyCut step.

\subsection{Case Studies}
\label{subsec:case-study}

\textbf{Case Study 1: Model-call efficiency.}
We decompose cumulative model-call time into realized MH transitions per model call (chain progress) and time per call (cost).
Figure~\ref{fig:steps-and-time-compare-qwen35} reports these metrics with Qwen3.5-9B on LCB V6, and Figure~\ref{fig:teaser} places all 105 settings by the two factors of the speedup~\eqref{eq:paired-call-speedup}: the reduction in model calls and the change in time per call.

\begin{figure}[!t]
    \centering
    \includegraphics[width=\linewidth,trim=0 0 0 11.5pt,clip]{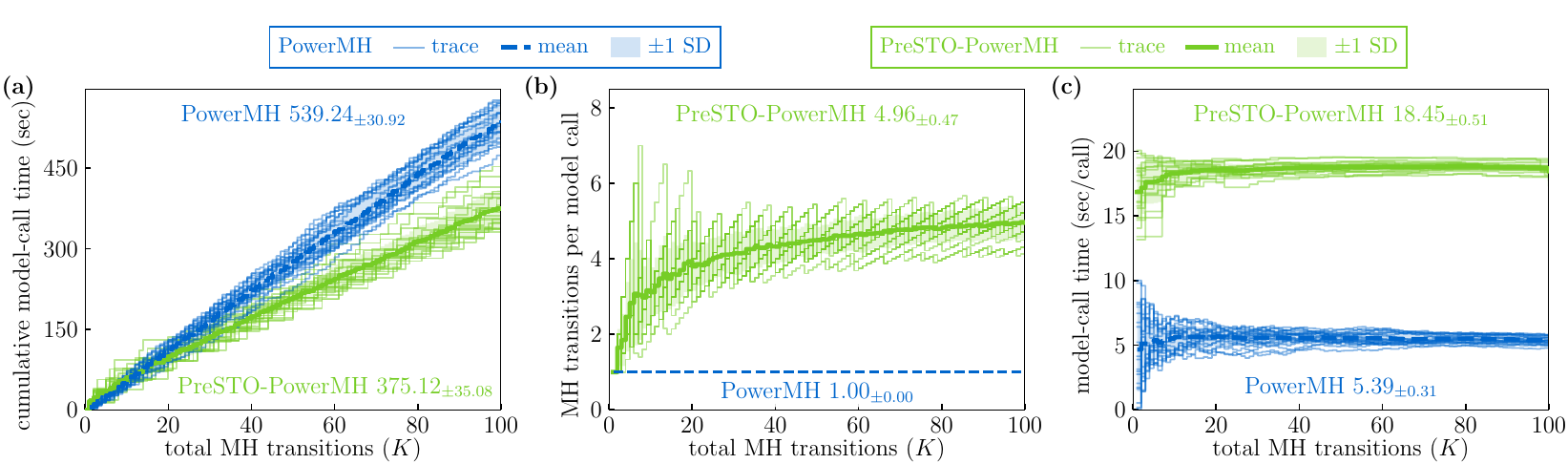}
    \caption{Under LCB V6 and Qwen3.5-9B,
    \framework (\(N_{\mathrm{pf}}=20\), default BFS traversal) speeds up PowerMH by \(1.44\times\) in mean cumulative model-call time: it makes \(4.92\times\) fewer model calls (a ratio of means), each taking \(3.42\times\) as long, while the per-prompt transitions per call average \(4.96\) (Appendix~\ref{apx:metrics}).}
    \label{fig:steps-and-time-compare-qwen35}
\end{figure}

\begin{table}[!t]
    \centering
    \caption{Under dataset LCB V6 and base LLM Qwen3.5-9B,
    raising the prefetch budget \(N_{\mathrm{pf}}\) from \(8\) to \(20\) increases the speedup of \method over PowerMH from \(1.20\times\) to \(1.44\times\).
    Transitions per call is \(K\) divided by the mean number of calls, and time per call is the total call time divided by the total number of calls.}
    \label{tab:prefetch-budget-impact}
    \small
    \begin{tabular}{lrrrrrr}
        \toprule
        & \multirow{2}{*}{PowerMH} & \multicolumn{5}{c}{\method, prefetch budget \(N_{\mathrm{pf}}\)} \\
        \cmidrule(lr){3-7}
        & & 8 & 10 & 12 & 16 & 20 \\
        \midrule
        MH transitions per model call        & 1.00 & 3.59 & 3.94 & 4.27 & 4.63 & 4.91 \\
        Time per model call (s)              & 5.39 & 16.15 & 16.56 & 17.58 & 18.15 & 18.43 \\
        Speedup over PowerMH                 & 1.00\(\times\) & 1.20\(\times\) & 1.28\(\times\) & 1.31\(\times\) & 1.38\(\times\) & 1.44\(\times\) \\
        \bottomrule
    \end{tabular}
\end{table}

From Figure~\ref{fig:steps-and-time-compare-qwen35}, we find that
\textbf{(1)} \method is \(1.44\times\) faster than PowerMH: it completes \(K=100\) transitions in \(375.12_{\pm35.08}\) rather than \(539.24_{\pm30.92}\) seconds.
\textbf{(2)} the speedup comes from fewer calls, not cheaper ones: each \method call realizes \(4.96_{\pm0.47}\) transitions instead of one, but takes \(18.45_{\pm0.51}\) rather than \(5.39_{\pm0.31}\) seconds. A \method call here takes about \(18\) seconds, against \(10.5\)--\(10.9\) seconds for Qwen3.5-9B on the other six datasets (Table~\ref{tab:all-paired-mean-metrics}), so the speedup is lower than in those settings. \methodEnt shows the same pattern on LCB V6 (\(1.91\times\)): \(4.22_{\pm0.58}\) transitions per call at \(10.08_{\pm0.47}\) rather than \(4.64_{\pm0.52}\) seconds; as a ratio of means, it makes \(4.15\times\) fewer calls, each \(2.17\times\) as long.
\textbf{(3)} across the 105 settings (Figure~\ref{fig:teaser}), what drives the speedup depends on the sampler. The log speedup equals the log reduction in calls minus the log increase in time per call; we split its variance between the two terms by their covariances with it (the shares sum to one). Over PowerMH, the reduction in calls (\(1.94\times\)--\(4.92\times\)) takes a \(95\%\) share (\(R^2=0.78\)), because time per call rises by a nearly constant \(1.90\times\)--\(2.42\times\). Over EntropyCut and MultiTryMH, the time-per-call ratio (\(1.37\times\)--\(2.74\times\) and \(0.26\times\)--\(1.00\times\)) takes \(94\%\) and \(99\%\) shares (\(R^2=0.75\) and \(0.99\)), because the reduction in calls varies less (\(3.11\times\)--\(4.52\times\) and \(1.77\times\)--\(2.05\times\)).

\textbf{Case Study 2: Effect of prefetch budget.}
We sweep \(N_{\mathrm{pf}}\in\{8,10,12,16,20\}\) for \method with the default BFS traversal on LCB V6 with Qwen3.5-9B, against the PowerMH run of Case Study~6; all runs share one host and were recorded within three days (Appendix~\ref{apx:lcb-provenance}).
Table~\ref{tab:prefetch-budget-impact} reports chain progress (transitions per call, speedup over PowerMH) and the cost of a call; each budget uses a single seed.

From Table~\ref{tab:prefetch-budget-impact}, we find that
\textbf{(1)} the speedup grows with the budget up to the default \(N_{\mathrm{pf}}=20\): from \(N_{\mathrm{pf}}=8\) to \(20\), transitions per call rise from \(3.59\) to \(4.91\) and the speedup from \(1.20\times\) to \(1.44\times\).
\textbf{(2)} each call slows with the budget, from \(16.2\) to \(18.4\) seconds against \(5.4\) seconds for a PowerMH step, so the speedup grows more slowly than the transitions per call; every prefetched proposal is generated whether or not it is realized.

\textbf{Case Study 3: Likelihood and confidence.}
\framework changes how proposals are scheduled and evaluated, not the MH target, so a sampler and its \framework variant should produce terminal traces with the same distribution of token-normalized log-likelihood and confidence under \(p_0\)~\citep{karan2026reasoning}.
We rescore the terminal traces of each pair under \(p_0\) after \(K=100\) transitions and test each statistic with a paired two one-sided tests (TOST) equivalence test at a margin of \(0.2\) baseline standard deviations (Appendix~\ref{apx:passes}).

\begin{figure}[!t]
    \centering
    \includegraphics[width=\linewidth]{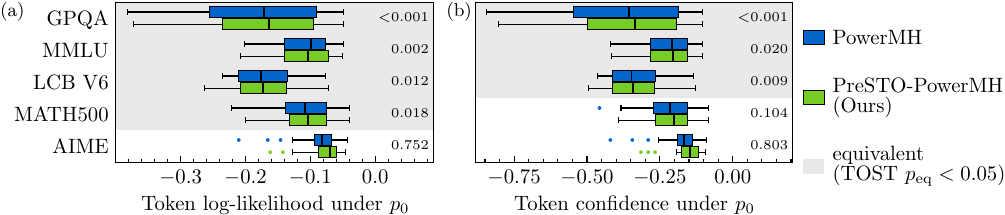}
    \caption{Under base LLM Qwen3.5-4B across five datasets,
    \method and PowerMH produce closely overlapping distributions of base-model log-likelihood and confidence, and the paired TOST finds them equivalent in 7 of the 10 comparisons (shaded).}
    \label{fig:likelihood-confidence}
\end{figure}

From Figure~\ref{fig:likelihood-confidence}, we find that
\textbf{(1)} with Qwen3.5-4B, the distributions of \method and PowerMH overlap closely on all five datasets, and the test establishes equivalence in 7 of the 10 comparisons; in the other three, both statistics on AIME and confidence on MATH500, the 90\% intervals extend past the margin.
\textbf{(2)} the result holds beyond this base LLM: across all pairs that pass the rescoring check of Appendix~\ref{apx:passes} (8 for PowerMH and 11 for EntropyCut, two statistics each), 29 of the 38 tests establish equivalence, and the pooled differences on both statistics lie within the \(\pm0.2\)-SD margin (\(p_{\mathrm{eq}}<0.001\)).

\textbf{Case Study 4: Distribution of $\epsilon$-certainty.}
Figures~\ref{fig:epsilon-certainty-N20} and~\ref{fig:entropycut-certainty} in Appendix~\ref{subsec:edge-certainty} report the distribution of \(\epsilon\)-certainty for all five base LLMs under \method and \methodEnt. We find that
\textbf{(1)} most decisions are \(\epsilon\)-certain for every base LLM: \(56\%\)--\(95\%\) of edges for the Qwen models and \(67\%\)--\(78\%\) for Gemma4-12B-it under \method, and \(60\%\)--\(92\%\) for the Qwen models and \(60\%\)--\(93\%\) for Gemma4-12B-it under \methodEnt.
\textbf{(2)} fewer certain decisions leave more branches in the \(\epsilon\)-rounded tree, so a fixed budget tends to realize fewer transitions per call. This link to runtime is correlational, because the histograms count all scored decisions rather than realized acceptances.

\begin{figure}[!h]
    \centering
    \includegraphics[width=\linewidth]{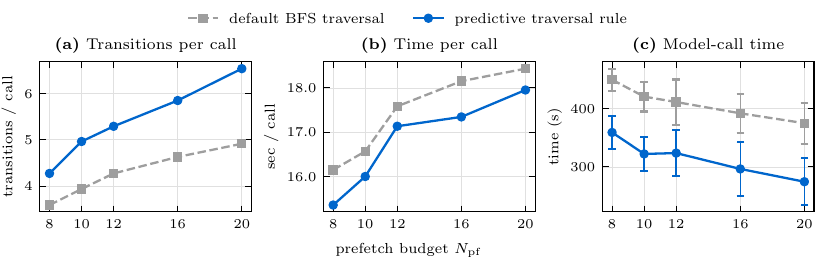}
    \par\medskip
    
    {\small \textbf{(d)} First prefetched subtree of one prompt}\par\smallskip
    \resizebox{\textwidth}{!}{\begin{minipage}{396.0pt}
    \makebox[24.0pt][l]{}%
    \makebox[45.8pt][c]{\footnotesize \(N_{\mathrm{pf}}=8\)}\hspace{4.0pt}%
    \makebox[61.7pt][c]{\footnotesize \(N_{\mathrm{pf}}=10\)}\hspace{4.0pt}%
    \makebox[77.5pt][c]{\footnotesize \(N_{\mathrm{pf}}=12\)}\hspace{4.0pt}%
    \makebox[79.6pt][c]{\footnotesize \(N_{\mathrm{pf}}=16\)}\hspace{4.0pt}%
    \makebox[91.2pt][c]{\footnotesize \(N_{\mathrm{pf}}=20\)}
    \par\smallskip
    \begin{minipage}[b]{20.0pt}\raggedleft\parbox[b][38.7pt][c]{\linewidth}{\raggedleft\rotatebox[origin=c]{90}{\begin{tabular}[c]{@{}c@{}}\footnotesize Prompt 0\\\footnotesize BFS\end{tabular}}}\par{\scriptsize\strut\par}\end{minipage}\hspace{4.0pt}%
    \begin{minipage}[b]{45.8pt}
        \centering
        \includegraphics[scale=0.233]{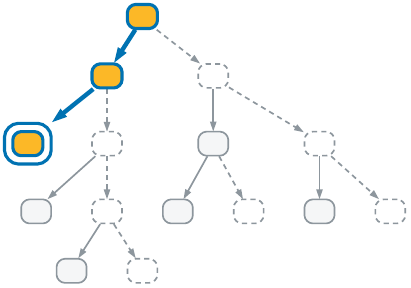}\par
        {\scriptsize MH \(0\rightarrow2\)\par}
    \end{minipage}\hspace{4.0pt}%
    \begin{minipage}[b]{61.7pt}
        \centering
        \includegraphics[scale=0.233]{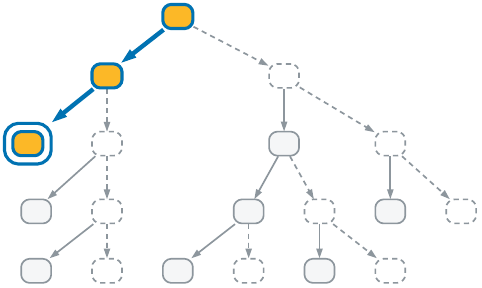}\par
        {\scriptsize MH \(0\rightarrow2\)\par}
    \end{minipage}\hspace{4.0pt}%
    \begin{minipage}[b]{77.5pt}
        \centering
        \includegraphics[scale=0.233]{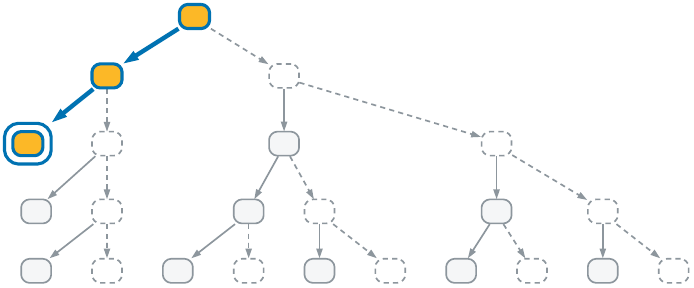}\par
        {\scriptsize MH \(0\rightarrow2\)\par}
    \end{minipage}\hspace{4.0pt}%
    \begin{minipage}[b]{79.6pt}
        \centering
        \includegraphics[scale=0.233]{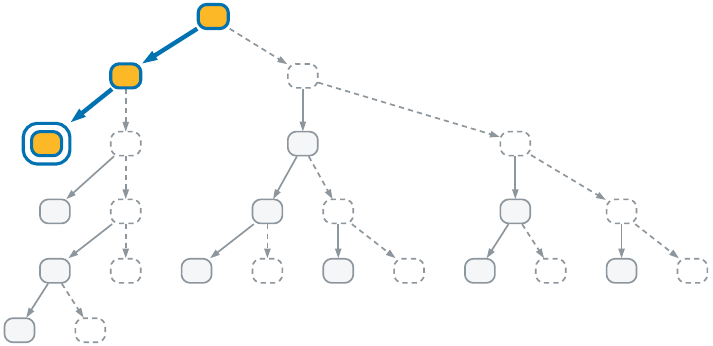}\par
        {\scriptsize MH \(0\rightarrow2\)\par}
    \end{minipage}\hspace{4.0pt}%
    \begin{minipage}[b]{91.2pt}
        \centering
        \includegraphics[scale=0.233]{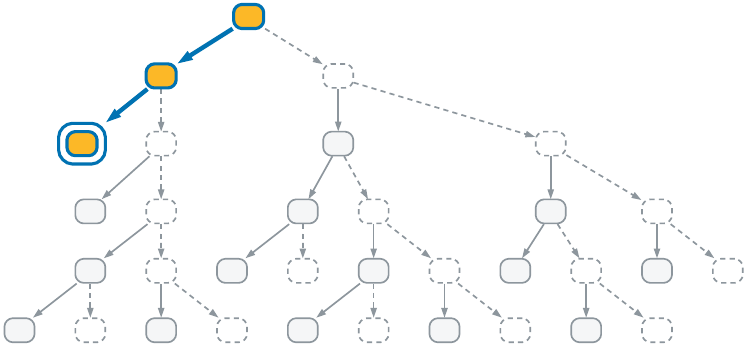}\par
        {\scriptsize MH \(0\rightarrow2\)\par}
    \end{minipage}
    \par\smallskip
    \begin{minipage}[b]{20.0pt}\raggedleft\parbox[b][45.4pt][c]{\linewidth}{\raggedleft\rotatebox[origin=c]{90}{\begin{tabular}[c]{@{}c@{}}\footnotesize Prompt 0\\\footnotesize Predictive\end{tabular}}}\par{\scriptsize\strut\par}\end{minipage}\hspace{4.0pt}%
    \begin{minipage}[b]{45.8pt}
        \centering
        \includegraphics[scale=0.233]{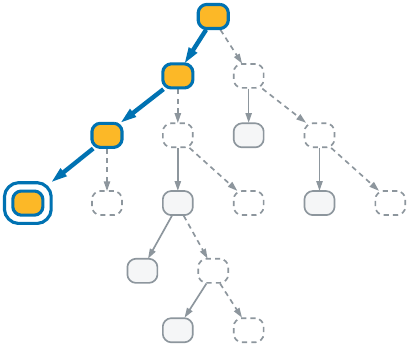}\par
        {\scriptsize MH \(0\rightarrow3\)\par}
    \end{minipage}\hspace{4.0pt}%
    \begin{minipage}[b]{61.7pt}
        \centering
        \includegraphics[scale=0.233]{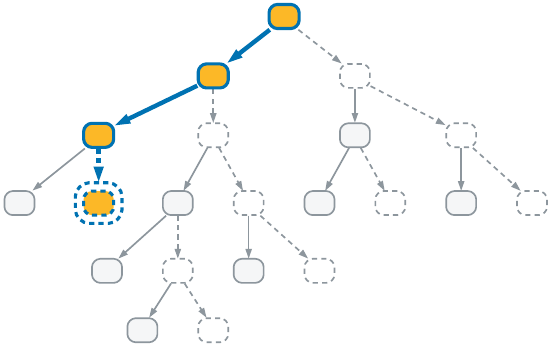}\par
        {\scriptsize MH \(0\rightarrow3\)\par}
    \end{minipage}\hspace{4.0pt}%
    \begin{minipage}[b]{77.5pt}
        \centering
        \includegraphics[scale=0.233]{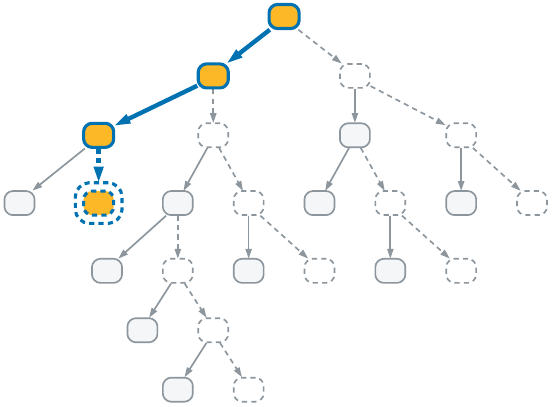}\par
        {\scriptsize MH \(0\rightarrow3\)\par}
    \end{minipage}\hspace{4.0pt}%
    \begin{minipage}[b]{79.6pt}
        \centering
        \includegraphics[scale=0.233]{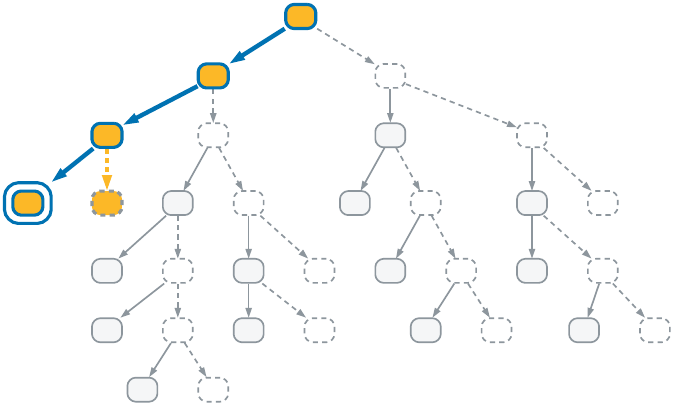}\par
        {\scriptsize MH \(0\rightarrow3\)\par}
    \end{minipage}\hspace{4.0pt}%
    \begin{minipage}[b]{91.2pt}
        \centering
        \includegraphics[scale=0.233]{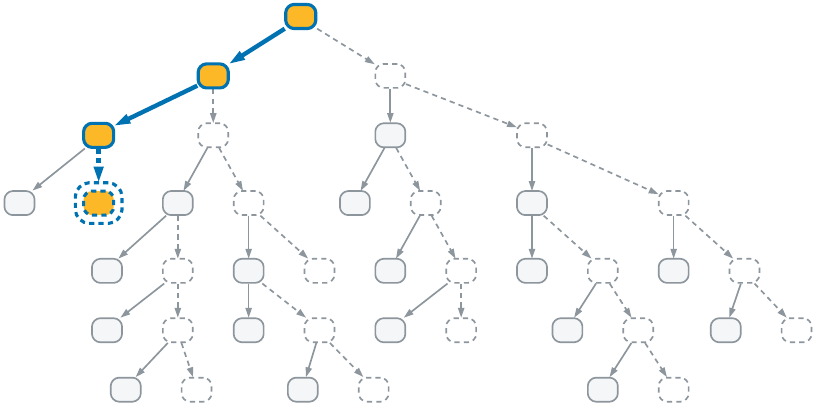}\par
        {\scriptsize MH \(0\rightarrow3\)\par}
    \end{minipage}
    \end{minipage}}

    \caption{Predictive traversal rule against BFS for \method on LCB V6 with Qwen3.5-9B.
    (a--c) The predictive rule is \(1.25\times\)--\(1.37\times\) faster at every prefetch budget because it realizes \(19\%\)--\(33\%\) more MH transitions per call; points in (a, b) are ratios of totals over the 20 prompts, and error bars in (c) show \(\pm1\) SD over prompts.
    (d) First prefetched subtree of prompt 0, a favorable case: 3 transitions against 2 at every budget (prompts 1--3 in Appendix~\ref{apx:traversal-rules}).
    Label MH \(0\rightarrow k\): the sampled path advances the chain from step 0 to \(k\). Solid left edges: acceptance; dashed right edges: rejection; orange: \(\epsilon\)-rounded subtree (\(\epsilon=0.01\)); blue: sampled path (endpoint double-bordered); gray: other prefetched branches.
    }
    \label{fig:traversal-rules}
\end{figure}

\textbf{Case Study 5: Effect of traversal rules.}
We compare the default BFS traversal with the predictive traversal rule, which solves the objective in~\eqref{eq:prefetch-selection-objective} by dynamic programming.
The two rules differ in two ways: the predictive rule weights each candidate by its predicted branch probabilities, and it never selects a request below a deferred one, whereas BFS spends part of its budget on such unreachable requests.
On LCB V6 with Qwen3.5-9B, we run \method with both rules at \(N_{\mathrm{pf}}\in\{8,10,12,16,20\}\) on the same 20 prompts (Appendix~\ref{apx:traversal-rules}).
Appendix~\ref{apx:acceptance-predictors} details the classifier that the predictive rule uses.

From Figure~\ref{fig:traversal-rules}, we find that
\textbf{(1)} the predictive rule is faster than BFS at every budget, by \(1.25\times\)--\(1.37\times\) in cumulative model-call time (panel~(c)), and on 98 of the 100 prompt--budget pairs.
\textbf{(2)} the gain comes from fewer model calls: the predictive rule realizes \(19\%\)--\(33\%\) more transitions per call (panel~(a)), so it needs \(1.19\times\)--\(1.33\times\) fewer calls, while its time per call is \(2.6\%\)--\(4.9\%\) shorter (panel~(b)).
\textbf{(3)} the gain is largest at \(N_{\mathrm{pf}}=20\) (\(1.37\times\)), where transitions per call rise from \(4.91\) to \(6.54\).
Panel~(d) shows the first prefetched subtree of prompt 0 under both rules, and Appendix~\ref{apx:traversal-rules} shows three more prompts.

Three further case studies compare serving engines (Case Study~6, Appendix~\ref{apx:serving-engines}), predict \(\epsilon\)-certainty before generation (Case Study~7, Appendix~\ref{apx:acceptance-predictors}), and compare \method with test-time baselines (Case Study~8, Appendix~\ref{apx:scaling-baselines}); Appendix~\ref{subsec:tree-visualization} shows an example prefetched-subtree trajectory.

\section{Conclusion}
We study how to accelerate MH-based power sampling, whose sequential accept/reject decisions force one model call per MH transition.
We propose \framework, which prefetches proposals for a subtree of future MH states in one batched call and follows the realized path with the original acceptance rule; for PowerMH, this leaves the MH transition kernel unchanged (Proposition~\ref{thm:spec_mh_main}).
Because the part of a prefetched subtree reached by MH is sparse and nearly unary, our predictive traversal rule allocates the prefetch budget by dynamic programming.
In experiments, across five base LLMs and seven datasets, \framework is never slower than its baseline and achieves median speedups of $1.71\times$ over PowerMH, $1.91\times$ over EntropyCut, and $2.33\times$ over MultiTryMH.
The speedup comes from fewer model calls, at a cost of at most $3.1\%$ of the preallocated KV-cache pool, and the log-likelihood and confidence of the outputs pass paired equivalence tests against the baselines in 29 of 38 comparisons.
The predictive traversal rule adds a further \(1.25\times\)--\(1.37\times\) speedup over the default BFS traversal on LCB V6 with Qwen3.5-9B.

\section*{Acknowledgments}
We thank the CoreWeave AI cloud platform for supporting part of the computational resources used in this work. Nan Jiang acknowledges support from the Texas Advanced Computing Center (TACC) under award CCR25054. Panagiotis Theodoropoulos was partly supported by the
Defense Advanced Research Projects Agency (DARPA) through the Artificial Intelligence Quantified
(AIQ) program, under Cooperative Agreement No. HR00112520010. The views, opinions, and/or
findings expressed are those of the authors and should not be interpreted as representing the official
views or policies of the Department of Defense or the U.S. Government.

\bibliography{references}

\begin{thebibliography}{63}
\providecommand{\natexlab}[1]{#1}
\providecommand{\url}[1]{\texttt{#1}}
\expandafter\ifx\csname urlstyle\endcsname\relax
  \providecommand{\doi}[1]{doi: #1}\else
  \providecommand{\doi}{doi: \begingroup \urlstyle{rm}\Url}\fi

\bibitem[Ahmadi et~al.(2026)Ahmadi, Azizi, Potraghloo, Kamal, and Pedram]{ahmadi2026chopthin}
Minoo Ahmadi, Seyedarmin Azizi, Erfan~Baghaei Potraghloo, Mehdi Kamal, and Massoud Pedram.
\newblock {Chopthin-Consensus Power Sampling: A Diversity-Preserving Approach to {LLM} Decoding}.
\newblock In \emph{COLM 2026 Workshop on Efficient Reasoning}, 2026.

\bibitem[Angelino et~al.(2014)Angelino, Kohler, Waterland, Seltzer, and Adams]{DBLP:conf/uai/AngelinoKWSA14}
Elaine Angelino, Eddie Kohler, Amos Waterland, Margo~I. Seltzer, and Ryan~P. Adams.
\newblock {Accelerating {MCMC} via Parallel Predictive Prefetching}.
\newblock In \emph{Proc. of the Conference on Uncertainty in Artificial Intelligence (UAI)}, pages 22--31, 2014.

\bibitem[Arzhantsev et~al.(2026)Arzhantsev, Sakhi, and Chopin]{arzhantsev2026selfconsistency}
Aleksei Arzhantsev, Otmane Sakhi, and Nicolas Chopin.
\newblock {Self-Consistency via Marginal Sharpening}.
\newblock \emph{arXiv preprint arXiv:2605.28142}, 2026.

\bibitem[Austin et~al.(2021)Austin, Odena, Nye, Bosma, Michalewski, Dohan, Jiang, Cai, Terry, Le, and Sutton]{mbpp_dataset}
Jacob Austin, Augustus Odena, Maxwell~I. Nye, Maarten Bosma, Henryk Michalewski, David Dohan, Ellen Jiang, Carrie~J. Cai, Michael Terry, Quoc~V. Le, and Charles Sutton.
\newblock {Program Synthesis with Large Language Models}.
\newblock \emph{arXiv preprint arXiv:2108.07732}, 2021.

\bibitem[Azizi et~al.(2026)Azizi, Potraghloo, Ahmadi, Kundu, and Pedram]{DBLP:arxiv/power-smc}
Seyedarmin Azizi, Erfan~Baghaei Potraghloo, Minoo Ahmadi, Souvik Kundu, and Massoud Pedram.
\newblock {{Power-SMC}: Low-Latency Sequence-Level Power Sampling for Training-Free {LLM} Reasoning}.
\newblock \emph{arXiv preprint arXiv:2602.10273}, 2026.

\bibitem[Brockwell(2006)]{Brockwell01032006}
A.~E. Brockwell.
\newblock {Parallel {Markov} chain {Monte Carlo} Simulation by Pre-Fetching}.
\newblock \emph{Journal of Computational and Graphical Statistics}, 15\penalty0 (1):\penalty0 246--261, 2006.

\bibitem[Cai et~al.(2024)Cai, Li, Geng, Peng, Lee, Chen, and Dao]{DBLP:conf/icml/CaiLGPLCD24}
Tianle Cai, Yuhong Li, Zhengyang Geng, Hongwu Peng, Jason~D. Lee, Deming Chen, and Tri Dao.
\newblock {Medusa: Simple {LLM} Inference Acceleration Framework with Multiple Decoding Heads}.
\newblock In \emph{Proc. of the International Conference on Machine Learning (ICML)}, volume 235, pages 5209--5235, 2024.

\bibitem[Calderhead(2014)]{calderhead2014general}
Ben Calderhead.
\newblock {A General Construction for Parallelizing {Metropolis--Hastings} Algorithms}.
\newblock \emph{Proceedings of the National Academy of Sciences}, 111\penalty0 (49):\penalty0 17408--17413, 2014.

\bibitem[Chen et~al.(2023)Chen, Borgeaud, Irving, Lespiau, Sifre, and Jumper]{chen2023accelerating}
Charlie Chen, Sebastian Borgeaud, Geoffrey Irving, Jean{-}Baptiste Lespiau, Laurent Sifre, and John Jumper.
\newblock {Accelerating Large Language Model Decoding with Speculative Sampling}.
\newblock \emph{arXiv preprint arXiv:2302.01318}, 2023.

\bibitem[Chen et~al.(2021)Chen, Tworek, Jun, Yuan, de~Oliveira~Pinto, Kaplan, Edwards, Burda, Joseph, Brockman, Ray, Puri, Krueger, Petrov, Khlaaf, Sastry, Mishkin, Chan, Gray, Ryder, Pavlov, Power, Kaiser, Bavarian, Winter, Tillet, Such, Cummings, Plappert, Chantzis, Barnes, Herbert-Voss, Guss, Nichol, Paino, Tezak, Tang, Babuschkin, Balaji, Jain, Saunders, Hesse, Carr, Leike, Achiam, Misra, Morikawa, Radford, Knight, Brundage, Murati, Mayer, Welinder, McGrew, Amodei, McCandlish, Sutskever, and Zaremba]{chen2021evaluatingllmcode}
M.~Chen, J.~Tworek, H.~Jun, Q.~Yuan, H.~P. de~Oliveira~Pinto, J.~Kaplan, H.~Edwards, Y.~Burda, N.~Joseph, G.~Brockman, A.~Ray, R.~Puri, G.~Krueger, M.~Petrov, H.~Khlaaf, G.~Sastry, P.~Mishkin, B.~Chan, S.~Gray, N.~Ryder, M.~Pavlov, A.~Power, L.~Kaiser, M.~Bavarian, C.~Winter, P.~Tillet, F.~P. Such, D.~Cummings, M.~Plappert, F.~Chantzis, E.~Barnes, A.~Herbert-Voss, W.~H. Guss, A.~Nichol, A.~Paino, N.~Tezak, J.~Tang, I.~Babuschkin, S.~Balaji, S.~Jain, W.~Saunders, C.~Hesse, A.~N. Carr, J.~Leike, J.~Achiam, V.~Misra, E.~Morikawa, A.~Radford, M.~Knight, M.~Brundage, M.~Murati, K.~Mayer, P.~Welinder, B.~McGrew, D.~Amodei, S.~McCandlish, I.~Sutskever, and W.~Zaremba.
\newblock {Evaluating Large Language Models Trained on Code}.
\newblock \emph{arXiv preprint arXiv:2107.03374}, 2021.

\bibitem[Chen et~al.(2025)Chen, Lu, Zhao, Wang, Yue, Song, and Huang]{DBLP:conf/nips/ChenLZWYSH25}
Zhiqi Chen, Rui Lu, Andrew Zhao, Zhaokai Wang, Yang Yue, Shiji Song, and Gao Huang.
\newblock {Does Reinforcement Learning Really Incentivize Reasoning Capacity in {LLMs} Beyond the Base Model?}
\newblock In \emph{Advances in Neural Information Processing Systems (NeurIPS)}, 2025.

\bibitem[Chen et~al.(2024)Chen, May, Svirschevski, Huang, Ryabinin, Jia, and Chen]{chen2024sequoia}
Zhuoming Chen, Avner May, Ruslan Svirschevski, Yuhsun Huang, Max Ryabinin, Zhihao Jia, and Beidi Chen.
\newblock {{Sequoia}: Scalable and Robust Speculative Decoding}.
\newblock volume~37, 2024.

\bibitem[Du et~al.(2024)Du, Amini, Hennigen, Yu, Lee, Eisner, and Cotterell]{DBLP:conf/icml/DuAHYLEC24}
Li~Du, Afra Amini, Lucas~Torroba Hennigen, Xinyan~Velocity Yu, Holden Lee, Jason Eisner, and Ryan Cotterell.
\newblock {Principled Gradient-Based {MCMC} for Conditional Sampling of Text}.
\newblock In \emph{Proc. of the International Conference on Machine Learning (ICML)}, volume 235, pages 11663--11685, 2024.

\bibitem[Du et~al.(2025)Du, Yang, and Welleck]{DBLP:conf/icml/Du0W25}
Weihua Du, Yiming Yang, and Sean Welleck.
\newblock {Optimizing Temperature for Language Models with Multi-Sample Inference}.
\newblock In \emph{Proc. of the International Conference on Machine Learning (ICML)}, volume 267, 2025.

\bibitem[Emara et~al.(2026)Emara, Barba~da Costa, Chang, Freer, Vieira, Cotterell, and Abdelfattah]{emara2026faster}
Yahya Emara, Mauricio Barba~da Costa, Chi-Chih Chang, Cameron Freer, Tim Vieira, Ryan Cotterell, and Mohamed~S. Abdelfattah.
\newblock {Faster {LLM} Inference via Sequential {Monte Carlo}}.
\newblock \emph{arXiv preprint arXiv:2604.15672}, 2026.

\bibitem[Faria et~al.(2024)Faria, Agrawal, Farinhas, Rei, de~Souza, and Martins]{faria2024quest}
Gon{\c{c}}alo Rui~Alves Faria, Sweta Agrawal, Ant{\'{o}}nio Farinhas, Ricardo Rei, Jos{\'{e}} Guilherme~Camargo de~Souza, and Andr{\'{e}} F.~T. Martins.
\newblock {{QUEST:} Quality-Aware {Metropolis-Hastings} Sampling for Machine Translation}.
\newblock In \emph{Advances in Neural Information Processing Systems (NeurIPS)}, volume~37, pages 89042--89068, 2024.

\bibitem[Forristal et~al.(2023)Forristal, Mireshghallah, Durrett, and Berg{-}Kirkpatrick]{DBLP:conf/conll/ForristalMDB23}
Jarad Forristal, Fatemehsadat Mireshghallah, Greg Durrett, and Taylor Berg{-}Kirkpatrick.
\newblock {A Block {Metropolis-Hastings} Sampler for Controllable Energy-based Text Generation}.
\newblock In \emph{Proc. of the Annual Conference on Computational Natural Language Learning (CoNLL)}, pages 403--413, 2023.

\bibitem[Guo et~al.(2025)Guo, Yang, Zhang, Song, Zhang, Xu, Zhu, Ma, Wang, Bi, et~al.]{guo2025deepseekr1}
Daya Guo, Dejian Yang, Haowei Zhang, Junxiao Song, Ruoyu Zhang, Runxin Xu, Qihao Zhu, Shirong Ma, Peiyi Wang, Xiao Bi, et~al.
\newblock {{DeepSeek-R1}: Incentivizing reasoning capability in {LLMs} via reinforcement learning}.
\newblock \emph{arXiv preprint arXiv:2501.12948}, 2025.

\bibitem[Guo et~al.(2026)Guo, Guo, Meinel, and Yang]{powermh-mtm}
Hong Guo, Nianhui Guo, Christoph Meinel, and Haojin Yang.
\newblock {Sample Where You Struggle: Sharpening Base Model Reasoning via Entropy-Guided Power Sampling}.
\newblock \emph{arXiv preprint arXiv:2606.09926}, 2026.

\bibitem[Hastings(1970)]{hastings1970monte}
W.~K. Hastings.
\newblock {{Monte Carlo} sampling methods using {Markov} chains and their applications}.
\newblock \emph{Biometrika}, 57\penalty0 (1):\penalty0 97--109, 1970.

\bibitem[He et~al.(2025)He, Fried, and Welleck]{he2025rewarding}
Andre He, Daniel Fried, and Sean Welleck.
\newblock {Rewarding the unlikely: Lifting {GRPO} beyond distribution sharpening}.
\newblock \emph{arXiv preprint arXiv:2506.02355}, 2025.

\bibitem[He et~al.(2026)He, Guo, Hern\'{a}ndez-Lobato, and Du]{he2026recipes}
Jiajun He, Zongyu Guo, Jos\'{e}~Miguel Hern\'{a}ndez-Lobato, and Yuanqi Du.
\newblock {Recipes for Steering and Scaling {LLM}s via Sampling}.
\newblock \emph{arXiv preprint arXiv:2608.26120}, 2026.

\bibitem[Hendrycks et~al.(2021{\natexlab{a}})Hendrycks, Burns, Basart, Zou, Mazeika, Song, and Steinhardt]{DBLP:conf/iclr/HendrycksBBZMSS21}
Dan Hendrycks, Collin Burns, Steven Basart, Andy Zou, Mantas Mazeika, Dawn Song, and Jacob Steinhardt.
\newblock {Measuring Massive Multitask Language Understanding}.
\newblock In \emph{Proc. of the International Conference on Learning Representations (ICLR)}, 2021{\natexlab{a}}.

\bibitem[Hendrycks et~al.(2021{\natexlab{b}})Hendrycks, Burns, Kadavath, Arora, Basart, Tang, Song, and Steinhardt]{DBLP:conf/nips/HendrycksBKABTS21}
Dan Hendrycks, Collin Burns, Saurav Kadavath, Akul Arora, Steven Basart, Eric Tang, Dawn Song, and Jacob Steinhardt.
\newblock {Measuring Mathematical Problem Solving With the {MATH} Dataset}.
\newblock In \emph{NeurIPS Datasets and Benchmarks}, 2021{\natexlab{b}}.

\bibitem[Hu et~al.(2025)Hu, Zhang, Han, Jiang, Zhang, and Shum]{hu2025openreasonerzero}
Jingcheng Hu, Yinmin Zhang, Qi~Han, Daxin Jiang, Xiangyu Zhang, and Heung-Yeung Shum.
\newblock {{Open-Reasoner-Zero}: An open source approach to scaling up reinforcement learning on the base model}.
\newblock \emph{arXiv preprint arXiv:2503.24290}, 2025.

\bibitem[Jain et~al.(2025)Jain, Han, Gu, Li, Yan, Zhang, Wang, Solar{-}Lezama, Sen, and Stoica]{DBLP:conf/iclr/JainHGLYZWSSS25}
Naman Jain, King Han, Alex Gu, Wen{-}Ding Li, Fanjia Yan, Tianjun Zhang, Sida Wang, Armando Solar{-}Lezama, Koushik Sen, and Ion Stoica.
\newblock {{LiveCodeBench}: Holistic and Contamination Free Evaluation of Large Language Models for Code}.
\newblock In \emph{Proc. of the International Conference on Learning Representations (ICLR)}, 2025.

\bibitem[Ji et~al.(2026)Ji, Tutunov, Zimmer, and Ammar]{ji2026scalable}
Xiaotong Ji, Rasul Tutunov, Matthieu Zimmer, and Haitham~Bou Ammar.
\newblock {Scalable Power Sampling: Unlocking Efficient, Training-Free Reasoning for {LLM}s via Distribution Sharpening}.
\newblock In \emph{Proc. of the International Conference on Machine Learning (ICML)}, 2026.

\bibitem[Karan and Du(2026)]{karan2026reasoning}
Aayush Karan and Yilun Du.
\newblock {Reasoning with Sampling: Your Base Model is Smarter Than You Think}.
\newblock In \emph{Proc. of the International Conference on Learning Representations (ICLR)}, 2026.

\bibitem[Kwon et~al.(2023)Kwon, Li, Zhuang, Sheng, Zheng, Yu, Gonzalez, Zhang, and Stoica]{DBLP:conf/sosp/KwonLZ0ZY0ZS23}
Woosuk Kwon, Zhuohan Li, Siyuan Zhuang, Ying Sheng, Lianmin Zheng, Cody~Hao Yu, Joseph~E. Gonzalez, Hao Zhang, and Ion Stoica.
\newblock {Efficient Memory Management for Large Language Model Serving with {PagedAttention}}.
\newblock In \emph{Proc. of the {ACM} Symposium on Operating Systems Principles ({SOSP})}, pages 611--626, 2023.

\bibitem[Lakens(2017)]{lakens2017equivalence}
Dani{\"e}l Lakens.
\newblock {Equivalence Tests: A Practical Primer for $t$ Tests, Correlations, and Meta-Analyses}.
\newblock \emph{Social Psychological and Personality Science}, 8\penalty0 (4):\penalty0 355--362, 2017.

\bibitem[Leviathan et~al.(2023)Leviathan, Kalman, and Matias]{DBLP:conf/icml/LeviathanKM23}
Yaniv Leviathan, Matan Kalman, and Yossi Matias.
\newblock {Fast Inference from Transformers via Speculative Decoding}.
\newblock In \emph{Proc. of the International Conference on Machine Learning (ICML)}, volume 202, pages 19274--19286, 2023.

\bibitem[Li et~al.(2025)Li, Karan, and Chen]{li2025blinkofaneyetheory}
Marvin Li, Aayush Karan, and Sitan Chen.
\newblock {Blink of an eye: a simple theory for feature localization in generative models}.
\newblock In \emph{Proc. of the International Conference on Machine Learning (ICML)}, 2025.

\bibitem[Li et~al.(2024{\natexlab{a}})Li, Wei, Zhang, and Zhang]{DBLP:conf/emnlp/LiW0024}
Yuhui Li, Fangyun Wei, Chao Zhang, and Hongyang Zhang.
\newblock {{EAGLE-2:} Faster Inference of Language Models with Dynamic Draft Trees}.
\newblock In \emph{Proc. of the Conference on Empirical Methods in Natural Language Processing (EMNLP)}, pages 7421--7432, 2024{\natexlab{a}}.

\bibitem[Li et~al.(2024{\natexlab{b}})Li, Wei, Zhang, and Zhang]{DBLP:conf/icml/LiW0024}
Yuhui Li, Fangyun Wei, Chao Zhang, and Hongyang Zhang.
\newblock {{EAGLE:} Speculative Sampling Requires Rethinking Feature Uncertainty}.
\newblock In \emph{Proc. of the International Conference on Machine Learning (ICML)}, volume 235, pages 28935--28948, 2024{\natexlab{b}}.

\bibitem[Li et~al.(2022)Li, Choi, Chung, Kushman, Schrittwieser, Leblond, Eccles, Keeling, Gimeno, Lago, Hubert, Choy, de~Masson~d’Autume, Babuschkin, Chen, Huang, Welbl, Gowal, Cherepanov, Molloy, Mankowitz, Robson, Kohli, de~Freitas, Kavukcuoglu, and Vinyals]{journal/science/alphacode22}
Yujia Li, David Choi, Junyoung Chung, Nate Kushman, Julian Schrittwieser, Rémi Leblond, Tom Eccles, James Keeling, Felix Gimeno, Agustin~Dal Lago, Thomas Hubert, Peter Choy, Cyprien de~Masson~d’Autume, Igor Babuschkin, Xinyun Chen, Po-Sen Huang, Johannes Welbl, Sven Gowal, Alexey Cherepanov, James Molloy, Daniel~J. Mankowitz, Esme~Sutherland Robson, Pushmeet Kohli, Nando de~Freitas, Koray Kavukcuoglu, and Oriol Vinyals.
\newblock {Competition-level code generation with {AlphaCode}}.
\newblock \emph{Science}, 378\penalty0 (6624):\penalty0 1092--1097, 2022.

\bibitem[Lightman et~al.(2024)Lightman, Kosaraju, Burda, Edwards, Baker, Lee, Leike, Schulman, Sutskever, and Cobbe]{lightman2024lets}
Hunter Lightman, Vineet Kosaraju, Yura Burda, Harri Edwards, Bowen Baker, Teddy Lee, Jan Leike, John Schulman, Ilya Sutskever, and Karl Cobbe.
\newblock {Let's Verify Step by Step}.
\newblock In \emph{Proc. of the International Conference on Learning Representations (ICLR)}, 2024.

\bibitem[Liu et~al.(2000)Liu, Liang, and Wong]{liu2000multipletry}
Jun~S. Liu, Faming Liang, and Wing~Hung Wong.
\newblock {The Multiple-Try Method and Local Optimization in {Metropolis} Sampling}.
\newblock \emph{Journal of the American Statistical Association}, 95\penalty0 (449):\penalty0 121--134, 2000.

\bibitem[Meng et~al.(2026)Meng, Xie, and Chen]{meng2026depthentropy}
Zibin Meng, Peng Xie, and Kani Chen.
\newblock {Depth-Entropy Guided Sampling for Training-Free {LLM} Reasoning}.
\newblock \emph{arXiv preprint arXiv:2607.09693}, 2026.

\bibitem[Metropolis et~al.(1953)Metropolis, Rosenbluth, Rosenbluth, Teller, and Teller]{metropolis1953equation}
Nicholas Metropolis, Arianna~W. Rosenbluth, Marshall~N. Rosenbluth, Augusta~H. Teller, and Edward Teller.
\newblock {Equation of state calculations by fast computing machines}.
\newblock \emph{Journal of Chemical Physics}, 21\penalty0 (6):\penalty0 1087--1092, 1953.

\bibitem[Miao et~al.(2024)Miao, Oliaro, Zhang, Cheng, Wang, Zhang, Wong, Zhu, Yang, Shi, Shi, Chen, Arfeen, Abhyankar, and Jia]{DBLP:conf/asplos/MiaoOZCWZWZYSSC24}
Xupeng Miao, Gabriele Oliaro, Zhihao Zhang, Xinhao Cheng, Zeyu Wang, Zhengxin Zhang, Rae Ying~Yee Wong, Alan Zhu, Lijie Yang, Xiaoxiang Shi, Chunan Shi, Zhuoming Chen, Daiyaan Arfeen, Reyna Abhyankar, and Zhihao Jia.
\newblock {{SpecInfer}: Accelerating Large Language Model Serving with Tree-based Speculative Inference and Verification}.
\newblock In \emph{Proc. of the {ACM} International Conference on Architectural Support for Programming Languages and Operating Systems ({ASPLOS})}, pages 932--949, 2024.

\bibitem[Mireshghallah et~al.(2022)Mireshghallah, Goyal, and Berg-Kirkpatrick]{mireshghallah2022mix}
Fatemehsadat Mireshghallah, Kartik Goyal, and Taylor Berg-Kirkpatrick.
\newblock {Mix and Match: Learning-Free Controllable Text Generation using Energy Language Models}.
\newblock In \emph{Proc. of the Annual Meeting of the Association for Computational Linguistics (ACL)}, pages 401--415, 2022.

\bibitem[Mittal et~al.(2026)Mittal, Gagnon, and Lajoie]{mittal2026beyond}
Sarthak Mittal, Leo Gagnon, and Guillaume Lajoie.
\newblock {Beyond Distribution Sharpening: The Importance of Task Rewards}.
\newblock In \emph{Proc. of the Annual Conference on Language Modeling (COLM)}, 2026.

\bibitem[Neal(2001)]{neal1998annealedimportance}
Radford~M. Neal.
\newblock {Annealed importance sampling}.
\newblock \emph{Statistics and Computing}, 11\penalty0 (2):\penalty0 125--139, 2001.

\bibitem[Nguyen et~al.(2026)Nguyen, Zimmer, Tutunov, Ji, and Ammar]{nguyen2026futurevalue}
Tu~Nguyen, Matthieu Zimmer, Rasul Tutunov, Xiaotong Ji, and Haitham~Bou Ammar.
\newblock {The Model Knows, the Decoder Finds: Future Value Guided Particle Power Sampling}.
\newblock In \emph{ICML 2026 Workshop on Structured Probabilistic Inference {\&} Generative Modeling}, 2026.

\bibitem[Rein et~al.(2024)Rein, Hou, Cooper~Stickland, Petty, Pang, Dirani, Michael, and Bowman]{conf/colm/rein2024gpqa}
David Rein, Betty~Li Hou, Asa Cooper~Stickland, Jackson Petty, Richard~Yuanzhe Pang, Julien Dirani, Julian Michael, and Samuel~R. Bowman.
\newblock {{GPQA}: A graduate-level {Google}-proof {Q\&A} benchmark}.
\newblock In \emph{Proc. of the Annual Conference on Language Modeling (COLM)}, 2024.

\bibitem[Sambridge(2014)]{sambridge2014paralleltempering}
Malcolm Sambridge.
\newblock {A Parallel Tempering algorithm for probabilistic sampling and optimization}.
\newblock \emph{Geophysical Journal International}, 196\penalty0 (1):\penalty0 357--374, 2014.

\bibitem[Scheufele et~al.(2026)Scheufele, Blohm, and Garg]{scheufele2026temperature}
Hannah Scheufele, Peter Blohm, and Vikas Garg.
\newblock {Temperature Scaling in Discrete Sequence (Language) Models}.
\newblock In \emph{Proc. of the International Conference on Machine Learning (ICML)}, volume 306, 2026.

\bibitem[Schuirmann(1987)]{schuirmann1987tost}
Donald~J. Schuirmann.
\newblock {A comparison of the two one-sided tests procedure and the power approach for assessing the equivalence of average bioavailability}.
\newblock \emph{Journal of Pharmacokinetics and Biopharmaceutics}, 15\penalty0 (6):\penalty0 657--680, 1987.

\bibitem[Shao et~al.(2026)Shao, Li, Xin, Geng, Wang, Oh, Du, Lambert, Min, Krishna, Tsvetkov, Hajishirzi, Koh, and Zettlemoyer]{shao2026spurious}
Rulin Shao, Shuyue~Stella Li, Rui Xin, Scott Geng, Yiping Wang, Sewoong Oh, Simon~Shaolei Du, Nathan Lambert, Sewon Min, Ranjay Krishna, Yulia Tsvetkov, Hannaneh Hajishirzi, Pang~Wei Koh, and Luke Zettlemoyer.
\newblock {Spurious Rewards: Rethinking Training Signals in {RLVR}}.
\newblock In \emph{Proc. of the International Conference on Machine Learning (ICML)}, 2026.

\bibitem[Shao et~al.(2024)Shao, Wang, Zhu, Xu, Song, Zhang, Li, Wu, and Guo]{depseekmath-grpo}
Zhihong Shao, Peiyi Wang, Qihao Zhu, Runxin Xu, Junxiao Song, Mingchuan Zhang, Y.~K. Li, Y.~Wu, and Daya Guo.
\newblock {{DeepSeekMath}: Pushing the Limits of Mathematical Reasoning in Open Language Models}.
\newblock \emph{arXiv preprint arXiv:2402.03300}, 2024.

\bibitem[Stern et~al.(2018)Stern, Shazeer, and Uszkoreit]{DBLP:conf/nips/SternSU18}
Mitchell Stern, Noam Shazeer, and Jakob Uszkoreit.
\newblock {Blockwise Parallel Decoding for Deep Autoregressive Models}.
\newblock In \emph{Advances in Neural Information Processing Systems (NeurIPS)}, pages 10107--10116, 2018.

\bibitem[Stiennon et~al.(2020)Stiennon, Ouyang, Wu, Ziegler, Lowe, Voss, Radford, Amodei, and Christiano]{DBLP:conf/nips/StiennonO0ZLVRA20}
Nisan Stiennon, Long Ouyang, Jeffrey Wu, Daniel~M. Ziegler, Ryan Lowe, Chelsea Voss, Alec Radford, Dario Amodei, and Paul~F. Christiano.
\newblock {Learning to Summarize with Human Feedback}.
\newblock In \emph{Advances in Neural Information Processing Systems (NeurIPS)}, 2020.

\bibitem[Strid(2010)]{strid2010prefetching}
Ingvar Strid.
\newblock {Efficient parallelisation of {Metropolis--Hastings} algorithms using a prefetching approach}.
\newblock \emph{Computational Statistics \& Data Analysis}, 54\penalty0 (11):\penalty0 2814--2835, 2010.

\bibitem[Sun et~al.(2024)Sun, Haider, Zhang, Yang, Qiu, Yin, Wang, Bartlett, and Zanette]{DBLP:conf/nips/SunHZYQYWBZ24}
Hanshi Sun, Momin Haider, Ruiqi Zhang, Huitao Yang, Jiahao Qiu, Ming Yin, Mengdi Wang, Peter~L. Bartlett, and Andrea Zanette.
\newblock {Fast Best-of-N Decoding via Speculative Rejection}.
\newblock In \emph{Advances in Neural Information Processing Systems (NeurIPS)}, volume~37, pages 32630--32652, 2024.

\bibitem[Tomihari and Sato(2026)]{tomihari2026power}
Akiyoshi Tomihari and Issei Sato.
\newblock {Power Distribution Bridges Sampling, Self-Reward {RL}, and Self-Distillation}.
\newblock \emph{arXiv preprint arXiv:2605.04542}, 2026.

\bibitem[Troshin et~al.(2025)Troshin, Mohammed, Meng, Monz, Fokkens, and Niculae]{troshin2025control}
Sergey Troshin, Wafaa Mohammed, Yan Meng, Christof Monz, Antske Fokkens, and Vlad Niculae.
\newblock {Control the Temperature: Selective Sampling for Diverse and High-Quality {LLM} Outputs}.
\newblock In \emph{Proc. of the Annual Conference on Language Modeling (COLM)}, 2025.

\bibitem[Xia et~al.(2024)Xia, Yang, Dong, Wang, Li, Ge, Liu, Li, and Sui]{DBLP:conf/acl/XiaYDW00L0S24}
Heming Xia, Zhe Yang, Qingxiu Dong, Peiyi Wang, Yongqi Li, Tao Ge, Tianyu Liu, Wenjie Li, and Zhifang Sui.
\newblock {Unlocking Efficiency in Large Language Model Inference: {A} Comprehensive Survey of Speculative Decoding}.
\newblock In \emph{Findings of the Association for Computational Linguistics ({ACL})}, pages 7655--7671, 2024.

\bibitem[Ye and Lu(2024)]{DBLP:journals/tai/YeL24}
Guifeng Ye and Shaowen Lu.
\newblock {Prefetching-based Multiproposal {Markov} Chain {Monte Carlo} Algorithm}.
\newblock \emph{{IEEE} Transactions on Artificial Intelligence}, 5\penalty0 (9):\penalty0 4493--4505, 2024.

\bibitem[Zenn and Geiping(2026)]{zenn2026likely}
Johannes Zenn and Jonas Geiping.
\newblock {When are likely answers right? {On} Sequence Probability and Correctness in {LLM}s}.
\newblock In \emph{ICML 2026 Workshop on Structured Probabilistic Inference {\&} Generative Modeling}, 2026.

\bibitem[Zhao et~al.(2024)Zhao, Brekelmans, Makhzani, and Grosse]{DBLP:conf/icml/ZhaoBMG24}
Stephen Zhao, Rob Brekelmans, Alireza Makhzani, and Roger~Baker Grosse.
\newblock {Probabilistic Inference in Language Models via Twisted Sequential {Monte Carlo}}.
\newblock In \emph{Proc. of the International Conference on Machine Learning (ICML)}, volume 235, pages 60704--60748, 2024.

\bibitem[Zheng et~al.(2024)Zheng, Yin, Xie, Sun, Huang, Yu, Cao, Kozyrakis, Stoica, Gonzalez, Barrett, and Sheng]{DBLP:conf/nips/ZhengYXS0YCKSGB24}
Lianmin Zheng, Liangsheng Yin, Zhiqiang Xie, Chuyue Sun, Jeff Huang, Cody~Hao Yu, Shiyi Cao, Christos Kozyrakis, Ion Stoica, Joseph~E. Gonzalez, Clark~W. Barrett, and Ying Sheng.
\newblock {{SGLang}: Efficient Execution of Structured Language Model Programs}.
\newblock In \emph{Advances in Neural Information Processing Systems (NeurIPS)}, 2024.

\bibitem[Zhou et~al.(2026)Zhou, Mehrotra, and Liu]{entropycut}
Felix Zhou, Anay Mehrotra, and Quanquan~C. Liu.
\newblock {Reasoning with Sampling: Cutting at Decision Points}.
\newblock \emph{arXiv preprint arXiv:2605.30327}, 2026.

\bibitem[Zhu et~al.(2024)Zhu, Li, Li, Zhao, Li, Jin, and Mei]{DBLP:conf/aaai/ZhuLLZLJ024}
Yuqi Zhu, Jia Li, Ge~Li, Yunfei Zhao, Jia Li, Zhi Jin, and Hong Mei.
\newblock {Hot or Cold? {Adaptive} Temperature Sampling for Code Generation with Large Language Models}.
\newblock In \emph{Proc. of the AAAI Conference on Artificial Intelligence (AAAI)}, pages 437--445, 2024.

\end{thebibliography}

\newpage
\appendix
\begingroup
\setcounter{tocdepth}{2}
\setlength{\cftbeforetoctitleskip}{0pt}
\setlength{\cftaftertoctitleskip}{6pt}
\setlength{\cftbeforesecskip}{2pt}
\renewcommand{\cftsubsecfont}{\small}
\renewcommand{\cftsubsecpagefont}{\small}
\tableofcontents
\endgroup
\allowdisplaybreaks
\newpage

\section{Extended Related Work}\label{apx:related}

\framework draws on three lines of work.
Power-sampling methods define the target and the samplers that \framework accelerates.
Speculative decoding and MCMC prefetching are two established ways to break serial dependence during sampling.
For each line, we summarize the problem it addresses, the mechanism its methods share, and the trade-offs among them.

\subsection{Sampling from the Power Distribution}

Power sampling draws complete responses from $\pi_\alpha\propto p_0^\alpha$ with $\alpha>1$~\citep{karan2026reasoning}.
Raising the sequence likelihood to a power concentrates probability on responses that the base model finds likely as a whole, the same construction that annealing and parallel tempering use to control how concentrated a target is~\citep{neal1998annealedimportance,sambridge2014paralleltempering}.

\textbf{Why token-level sharpening is not enough.}
Token-level temperature scaling and filtering~\citep{DBLP:conf/aaai/ZhuLLZLJ024,DBLP:conf/icml/Du0W25,troshin2025control} sharpen each next-token distribution and renormalize at every prefix.
The exact next-token conditional of the power target is instead
\begin{equation*}
\pi_\alpha(x_t\mid\mathbf{x}_{0:t-1})\;\propto\;p_0(x_t\mid\mathbf{x}_{0:t-1})^{\alpha}\,Z_\alpha(\mathbf{x}_{0:t}),
\qquad\text{where}\qquad
Z_\alpha(\mathbf{x}_{0:t})=\sum_{\mathbf{x}_{t+1:n}}p_0(\mathbf{x}_{t+1:n}\mid\mathbf{x}_{0:t})^{\alpha}.
\end{equation*}
Token-level temperature $\tau=1/\alpha$ keeps the first factor and drops $Z_\alpha$.
For $\alpha>1$, $Z_\alpha$ is large when the continuations of a prefix concentrate on a few high-likelihood completions, so the power target prefers such prefixes, whereas low-temperature decoding cannot distinguish them.
Because $Z_\alpha$ sums over exponentially many continuations, every power-sampling method must avoid it, approximate it, or change the target.
Existing methods follow three routes.

\textbf{MCMC over complete sequences.}
MH needs the target only up to a constant, so it avoids $Z_\alpha$ entirely: it compares complete sequences through the ratio $\pi_\alpha(\mathbf{y})/\pi_\alpha(\mathbf{x})=\bigl(p_0(\mathbf{y})/p_0(\mathbf{x})\bigr)^{\alpha}$.
The main design choice is the proposal.
Earlier MCMC samplers for controlled generation with energy-based targets edit one position at a time and need many steps to make coordinated changes across a sequence~\citep{mireshghallah2022mix}.
Regenerating a complete sequence or a suffix with an LLM instead makes large, coherent edits whose proposal probabilities come from the same forward pass~\citep{DBLP:conf/conll/ForristalMDB23,faria2024quest}.
PowerMH applies suffix regeneration to $\pi_\alpha$ with a blockwise schedule~\citep{karan2026reasoning}, and later variants keep this structure while changing where cuts are placed or how many candidates each step considers.
EntropyCut~\citep{entropycut} places cuts using token entropy, and MultiTryMH, a variant of the entropy-guided sampler of~\citet{powermh-mtm}, draws several candidate suffixes per step and moves with a multiple-try Metropolis update~\citep{liu2000multipletry}.
Other variants add terms to the target~\citep{meng2026depthentropy}.
These samplers share three costs: each transition regenerates on average about half of the response when the cut is uniform, each stage needs many transitions, and transitions are serial because every proposal is conditioned on the current state.

\textbf{Forward approximations.}
Forward methods generate left to right and estimate $Z_\alpha$ rather than drop it.
Particle methods propagate a population of partial sequences from a tractable proposal, such as low-temperature decoding, reweight them at every step by the ratio of target to proposal, and resample to discard low-weight prefixes.
Without twisting, these intermediate weights ignore the future term $Z_\alpha$; twisted variants learn approximations of $Z_\alpha$ so that the intermediate weights account for it~\citep{DBLP:conf/icml/ZhaoBMG24}.
For power targets, these samplers are consistent as the number of particles grows~\citep{DBLP:arxiv/power-smc,he2026recipes}, and in practice their quality depends on resampling rules that prevent weight degeneracy and particle collapse~\citep{ahmadi2026chopthin}.
Lookahead methods instead estimate $Z_\alpha(\mathbf{x}_{0:t})$ directly by sampling a few continuations for each candidate token, and use the estimate to correct the next-token distribution~\citep{ji2026scalable,nguyen2026futurevalue}.
Both families parallelize naturally across particles or rollouts.
Their bias and variance, however, depend on the number of particles or rollouts, and they never revise a token; resampling only drops or duplicates whole prefixes.

\textbf{Changing or amortizing the target.}
A third route changes what is sampled or moves the cost into training.
Sharpening the sequence likelihood can favor a single high-likelihood reasoning trace over an answer supported by many distinct traces, so marginal sharpening sharpens the distribution over final answers instead~\citep{arzhantsev2026selfconsistency}.
Distillation-style methods fine-tune the model so that ordinary sampling approximates a powered or low-temperature sequence distribution, which removes the inference-time overhead at the price of training~\citep{tomihari2026power,scheufele2026temperature}.
Evidence on when sharpening helps is mixed: raising sequence likelihood through decoding need not improve correctness~\citep{zenn2026likely}, and reward-trained models can gain in ways that sharpening does not recover~\citep{mittal2026beyond}.

\subsection{Tree-based Speculative Decoding}

Tree-based speculative decoding (TreeSD) is the closest counterpart to \framework in the speculative-decoding literature, because it also spends extra parallel computation on work that may be discarded in order to reduce the number of serial model calls.
We first describe the draft-then-verify scheme that TreeSD builds on, then how trees extend it, and finally how \framework differs.

\textbf{Draft-then-verify decoding.}
Autoregressive decoding calls the target model once per token, and these calls cannot be batched because each token depends on the previous one.
Speculative decoding breaks this dependence with a cheap drafter that proposes several future tokens~\citep{DBLP:conf/nips/SternSU18,DBLP:conf/acl/XiaYDW00L0S24}.
The target model then scores all of them in a single forward pass, and a verification rule commits the longest accepted prefix together with one token drawn from a corrected distribution~\citep{DBLP:conf/icml/LeviathanKM23,chen2023accelerating}.
With the rejection-based rule, the committed tokens follow the target model's distribution exactly, so drafter quality affects only speed.
Every pass commits at least one token, and the expected number of tokens committed per pass grows with the drafter's acceptance rate.

\textbf{Drafting a tree of continuations.}
A chain of draft tokens is truncated at its first rejected token, so one early rejection wastes all later draft work.
Tree-based methods instead draft several alternative continuations at each position, and verify all branches in one pass with a tree-structured attention mask, so that each branch attends only to its own ancestors~\citep{DBLP:conf/asplos/MiaoOZCWZWZYSSC24}.
The target model can then commit whichever branch is accepted furthest, which raises the expected number of tokens committed per pass at the cost of scoring more candidate tokens.
In all of these methods, the nodes of the tree are draft \emph{tokens}, the branching reflects uncertainty about what the target model will say next, and a verification pass ends with a single committed prefix.

\textbf{Better drafters.}
A second line of work raises the acceptance rate of the drafter itself.
Extra decoding heads attached to the target model draft several positions at once without a separate drafter model~\citep{DBLP:conf/icml/CaiLGPLCD24}, and autoregression at the level of the target model's features gives a drafter that tracks the target more closely~\citep{DBLP:conf/icml/LiW0024}.
Better drafters shift the balance between how many tokens are drafted and how many are accepted, and so change which tree shapes are worth verifying.

\textbf{Tree construction as budget allocation.}
For a fixed verification budget, i.e., a fixed number of nodes per pass, a tree can take many shapes, so tree construction is a budget-allocation problem.
Early methods use a static shape~\citep{DBLP:conf/asplos/MiaoOZCWZWZYSSC24,DBLP:conf/icml/CaiLGPLCD24}.
Later methods add branches where the drafter is confident, since those nodes are more likely to be accepted~\citep{DBLP:conf/emnlp/LiW0024}, or choose the tree shape by dynamic programming over estimated acceptance rates~\citep{chen2024sequoia}.
These approaches share one principle: under a node budget, expand the nodes whose expected contribution to the committed prefix is largest.

\textbf{Relaxing exact verification.}
Exact rejection is not the only option.
SMC-based variants weight and resample draft blocks instead of truncating at the first rejected token.
They accept a finite-particle approximation in exchange for using more of the drafted work, and have been applied to power targets~\citep{emara2026faster}.
This differs from \framework in that the output is no longer exact for the target, whereas \framework preserves the PowerMH trajectory.

\textbf{Relation to \framework.}
\framework and TreeSD both use a tree to batch evaluations that would otherwise be serial, both leave the law of the underlying sampler unchanged (for \method, the law of the PowerMH chain, not $\pi_\alpha$ itself), and both face the same question of where to spend a limited speculative budget, so TreeSD's budget-allocation arguments carry over to our setting.
The two trees encode different objects, however.
In TreeSD, a branch is an alternative draft-token continuation, and verification commits an accepted token prefix.
In \framework, a branch is an accept/reject outcome of a complete-sequence MH proposal, and following the realized MH decisions reproduces the sequential PowerMH trajectory.
The unit of progress also differs: every visited expanded node in \framework advances the chain by one MH transition, whether the proposal is accepted or rejected, while a TreeSD node contributes only if it lies on the accepted prefix.
Finally, the cost structure differs.
In TreeSD, a cheap drafter proposes the nodes, and only the target model's verification pass is expensive.
\framework has no cheap drafter: every node is a full suffix that the proposal LLM must generate, the same work as one sequential PowerMH proposal.
Expanding the full tree would spend most of this work on nodes off the realized path, so \framework confines speculation to a budgeted subtree.

\subsection{MCMC Prefetching}

Parallel MCMC methods typically run many independent chains, parallelize the evaluation of the target density, or speculate on the future of a single chain.
Prefetching is the third approach~\citep{Brockwell01032006,strid2010prefetching}.
It rests on one observation: from any state, the next $d$ accept/reject outcomes form a binary tree of depth $d$, and the chain follows exactly one root-to-leaf path.
Proposals for every node can be drawn in advance, so the expensive target evaluations at all nodes can run in parallel.
The realized path is then read off with the usual acceptance uniforms, so the chain has exactly the law of sequential MH.

Evaluating the whole tree is wasteful.
It has $2^d-1$ decision nodes, of which only $d$ lie on the realized path, so the speedup obtained from $p$ parallel workers grows only logarithmically in $p$~\citep{Brockwell01032006}.
Later methods therefore evaluate only the nodes most likely to lie on the realized path.
Static prefetching uses the acceptance rate: when acceptance is low, rejection branches dominate, and the tree can be restricted to long rejection chains~\citep{strid2010prefetching}.
Predictive prefetching uses a cheap approximation of the target to estimate each branch probability, and still decides every realized transition with the exact target~\citep{DBLP:conf/uai/AngelinoKWSA14}.

A different approach parallelizes within one transition by drawing several candidates from the current state, whose target evaluations can run in parallel.
Multiple-try Metropolis selects one candidate with probability proportional to its weight and accepts or rejects it with a corrected acceptance probability~\citep{liu2000multipletry}.
Multiproposal MH instead moves according to a Markov chain over the candidate set~\citep{calderhead2014general}.
Both preserve the stationary distribution but change the transition kernel.
Multiproposal steps can be combined with prefetching by expanding a tree of multiproposal steps~\citep{DBLP:journals/tai/YeL24}.
MultiTryMH~\citep{powermh-mtm} applies multiple-try Metropolis to $\pi_\alpha$, drawing its candidate suffixes in parallel; \framework is complementary to this within-transition parallelism, and \methodMTM prefetches across MultiTryMH transitions.

\newpage

\section{Implementation and Correctness of \framework} 
\label{apx:implement}

\subsection{Computing Log-Acceptance Probabilities}
\label{subsec:log-acc-derivation}

We follow the standard steps to derive a simplified expression for the suffix-resampling log-acceptance probability in~\eqref{eq:log-acc-ratio}
by specializing the generic MH acceptance rule in~\eqref{eq:accept-ratio}.

All states share the same length $T$: every proposal is generated to the full horizon and does not stop at the end-of-sequence token.
Fix a current state $\mathbf{x}=\mathbf{x}_{0:T}$ and a cut index $c$.
At this cut, the proposal retains the prefix $\mathbf{x}_{0:c}$ and samples a new suffix
$\mathbf{y}_{c+1:T}\sim q_\tau(\cdot\mid\mathbf{x}_{0:c})$ from the proposal LLM.
The resulting candidate is $\mathbf{y}=(\mathbf{x}_{0:c},\mathbf{y}_{c+1:T})$,
so $\mathbf{y}_{0:c}=\mathbf{x}_{0:c}$.
We write $q(\mathbf{y}\mid\mathbf{x},c)$ for the MH proposal kernel conditional on the cut, to distinguish it from the token-level LLM conditionals $q_\tau$.
PowerMH draws $c\sim\mathrm{Unif}\{0,\ldots,T-1\}$ independently of the state, and the reverse move from $\mathbf{y}$ to $\mathbf{x}$ uses the same cut.
Each cut-conditional kernel with the acceptance probability below is $\pi_\alpha$-reversible, and so is their uniform mixture over $c$.
For this candidate, the acceptance rule in~\eqref{eq:accept-ratio} becomes
\begin{equation}
\label{eq:acc-instantiated}
A(\mathbf{y},\mathbf{x})
=
\min\!\left\{1,\;
\frac{\pi_\alpha(\mathbf{y}_{1:T}\mid\mathbf{x}_0)\,
      q(\mathbf{x}\mid\mathbf{y},c)}
     {\pi_\alpha(\mathbf{x}_{1:T}\mid\mathbf{x}_0)\,
      q(\mathbf{y}\mid\mathbf{x},c)}
\right\}.
\end{equation}
We next simplify the proposal and target ratios in~\eqref{eq:acc-instantiated} separately.

\begin{enumerate}[leftmargin=*]
  \item \textbf{Proposal log-ratio.}
  At the fixed cut $c$, the candidate differs from $\mathbf{x}$ only in the resampled suffix, so the forward proposal probability is
\begin{equation*}
q(\mathbf{y}\mid\mathbf{x},c)
=
q_\tau\bigl(\mathbf{y}_{c+1:T}\mid\mathbf{x}_{0:c}\bigr).
\end{equation*}
The reverse move applies the same cut $c$ to propose $\mathbf{x}$ from $\mathbf{y}$.
It preserves $\mathbf{y}_{0:c}$ and generates the original suffix
$\mathbf{x}_{c+1:T}$.
Using $\mathbf{y}_{0:c}=\mathbf{x}_{0:c}$, the reverse proposal probability becomes
\begin{equation*}
q(\mathbf{x}\mid\mathbf{y},c)
=
q_\tau\bigl(\mathbf{x}_{c+1:T}\mid\mathbf{y}_{0:c}\bigr)
=
q_\tau\bigl(\mathbf{x}_{c+1:T}\mid\mathbf{x}_{0:c}\bigr).
\end{equation*}
Therefore, the log-proposal ratio is
\begin{equation}
\label{eq:proposal-ratio}
\log\frac{q(\mathbf{x}\mid\mathbf{y},c)}{q(\mathbf{y}\mid\mathbf{x},c)}
=
\log q_\tau\bigl(\mathbf{x}_{c+1:T}\mid\mathbf{x}_{0:c}\bigr)
-
\log q_\tau\bigl(\mathbf{y}_{c+1:T}\mid\mathbf{x}_{0:c}\bigr).
\end{equation}

\item \textbf{Target log-ratio.}
By the stage target~\eqref{eq:stage-target} with $\ell B=T$,
$\pi_\alpha(\mathbf{x}_{1:T}\mid\mathbf{x}_0)
\propto p_0(\mathbf{x}_{1:T}\mid\mathbf{x}_0)^\alpha$.
Because the numerator and denominator use the same prompt and the same fixed length $T$,
their common normalizing constant cancels:
\begin{equation*}
\frac{\pi_\alpha(\mathbf{y}_{1:T}\mid\mathbf{x}_0)}
     {\pi_\alpha(\mathbf{x}_{1:T}\mid\mathbf{x}_0)}
=
\left(
\frac{p_0(\mathbf{y}_{1:T}\mid\mathbf{x}_0)}
     {p_0(\mathbf{x}_{1:T}\mid\mathbf{x}_0)}
\right)^{\!\alpha}.
\end{equation*}
Since $\mathbf{y}_{1:c}=\mathbf{x}_{1:c}$, autoregressively factorizing each term on the right-hand side at the cut gives
\begin{align*}
p_0(\mathbf{y}_{1:T}\mid\mathbf{x}_0)
&=
p_0(\mathbf{x}_{1:c}\mid\mathbf{x}_0)\,
p_0(\mathbf{y}_{c+1:T}\mid\mathbf{x}_{0:c}),\\
p_0(\mathbf{x}_{1:T}\mid\mathbf{x}_0)
&=
p_0(\mathbf{x}_{1:c}\mid\mathbf{x}_0)\,
p_0(\mathbf{x}_{c+1:T}\mid\mathbf{x}_{0:c}).
\end{align*}
Canceling the common prefix probability
$p_0(\mathbf{x}_{1:c}\mid\mathbf{x}_0)$ and taking logarithms yields
\begin{equation}
\label{eq:target-ratio}
\log\frac{\pi_\alpha(\mathbf{y}_{1:T}\mid\mathbf{x}_0)}
         {\pi_\alpha(\mathbf{x}_{1:T}\mid\mathbf{x}_0)}
=
\alpha\Bigl(
\log p_0\bigl(\mathbf{y}_{c+1:T}\mid\mathbf{x}_{0:c}\bigr)
-
\log p_0\bigl(\mathbf{x}_{c+1:T}\mid\mathbf{x}_{0:c}\bigr)
\Bigr).
\end{equation}

\item \textbf{Combined log-ratio.}
Taking the logarithm of the acceptance rule in~\eqref{eq:acc-instantiated}, using
$\log\min\{1,r\}=\min\{0,\log r\}$, and substituting~\eqref{eq:proposal-ratio}
and~\eqref{eq:target-ratio} yields
\begin{equation*}
\begin{aligned}
\log A(\mathbf{y},\mathbf{x})
=
\min\Bigl\{0,&\;
\log q_\tau(\mathbf{x}_{c+1:T}\mid\mathbf{x}_{0:c})
-\log q_\tau(\mathbf{y}_{c+1:T}\mid\mathbf{x}_{0:c})\\
&
\quad +\alpha\bigl(\log p_0(\mathbf{y}_{c+1:T}\mid\mathbf{x}_{0:c})
-\log p_0(\mathbf{x}_{c+1:T}\mid\mathbf{x}_{0:c})\bigr)
\Bigr\},
\end{aligned}
\end{equation*}
which is~\eqref{eq:log-acc-ratio} with $c_1$ replaced by $c$.
Each suffix log-probability on the right-hand side is evaluated by summing
token-level conditional log-probabilities from the serving engine.
For example, 
\begin{equation*}
\log q_\tau(\mathbf{y}_{c+1:T}\mid\mathbf{x}_{0:c})
=
\sum_{t=c+1}^{T}
\log q_\tau(y_t\mid\mathbf{x}_{0:c},\mathbf{y}_{c+1:t-1}).
\end{equation*}
Evaluating these quantities in the log domain avoids underflow from multiplying
many small token probabilities and improves numerical stability.

\end{enumerate}

In our implementation, the serving engine of the proposal LLM returns these log-probabilities. A single batched model call generates
candidate suffixes and provides their log-probabilities under both $q_\tau$ and
$p_0$, using the same logits with and without temperature scaling.

\subsection{Exactness of Subtree Prefetching}
\label{apx:proofs}

Reproducing the sequential chain under speculative evaluation is a standard
property of MCMC prefetching~\citep{Brockwell01032006,DBLP:conf/uai/AngelinoKWSA14}.
The following result specializes the standard argument to our suffix-resampling
scheme. We include the proof for completeness, verifying that generating
proposals early under Condition~\ref{cond:prefetch} does not change the
sequential PowerMH trajectory.
Appendix~\ref{app:tree-traversal} then analyzes how many transitions a
prefetched subtree yields and which requests to prefetch under a request budget.

\begin{proposition}
\label{thm:spec_mh_main}
Fix an integer $K\geq0$. For $k=0,\ldots,K$, let $\mathbf{x}_{\mathrm{pf}}^{(k)}$ and
$\mathbf{x}_{\mathrm{MH}}^{(k)}$ denote the states after $k$ committed
transitions of \method and PowerMH, respectively.
Suppose both samplers start from the same state and \method satisfies
the following assumptions:
\begin{enumerate}[label=(A\arabic*), leftmargin=*]
  \item Prefetched proposals satisfy Condition~\ref{cond:prefetch} and
  are generated from the prefix required by the corresponding decision.
  \item Each decision is assigned cut, suffix-sampling, and acceptance
  random inputs that follow the PowerMH laws and are independent within
  and across decisions. Scheduling and deferral preserve these inputs, and every
  proposal required by the realized trajectory is eventually evaluated.
  \item Every committed transition uses its assigned proposal and the
  exact MH acceptance rule in~\eqref{eq:accept-ratio}. Both children are
  retained as possible outcomes; reaching a leaf postpones subsequent
  decisions to a later batch.
\end{enumerate}
Then a coupling that uses the same random inputs at corresponding
decisions preserves both sampling laws and satisfies
\begin{align*}
\mathbf{x}_{\mathrm{pf}}^{(k)} = \mathbf{x}_{\mathrm{MH}}^{(k)},\qquad \text{ for every } k=0,\ldots,K.
\end{align*}
Consequently, if $P_{\mathrm{MH}}$ denotes the PowerMH transition kernel,
the kernel of \method after $K$ committed transitions is $P_{\mathrm{MH}}^K$.
\end{proposition}

\begin{proof}
Assign independent PowerMH random inputs to each decision of the full
accept/reject tree, and couple the two executions through these inputs.
Whether a decision is reached in the eventual committed trajectory
depends only on its ancestors' inputs, so its own inputs retain their
original conditional law. Inclusion in the current batch may depend on
scheduling; this changes only when the proposal is evaluated.
Preserving deferred inputs and eventually evaluating every required
proposal makes the coupling valid across batch boundaries.

Suppose the executions agree after $k<K$ committed transitions.
By Condition~\ref{cond:prefetch} and the prefix identity established in
Section~\ref{subsec:method-llm}, any intervening accepted proposal preserves
the prefix required by a feasible prefetched request. Rejections leave the
state unchanged. A deferred request is evaluated from the required prefix
when it becomes available. Thus, both executions use the same prefix in the
suffix proposal defined by~\eqref{eq:suffix-q}, and their common suffix-sampling
randomness produces the same candidate.

The current states and candidates therefore coincide, giving the same MH
acceptance probability in~\eqref{eq:accept-ratio}. The shared acceptance uniform
yields the same decision and next state. Since the initial states agree,
induction establishes equality through transition $K$. Each committed
transition consequently has kernel $P_{\mathrm{MH}}$, and the $K$-step
kernel is $P_{\mathrm{MH}}^K$.
\end{proof}

\textbf{Remark (scope).}
The proof uses PowerMH only through the laws of the per-decision random
inputs in (A2) and the acceptance rule in (A3). In PowerMH the cut law is
uniform and does not depend on the state, so every cut in the batch, and the
feasibility check of Condition~\ref{cond:prefetch}, can be drawn before any
proposal is generated.

\subsection{Examples of Prefix Sharing and MH Transitions through the Prefetched Subtree}
\label{subsec:prefetch-examples}

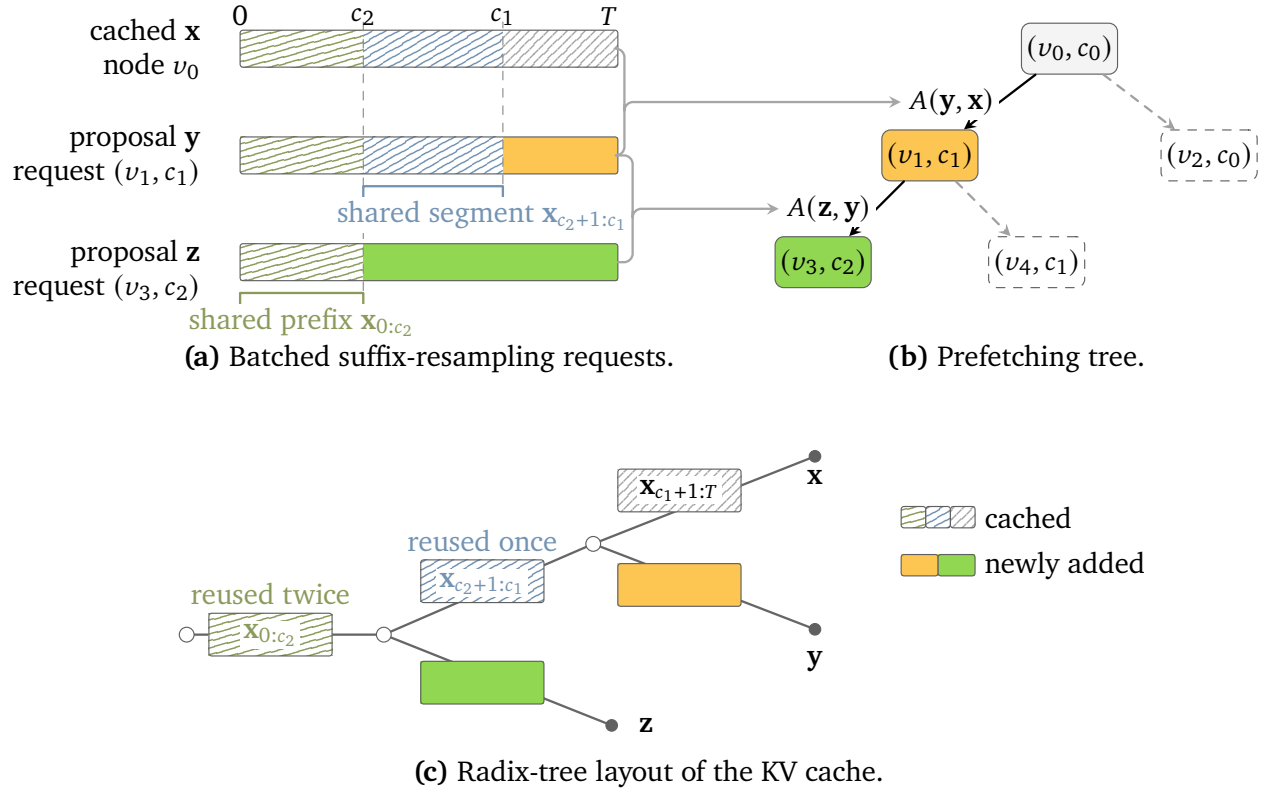
\begin{figure}[!ht]
  \centering
  \resizebox{\linewidth}{!}{%
  \begin{tikzpicture}[paperfig]
    \begin{scope}[yshift=1.625cm]
      \def\rY{-1.3}
      \def\rZ{-2.6}
      \node[seqlabel] at (-0.35,{\rowh/2})
        {cached $\mathbf{x}$\\node $v_0$};
      \node[seqlabel] at (-0.35,{\rY+\rowh/2})
        {proposal $\mathbf{y}$\\ request $(v_1,c_1)$};
      \node[seqlabel] at (-0.35,{\rZ+\rowh/2 -0.1})
        {proposal $\mathbf{z}$\\ request $(v_3,c_2)$};

      \drawSequenceBars
      \drawSequenceReuseBrackets

      \coordinate (xymerge) at ({\barw+0.08},-0.425);
      \draw[NeutralArrowColor,thick,rounded corners=3pt]
        (\barw,{\rowh/2}) -- ({\barw+0.08},{\rowh/2})
        -- ({\barw+0.08},{\rY+\rowh/2})
        -- (\barw,{\rY+\rowh/2});

      \coordinate (yzmerge) at ({\barw+0.18},-1.725);
      \draw[NeutralArrowColor,thick,rounded corners=3pt]
        (\barw,{\rY+\rowh/2}) -- ({\barw+0.18},{\rY+\rowh/2})
        -- ({\barw+0.18},{\rZ+\rowh/2})
        -- (\barw,{\rZ+\rowh/2});

      \node[panel] at ({\barw/2},{\rZ-0.95})
        {\textbf{(a)} Batched suffix-resampling requests.};
    \end{scope}

    \begin{scope}[xshift=10.6cm,yshift=1.85cm]
      \node[treenode,fill=InactiveNodeFillColor] (n0) at (-0.5,0)
        {$(v_0,c_0)$};

      \node[treenode,newAfill] (n1) at (-2.2,-1.3)
        {$(v_1,c_1)$};
      \node[treenode,rejectednode] (n2) at (1.2,-1.3)
        {$(v_2,c_0)$};

      \node[treenode,newBfill] (n3) at (-3.5,-2.6)
        {$(v_3,c_2)$};
      \node[treenode,rejectednode] (n4) at (-0.9,-2.6)
        {$(v_4,c_1)$};

      \draw[rej] (n0) -- (n2);
      \draw[acc] (n0) -- (n1)
        node[edgelabel,midway, left] (p01label) {$A(\mathbf{y},\mathbf{x})$};

      \draw[rej] (n1) -- (n4);
      \draw[acc] (n1) -- (n3)
        node[edgelabel,midway, left] (p13label) {$A(\mathbf{z},\mathbf{y})$};

      \node[panel] at (-1.15,-3.775)
        {\textbf{(b)} Prefetching tree.};
    \end{scope}

    \draw[maparrow,NeutralArrowColor]
      (xymerge) |- (p01label.west);
    \draw[maparrow,NeutralArrowColor]
      (yzmerge) |- (p13label.west);

    \begin{scope}[xshift=4.0cm,yshift=-3.4cm,xscale=1.5,yscale=1.3]
    \def\kvw{1.0}
    \def\kvh{0.4}

    \foreach \sty [count=\i from 0] in {prefixfill,segmentfill,cachedfill} {
      \fill[plate]  ({2.7+0.2*\i},-0.45) rectangle ++(0.2,0.2);
      \fill[\sty]   ({2.7+0.2*\i},-0.45) rectangle ++(0.2,0.2);
      \draw[kvbox]  ({2.7+0.2*\i},-0.45) rectangle ++(0.2,0.2);
    }
    \node[annot,anchor=west] at (3.35,-0.35) {cached};

    \foreach \sty [count=\i from 0] in {newAfill,newBfill} {
      \fill[\sty]   ({2.7+0.3*\i},-0.9) rectangle ++(0.3,0.2);
      \draw[kvbox]  ({2.7+0.3*\i},-0.9) rectangle ++(0.3,0.2);
    }
    \node[annot,anchor=west] at (3.35,-0.8) {newly added};

    \coordinate (root)      at (-3.1,-1.45);
    \coordinate (sharedtwo) at (-1.5,-1.45);
    \coordinate (sharedone) at (0.2,-0.6);
    \coordinate (leaftwo)   at (2,0.22);
    \coordinate (leaffour)  at (2,-1.4);
    \coordinate (leafsix)   at (0.35,-2.3);

    \draw[kvlink] (root) -- (sharedtwo);
    \draw[kvlink] (sharedtwo) -- (leafsix);
    \draw[kvlink] (sharedtwo) -- (sharedone);
    \draw[kvlink] (sharedone) -- (leaftwo);
    \draw[kvlink] (sharedone) -- (leaffour);

    \node[kvpoint] at (root)      {};
    \node[kvpoint] at (sharedtwo) {};
    \node[kvpoint] at (sharedone) {};
    \node[kvleaf]  at (leaftwo)   {};
    \node[kvleaf]  at (leaffour)  {};
    \node[kvleaf]  at (leafsix)   {};

    \node[anchor=north] at (2,0.2)     {$\mathbf{x}$};
    \node[anchor=north] at (2,-1.5)    {$\mathbf{y}$};
    \node[anchor=west]  at (0.48,-2.3) {$\mathbf{z}$};

    \fill[plate]      (-2.92,-1.65) rectangle ++(\kvw,\kvh);
    \fill[prefixfill] (-2.92,-1.65) rectangle ++(\kvw,\kvh);
    \draw[kvbox]      (-2.92,-1.65) rectangle ++(\kvw,\kvh);
    \node[annot,text=SharedPrefixColor] at ({-2.92+\kvw/2},-1.45) {$\mathbf{x}_{0:c_2}$};
    \node[annot,anchor=south,text=SharedPrefixColor] at ({-2.92+\kvw/2},-1.2) {reused twice};

    \fill[plate]       (-1.2,-1.15) rectangle ++(\kvw,\kvh);
    \fill[segmentfill] (-1.2,-1.15) rectangle ++(\kvw,\kvh);
    \draw[kvbox]       (-1.2,-1.15) rectangle ++(\kvw,\kvh);
    \node[annot,text=SegmentColor] at ({-1.2+\kvw/2},-0.97) {$\mathbf{x}_{c_2+1:c_1}$};
    \node[annot,anchor=south,text=SegmentColor] at ({-1.2+\kvw/2},-0.7) {reused once};

    \fill[newBfill] (-1.2,-2.1) rectangle ++(\kvw,\kvh);
    \draw[kvbox]    (-1.2,-2.1) rectangle ++(\kvw,\kvh);

    \fill[plate]      (0.4,-0.3) rectangle ++(\kvw,\kvh);
    \fill[cachedfill] (0.4,-0.3) rectangle ++(\kvw,\kvh);
    \draw[kvbox]      (0.4,-0.3) rectangle ++(\kvw,\kvh);
    \node[annot,text=FigureTextColor] at ({0.4+\kvw/2},-0.1) {$\mathbf{x}_{c_1+1:T}$};

    \fill[newAfill] (0.4,-1.185) rectangle ++(\kvw,\kvh);
    \draw[kvbox]    (0.4,-1.185) rectangle ++(\kvw,\kvh);

    \node[panel] at (0.65,-2.75)
      {\textbf{(c)} Radix-tree layout of the KV cache.};
    \end{scope}
  \end{tikzpicture}%
}
  \vspace{-0.5em}
  \caption{Two batched suffix-resampling requests, their prefetching
  tree, and their shared KV cache.
  \textbf{(a)} Both requests reuse the cached prefix
  \textcolor{SharedPrefixColor}{$\mathbf{x}_{0:c_2}$}, and the request for
  $\mathbf{y}$ also reuses
  \textcolor{SegmentColor}{$\mathbf{x}_{c_2+1:c_1}$}.
  \textbf{(b)} Each acceptance (left) edge carries the MH acceptance
  probability computed from the pair of sequences it connects in (a).
  \textbf{(c)} Radix-tree representation of the shared KV-cache segments.}
  \label{fig:prefetch-example}
\end{figure}

\begin{example}
  \label{ex:prefix-sharing}
  Suppose $\mathbf{x}_{0:T}$ is cached and $c_2\le c_1$.
  In Figure~\ref{fig:prefetch-example}(a), the requests for $\mathbf{y}$ and
  $\mathbf{z}$ condition on $\mathbf{x}_{0:c_1}$ and $\mathbf{x}_{0:c_2}$,
  respectively. Both requests reuse the cached prefix
  \textcolor{SharedPrefixColor}{$\mathbf{x}_{0:c_2}$}; the request for
  $\mathbf{y}$ also reuses
  \textcolor{SegmentColor}{$\mathbf{x}_{c_2+1:c_1}$}.
  Figure~\ref{fig:prefetch-example}(c) represents these shared KV-cache segments
  as a radix tree, as in SGLang~\citep{DBLP:conf/nips/ZhengYXS0YCKSGB24}.
  vLLM~\citep{DBLP:conf/sosp/KwonLZ0ZY0ZS23} instead reuses full, fixed-size KV-cache blocks, indexed by
  hashes that incorporate each block's token IDs and the preceding block's hash.
  Both implementations avoid repeated computation of cached prefix states
  while preserving the proposal distribution and autoregressive decoding of
  each requested suffix.
\end{example}

\begin{example}\label{ex:random-walk}
  After generating the resampled sequences in Example~\ref{ex:acc-ratio-request},
  we associate each prefetched proposal with its MH log-acceptance probability.
  In Figure~\ref{fig:prefetch-example}(a), request $(v_1,c_1)$ produces
  $\mathbf{y}$ by resampling the suffix of the cached state $\mathbf{x}$ after
  cut index $c_1$.
  Using~\eqref{eq:log-acc-ratio} with cut index $c_1$, we compute
  $\log A(\mathbf{y},\mathbf{x})$ and associate it with the acceptance edge
  $v_0\to v_1$ in Figure~\ref{fig:prefetch-example}(b).
  Likewise, request $(v_3,c_2)$ proposes $\mathbf{z}$ from $\mathbf{y}$.
  We compute $\log A(\mathbf{z},\mathbf{y})$ using cut index $c_2$ and
  associate it with the acceptance edge $v_1\to v_3$.

  We simulate the sequential MH decisions by traversing the prefetched subtree
  in Figure~\ref{fig:prefetch-example}(b) from its root.
  At node $v_0$, draw $u\sim\mathrm{Unif}(0,1)$ and select the next node as
  \begin{equation*}
  v_{\mathrm{next}}=
  \begin{cases}
  v_1, & \text{if } u\le A(\mathbf{y},\mathbf{x}),\\
  v_2, & \text{otherwise},
  \end{cases}
  \end{equation*}
  where $A(\mathbf{y},\mathbf{x})$ is the MH acceptance probability of
  $\mathbf{y}$ from $\mathbf{x}$.
  The realization in Figure~\ref{fig:pipeline} accepts $\mathbf{y}$, so the traversal proceeds
  from $v_0$ to $v_1$.
  At $v_1$, draw $u'\sim\mathrm{Unif}(0,1)$ independently.
  In this realization, $u'>A(\mathbf{z},\mathbf{y})$, so the second proposal $\mathbf{z}$ is rejected, yielding
  the tree path $v_0\to v_1\to v_4$ and the corresponding MH trajectory
  \begin{equation*}
  \mathbf{x}^{(0)}=\mathbf{x}
  \to \mathbf{x}^{(1)}=\mathbf{y}
  \to \mathbf{x}^{(2)}=\mathbf{y}.
  \end{equation*}
  Together with Examples~\ref{example:k2-traj},~\ref{ex:compose-cuts},
  and~\ref{ex:acc-ratio-request}, this example shows how \method batches two suffix-resampling
  requests into one model call while preserving sequential MH decisions.
  Along the illustrated path, this batched call realizes two MH transitions that
  would require two serial model calls in PowerMH.
\end{example}

\newpage

\section{Prefetch Tree Selection under a Request Budget}
\label{app:tree-traversal}

Appendix~\ref{app:realized-prefetched-path} fixes the notation for the prefetching tree.
Appendix~\ref{subsec:transition-count-objective} derives the expected number of MH transitions completed by a prefetched subtree.
Appendix~\ref{subsec:predictive-dp} turns this count into an allocation objective with predicted branch probabilities, and Appendix~\ref{subsec:dp-allocation} solves it by dynamic programming.

\subsection{Setup and Notation}
\label{app:realized-prefetched-path}

We use the prefetching tree of Section~\ref{subsec:method-llm}, with nodes $v_0,v_1,\ldots$ and root $v_0$ holding the current MH state.
When only the parent--child relation matters, we write $v$ for a node and $v^{+}$ and $v^{-}$ for its acceptance and rejection children, as in Section~\ref{subsec:efficiency}.
Suppose that $k$ of the $K$ transitions of the current block are complete (Algorithm~\ref{alg:subtree-mh}), and let $K'=K-k$ be the number that remain.
We cap the tree at depth $K'$ and write $\mathcal T_{K'}$ for the resulting finite tree.

\begin{itemize}[leftmargin=*]
    \item \emph{Expandable node.} A node $v_i\in\mathcal T_{K'}$ with $d(v_i)<K'$ is expandable if the request at its acceptance child $v_j$ satisfies Condition~\ref{cond:prefetch}, that is, $c_j\leq c_i$, where $c_i$ and $c_j$ are the cut indices of $v_i$ and $v_j$.
    \item \emph{Candidate node.} An expandable node whose strict ancestors are all expandable; $\mathcal I_{\max}$ denotes the set of candidate nodes. Once the proposal at $v^{+}$ is generated, the chain moves from a candidate node $v$ to $v^{+}$ if the proposal is accepted and to $v^{-}$ if it is rejected.
    \item \emph{Ancestor-closed.} A set $\mathcal I\subseteq\mathcal I_{\max}$ is ancestor-closed if, for every $v\in\mathcal I$, every strict ancestor of $v$ also lies in $\mathcal I$.
    \item \emph{Selection.} A selection is any ancestor-closed $\mathcal I\subseteq\mathcal I_{\max}$; its elements are the expanded nodes.
    \item \emph{Prefetched subtree.} The root $v_0$ together with both children of every $v\in\mathcal I$. Its internal nodes are exactly the nodes in $\mathcal I$, and its leaves are its nodes outside $\mathcal I$; for $\mathcal I=\varnothing$, it consists of $v_0$ alone, which is then its only leaf.
\end{itemize}
Because a selection is ancestor-closed, a selection of at most $N_{\mathrm{pf}}$ nodes contains no node deeper than $N_{\mathrm{pf}}-1$, so the depth cap is inactive when $K'\geq N_{\mathrm{pf}}$.

\begin{definition}
\label{def:finite-tree-continuation}
Fix a realization of the proposals and of their acceptance probabilities, with $\epsilon\in(0,\tfrac12)$ as in Section~\ref{subsec:efficiency}.

\textbf{Rounded branch probabilities.} At each node $v$,
\begin{equation}
\label{eq:rounded-branch-probability}
\begin{aligned}
h(v^{+};\epsilon)&:=
\begin{cases}
0,&A(v^{+},v)\leq\epsilon,\\
1,&A(v^{+},v)\geq1-\epsilon,\\
A(v^{+},v),&\textup{otherwise},
\end{cases}\\
h(v^{-};\epsilon)&:=1-h(v^{+};\epsilon).
\end{aligned}
\end{equation}
\textbf{Traversal.} Fix a set $\mathcal I$ of nodes of $\mathcal T_{K'}$, for example a selection.
A traversal follows one child of each node of $\mathcal I$ that it reaches: from a reached node $v\in\mathcal I$, it moves to $v^{+}$ with probability $h(v^{+};\epsilon)$ and to $v^{-}$ with probability $h(v^{-};\epsilon)$, independently across nodes, and it stops at the first node outside $\mathcal I$.
The traversal \emph{completes a transition} at each node of $\mathcal I$ that it reaches, whichever child it then moves to.
Write $F_{\mathcal I}(v;\epsilon)$ for the expected number of transitions completed by a traversal started at $v$.

\textbf{Reaching weights.} The reaching weight of a node is the product of the branch probabilities along its path from the root:
\begin{equation}
\label{eq:reaching-weight}
R(v_0;\epsilon)=1,\qquad
R(v^{\pm};\epsilon)=R(v;\epsilon)\,h(v^{\pm};\epsilon)
\quad\textup{for every node }v.
\end{equation}
The traversal, $F_{\mathcal I}$, and $R$ use only branch probabilities in $[0,1]$ with $h(v^{-};\epsilon)=1-h(v^{+};\epsilon)$; the lemmas below therefore also hold for any such branch probabilities, in particular for the predicted $\widehat h$ of Appendix~\ref{subsec:predictive-dp}.
\end{definition}

\subsection{The Transition-Count Objective}
\label{subsec:transition-count-objective}

Following prior prefetching analyses \citep{strid2010prefetching}, we score a selection $\mathcal I$ by the expected number of MH transitions it completes and by the number of requests it issues.
By Definition~\ref{def:finite-tree-continuation}, a traversal completes one transition at each node of $\mathcal I$ that it visits, whether the proposal there is accepted or rejected.
Lemma~\ref{lem:prefetch-expectation} writes the expected count from the root, $F_{\mathcal I}(v_0;\epsilon)$, as a sum of reaching weights, and Lemma~\ref{lem:prefetch-recursion} gives the recursive form used by the dynamic program.

\begin{lemma}[Expected transition count]
\label{lem:prefetch-expectation}
For any ancestor-closed set $\mathcal I$ of nodes of $\mathcal T_{K'}$ (in particular, any selection), with reaching weights $R(v;\epsilon)$ as in~\eqref{eq:reaching-weight},
\begin{equation}
\label{eq:prefetch-transition-count}
F_{\mathcal I}(v_0;\epsilon)=\sum_{v\in\mathcal I}R(v;\epsilon).
\end{equation}
\end{lemma}

\begin{proof}
A traversal from $v_0$ follows a single path, so it visits each node at most once, and its transition count equals the number of nodes in $\mathcal I$ it visits.
By linearity of expectation,
\[
F_{\mathcal I}(v_0;\epsilon)
=
\sum_{v\in\mathcal I}
P(\text{the traversal visits }v).
\]
Fix $v\in\mathcal I$.
By ancestor closure, every strict ancestor of $v$ lies in $\mathcal I$, so the traversal cannot stop on the path from $v_0$ to $v$.
It therefore visits $v$ exactly when it takes every branch on that path, an event whose probability is the product of the branch probabilities along the path, that is, $R(v;\epsilon)$ by~\eqref{eq:reaching-weight}.
Summing over $v\in\mathcal I$ gives~\eqref{eq:prefetch-transition-count}; for $\mathcal I=\varnothing$, both sides are zero.
\end{proof}

\begin{lemma}[Recursive form]
\label{lem:prefetch-recursion}
For any set $\mathcal I$ of nodes of $\mathcal T_{K'}$, $F_{\mathcal I}(v;\epsilon)=0$ for $v\notin\mathcal I$, and
\begin{equation}
\label{eq:prefetch-continuation}
\begin{aligned}
F_{\mathcal I}(v;\epsilon)
&=1+h(v^{+};\epsilon)F_{\mathcal I}(v^{+};\epsilon)+h(v^{-};\epsilon)F_{\mathcal I}(v^{-};\epsilon),
\qquad v\in\mathcal I.
\end{aligned}
\end{equation}
\end{lemma}

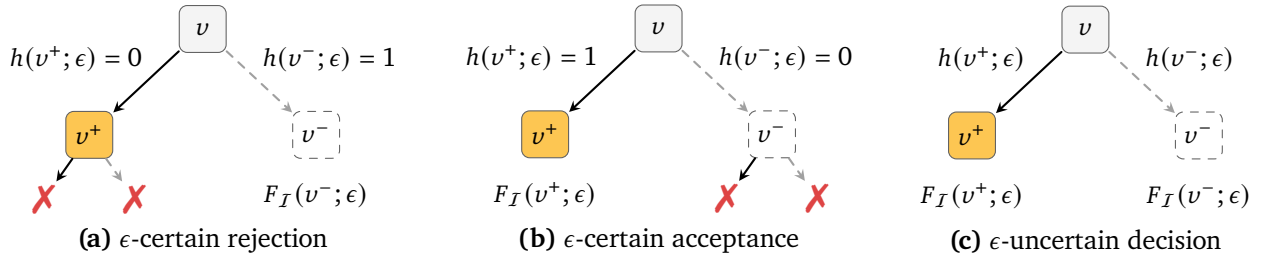
\begin{figure}[!ht]
  \centering
    \begin{tikzpicture}[paperfig]
    \node[treenode,fill=InactiveNodeFillColor] (v) at (0,0.35) {$v$};
    \node[treenode,newAfill] (vacc) at (-1.5,-1.05) {$v^{+}$};
    \node[treenode,rejectednode] (vrej) at (1.5,-1.05) {$v^{-}$};

    \draw[acc] (v) -- (vacc)
      node[edgelabel,pos=0.4,anchor=south east,xshift=-2pt]
      {$h(v^{+};\epsilon)=0$};
    \draw[rej] (v) -- (vrej)
      node[edgelabel,pos=0.4,anchor=south west,xshift=2pt]
      {$h(v^{-};\epsilon)=1$};

    \node[annot,color=PrunedMarkColor,fill=FigureBackgroundColor] (vacc_l)
      at (-2.1,-1.9) {\scalebox{1.5}{\xmark}};
    \node[annot,color=PrunedMarkColor,fill=FigureBackgroundColor] (vacc_r)
      at (-0.9,-1.9) {\scalebox{1.5}{\xmark}};
    \draw[acc] (vacc) -- (vacc_l);
    \draw[rej] (vacc) -- (vacc_r);
    \node[annot] at (1.5,-1.85) {$F_{\mathcal I}(v^{-};\epsilon)$};

    \node[panel] at (0,-2.45)
      {\textbf{(a)} $\epsilon$-certain rejection};
  \end{tikzpicture}
  \hfill
  \begin{tikzpicture}[paperfig]
    \node[treenode,fill=InactiveNodeFillColor] (v) at (0,0.35) {$v$};
    \node[treenode,newAfill] (vacc) at (-1.5,-1.05) {$v^{+}$};
    \node[treenode,rejectednode] (vrej) at (1.5,-1.05) {$v^{-}$};

    \draw[acc] (v) -- (vacc)
      node[edgelabel,pos=0.4,anchor=south east,xshift=-2pt]
      {$h(v^{+};\epsilon)=1$};
    \draw[rej] (v) -- (vrej)
      node[edgelabel,pos=0.4,anchor=south west,xshift=2pt]
      {$h(v^{-};\epsilon)=0$};

    \node[annot] at (-1.5,-1.85) {$F_{\mathcal I}(v^{+};\epsilon)$};
    \node[annot,color=PrunedMarkColor,fill=FigureBackgroundColor] (vrej_l)
      at (0.9,-1.9) {\scalebox{1.5}{\xmark}};
    \node[annot,color=PrunedMarkColor,fill=FigureBackgroundColor] (vrej_r)
      at (2.1,-1.9) {\scalebox{1.5}{\xmark}};
    \draw[acc] (vrej) -- (vrej_l);
    \draw[rej] (vrej) -- (vrej_r);

    \node[panel] at (0,-2.45)
      {\textbf{(b)} $\epsilon$-certain acceptance};
  \end{tikzpicture}
  \hfill
  \begin{tikzpicture}[paperfig]
    \node[treenode,fill=InactiveNodeFillColor] (v) at (0,0.35) {$v$};
    \node[treenode,newAfill] (vacc) at (-1.5,-1.05) {$v^{+}$};
    \node[treenode,rejectednode] (vrej) at (1.5,-1.05) {$v^{-}$};

    \draw[acc] (v) -- (vacc)
      node[edgelabel,pos=0.4,anchor=south east,xshift=-2pt]
      {$h(v^{+};\epsilon)$};
    \draw[rej] (v) -- (vrej)
      node[edgelabel,pos=0.4,anchor=south west,xshift=2pt]
      {$h(v^{-};\epsilon)$};

    \node[annot] at (-1.5,-1.85) {$F_{\mathcal I}(v^{+};\epsilon)$};
    \node[annot] at (1.5,-1.85) {$F_{\mathcal I}(v^{-};\epsilon)$};

    \node[panel] at (0,-2.45)
      {\textbf{(c)} $\epsilon$-uncertain decision};
  \end{tikzpicture}
  \caption{Visualization of
  \eqref{eq:prefetch-continuation}: (a) certain rejection,
  (b) certain acceptance, and (c) an uncertain rounded decision.
  A \textcolor{PrunedMarkColor}{\xmark} marks a pruned edge.
  The proposal at $v^{+}$ is required in every case.}
  \label{fig:finite-tree-continuation-recurrence}
\end{figure}

\begin{proof}
A traversal started at $v\notin\mathcal I$ stops immediately and completes no transition.
A traversal started at $v\in\mathcal I$ completes one transition at $v$ and, since the choices at different nodes are independent, then proceeds as a traversal started at the child it moves to.
By~\eqref{eq:rounded-branch-probability}, the MH decision at $v$ falls into one of three cases (Figure~\ref{fig:finite-tree-continuation-recurrence}):

\begin{itemize}[leftmargin=*]
    \item \emph{Certain rejection} ($h(v^{+};\epsilon)=0$, panel~(a)). The traversal moves to $v^{-}$ with probability one, so $F_{\mathcal I}(v;\epsilon)=1+F_{\mathcal I}(v^{-};\epsilon)$, which is~\eqref{eq:prefetch-continuation} with a vanishing $v^{+}$ term.
    \item \emph{Certain acceptance} ($h(v^{+};\epsilon)=1$, panel~(b)). The traversal moves to $v^{+}$ with probability one, so $F_{\mathcal I}(v;\epsilon)=1+F_{\mathcal I}(v^{+};\epsilon)$, which is~\eqref{eq:prefetch-continuation} with a vanishing $v^{-}$ term.
    \item \emph{Uncertain decision} ($0<h(v^{+};\epsilon)<1$, panel~(c)). The traversal moves to $v^{+}$ with probability $h(v^{+};\epsilon)$ and to $v^{-}$ with probability $h(v^{-};\epsilon)$; conditioning on this outcome gives~\eqref{eq:prefetch-continuation}.
\end{itemize}
\end{proof}

Lemma~\ref{lem:prefetch-recursion} is the basis of the dynamic program in Proposition~\ref{prop:predictive-dp-optimality}: by~\eqref{eq:prefetch-continuation}, $F_{\mathcal I}(v;\epsilon)$ depends on $\mathcal I$ only through its parts below $v^{+}$ and $v^{-}$, so an optimal selection below $v$ combines optimal selections below its two children.

\textbf{Remark.}
The request count depends only on $\mathcal I$, not on the traversal: expanding a node $v$ issues exactly one request, for the proposal at $v^{+}$, while the rejection child $v^{-}$ keeps the state of $v$ and needs no proposal.
A selection $\mathcal I$ thus issues $|\mathcal I|$ requests and yields a prefetched subtree of $2|\mathcal I|+1$ nodes.
Rounding does not lower this cost: at a certain rejection, $h(v^{+};\epsilon)=0$ makes the acceptance branch unreachable, yet expanding $v$ still costs the request for the proposal at $v^{+}$. Rounding only prunes the low-probability branch below $v$, so that no further requests are spent there.

\subsection{Predictive Subtree Allocation}
\label{subsec:predictive-dp}

We want to expand the selection $\mathcal I$ that maximizes the expected transition count $F_{\mathcal I}(v_0;\epsilon)$ within the request budget.
This count depends on the branch probabilities $h(v^{\pm};\epsilon)$, which by~\eqref{eq:rounded-branch-probability} are determined by the acceptance probability $A(v^{+},v)$.
However, $A(v^{+},v)$ is known only after the proposal at $v^{+}$ has been generated and scored, so it is unavailable while we assemble the request batch.
We therefore use a classifier that, from information available before generation, predicts the probability that each candidate node is an $\epsilon$-certain reject, an $\epsilon$-certain accept, or $\epsilon$-uncertain (Appendix~\ref{apx:acceptance-predictors} reports its empirical results).
Below, we convert these class probabilities into predicted branch probabilities $\widehat h$ and formulate the allocation objective with $\widehat h$ in place of $h$.

\textbf{Notation.}
Throughout this subsection, $A:=A(v^{+},v)$ denotes the acceptance probability of the proposal at $v^{+}$, viewed as a random variable, and conditioning on $v$ means conditioning on the information available before allocation.
As in Section~\ref{subsec:efficiency}, $v$ is an $\epsilon$-certain reject if $A\leq\epsilon$, an $\epsilon$-certain accept if $A\geq1-\epsilon$, and $\epsilon$-uncertain otherwise; $P(v\text{ is $\epsilon$-certain reject})$ denotes the probability of this event given the information available before allocation, and similarly for the other two classes.

\textbf{Computing $\widehat h$ from the classifier.}
For each candidate node $v$, the classifier takes the features of $v$ that are available before generation (in our implementation, the cut fraction of the request at $v^{+}$) and outputs, through a softmax (Appendix~\ref{apx:acceptance-predictors}), a predicted probability for each of the three classes rather than only the most likely class:
\begin{gather*}
\widehat P(v\text{ is $\epsilon$-certain reject}),\qquad
\widehat P(v\text{ is $\epsilon$-certain accept}),\qquad\text{and}\\
\widehat P(v\text{ is $\epsilon$-uncertain})=1-\widehat P(v\text{ is $\epsilon$-certain reject})-\widehat P(v\text{ is $\epsilon$-certain accept}).
\end{gather*}
We then compute the \emph{predicted branch probabilities} as
\begin{equation}
\label{eq:predictive-rounded-branch-probabilities}
\begin{aligned}
\widehat h(v^{+};\epsilon)
&=\tfrac12\bigl[1+\widehat P(v\text{ is $\epsilon$-certain accept})
                 -\widehat P(v\text{ is $\epsilon$-certain reject})\bigr],\\
\widehat h(v^{-};\epsilon)&=1-\widehat h(v^{+};\epsilon).
\end{aligned}
\end{equation}
That is, $\widehat h(v^{+};\epsilon)$ is the average of $1$ (certain accept), $0$ (certain reject), and $\tfrac12$ (uncertain, the midpoint of $[\epsilon,1-\epsilon]$), weighted by the predicted class probabilities; it is the value $h^{\star}(v^{+};\epsilon)$ of Lemma~\ref{lem:midpoint-branch-probability} with the predicted class probabilities in place of the exact ones.

\begin{lemma}[Midpoint branch probability]
\label{lem:midpoint-branch-probability}
Fix a node $v$. Its exact class probabilities bound the mean rounded branch probability:
\begin{equation}
\label{eq:mean-rounded-factor-interval}
\begin{aligned}
\mathbb E[h(v^{+};\epsilon)\mid v]
&\geq P(v\textup{ is $\epsilon$-certain accept})+\epsilon\,P(v\textup{ is $\epsilon$-uncertain}),\\
\mathbb E[h(v^{+};\epsilon)\mid v]
&\leq P(v\textup{ is $\epsilon$-certain accept})+(1-\epsilon)\,P(v\textup{ is $\epsilon$-uncertain}).
\end{aligned}
\end{equation}
Define the \emph{error} of a candidate value $x\in[0,1]$ as
$\bigl|x-\mathbb E[h(v^{+};\epsilon)\mid v]\bigr|$,
and its \emph{worst-case error} as the largest such error over all means
allowed by \eqref{eq:mean-rounded-factor-interval}. The worst-case error
is smallest at the midpoint of this interval,
\begin{equation}
\label{eq:midpoint-choice}
h^{\star}(v^{+};\epsilon)
=\tfrac12\bigl[1+P(v\textup{ is $\epsilon$-certain accept})-P(v\textup{ is $\epsilon$-certain reject})\bigr].
\end{equation}
\end{lemma}

\begin{proof}
By~\eqref{eq:rounded-branch-probability}, $h(v^{+};\epsilon)$ is $1$ on the certain-accept class, $0$ on the certain-reject class, and $A$ on the uncertain class, so
\[
\mathbb E[h(v^{+};\epsilon)\mid v]
=P(v\text{ is $\epsilon$-certain accept})
+P(v\text{ is $\epsilon$-uncertain})\,\mathbb E\bigl[A\bigm|v\text{ is $\epsilon$-uncertain}\bigr].
\]
The conditional mean of $A$ on the uncertain class lies in $[\epsilon,1-\epsilon]$, which gives~\eqref{eq:mean-rounded-factor-interval}.
For a quantity known only to lie in an interval $[a,b]$, an estimate $x$ has worst-case error $\max\{|x-a|,|x-b|\}\geq(b-a)/2$, with equality exactly at $x=(a+b)/2$.
The midpoint of the interval in~\eqref{eq:mean-rounded-factor-interval} is
\[
\begin{aligned}
&P(v\text{ is $\epsilon$-certain accept})+\tfrac12\,P(v\text{ is $\epsilon$-uncertain})\\
&\qquad=\tfrac12\bigl[1+P(v\text{ is $\epsilon$-certain accept})-P(v\text{ is $\epsilon$-certain reject})\bigr],
\end{aligned}
\]
where the equality uses that the three class probabilities sum to one. This is~\eqref{eq:midpoint-choice}.
\end{proof}

\textbf{Objective formulation.}
With $\widehat h$ in place of $h$, Lemma~\ref{lem:prefetch-expectation} gives the allocation objective: choose the selection $\mathcal I$ that maximizes the predicted expected transition count within the request budget,
\begin{equation}
\label{eq:prefetch-selection-objective}
\begin{aligned}
\max_{\mathcal I}\quad
&\widehat F_{\mathcal I}(v_0;\epsilon)
 =\sum_{v\in\mathcal I}\widehat R(v;\epsilon)\\
\text{s.t.}\quad
&\mathcal I\subseteq\mathcal I_{\max}\ \text{ancestor-closed},\\
&|\mathcal I|\leq N_{\mathrm{pf}},
\end{aligned}
\end{equation}
where
\begin{itemize}
\item $\widehat R$ and $\widehat F$ are the reaching weights and the expected transition count of Definition~\ref{def:finite-tree-continuation}, computed with the predicted branch probabilities $\widehat h$ in~\eqref{eq:predictive-rounded-branch-probabilities} in place of the unknown $h$;
\item the candidate set $\mathcal I_{\max}$, determined by Condition~\ref{cond:prefetch} and the depth cap $K'$, ensures that every selected request can be generated in a single batch;
\item ancestor closure requires every ancestor of a selected node to be selected as well.
\end{itemize}
A node with $\widehat R(v;\epsilon)=0$ would contribute nothing to the objective but still cost a request; this requires some $\widehat h$ on its path to be exactly $0$ or $1$, which our softmax classifier never outputs, since all three of its class probabilities are positive.

\textbf{Why rounded estimates.}
The exact count $F_{\mathcal I}(v_0;0)$ requires the value of $A$, which is unknown until the proposal is generated, and an estimate of $A$ itself would carry error at every decision.
Rounding removes the need to estimate $A$ on the $\epsilon$-certain decisions, at a bias of at most $\epsilon$ per branch probability, but it does not remove the error of the class probabilities themselves.
On the $\epsilon$-certain decisions, $h(v^{+};\epsilon)$ is exactly $0$ or $1$ and is determined by the class alone; on the $\epsilon$-uncertain ones, given exact class probabilities, the midpoint estimate is off by at most $\tfrac12(1-2\epsilon)\,P(v\text{ is $\epsilon$-uncertain})$ (Lemma~\ref{lem:midpoint-branch-probability}).
Since only $5\%$--$33\%$ of decisions are $\epsilon$-uncertain (Figure~\ref{fig:epsilon-certainty-datasets}), the midpoint error affects a minority of decisions; the error of the classifier adds to it on every decision.

\subsection{Optimal Budget Allocation via Dynamic Programming}
\label{subsec:dp-allocation}

Define $V(v,b;\epsilon)$ as the largest predicted transition count attainable from $v$ with at most $b$ requests,
\begin{equation}
\label{eq:predictive-dp-value}
V(v,b;\epsilon)
=
\max_{\substack{\mathcal I_v\text{ feasible below }v,\\
                |\mathcal I_v|\leq b}}
\widehat F_{\mathcal I_v}(v;\epsilon),
\end{equation}
where a set of candidate nodes $\mathcal I_v$ is \emph{feasible below $v$} if
$\mathcal I_v\subseteq\mathcal I_{\max}$, every element of
$\mathcal I_v$ lies in the subtree rooted at $v$, and every candidate node on
the path from $v$ to an element of $\mathcal I_v$ also lies in
$\mathcal I_v$.

The optimization in \eqref{eq:predictive-dp-value} decomposes the budget at each node: one request expands $v$ itself, and the remaining $b-1$ are split between its two child subtrees before the outcome at $v$ is known (Figure~\ref{fig:dp-budget-split}).
This decomposition is inspired by the tree-based speculative decoding of Sequoia~\citep{chen2024sequoia}.

\begin{proposition}[Optimal predictive allocation]
\label{prop:predictive-dp-optimality}
Fix $\epsilon$ and the predicted branch probabilities
$\widehat h(v^{+};\epsilon)$ and $\widehat h(v^{-};\epsilon)$ in~\eqref{eq:predictive-rounded-branch-probabilities}. For every node $v$ and every budget $b\geq0$,
\begin{equation}
\label{eq:predictive-dp-bellman}
V(v,b;\epsilon)=
\begin{cases}
0, & \textup{if } v\notin\mathcal I_{\max},\\[2pt]
0, & \textup{if } b=0,\\[2pt]
1+\max\limits_{0\leq m\leq b-1}\big\{
 \widehat h(v^{+};\epsilon)V(v^{+},m;\epsilon)+\widehat h(v^{-};\epsilon)V(v^{-},b-1-m;\epsilon)\big\},
& \textup{otherwise}.
\end{cases}
\end{equation}
The maximum in
\eqref{eq:prefetch-selection-objective} equals
$V(v_0,N_{\mathrm{pf}};\epsilon)$.
\end{proposition}

\begin{figure}[!ht]
  \centering
  \begin{tikzpicture}[paperfig,
  dpnode/.style={treenode,minimum width=0.8cm},
  accsubtree/.style={draw=TreeEdgeColor,fill=NewSuffixAFillColor!20,thick},
  rejsubtree/.style={draw=RejectEdgeColor,fill=FigureBackgroundColor,thick,dashed}]
  \draw[accsubtree] (-4.6,-1.3) -- (-6.6,-2.6) -- (-2.6,-2.6) -- cycle;
  \draw[rejsubtree] (4.6,-1.3) -- (2.6,-2.6) -- (6.6,-2.6) -- cycle;
  \node[dpnode,fill=InactiveNodeFillColor,inner ysep=2.5pt] (v) at (0,0)
    {\vphantom{h}$v$\\one request};
  \node at (0,0.75) {$n$ available requests};
  \node[dpnode,newAfill] (l) at (-4.6,-1.3) {$v^{+}$};
  \node[dpnode,rejectednode,fill=FigureBackgroundColor] (r) at (4.6,-1.3) {$v^{-}$};
  \draw[acc] (v) --
    node[edgelabel,anchor=south east,align=right,xshift=-2pt]
      {accept with predicted\\probability $\widehat h(v^{+};\epsilon)$} (l);
  \draw[rej] (v) --
    node[edgelabel,anchor=south west,align=left,xshift=2pt]
      {reject with predicted\\probability $\widehat h(v^{-};\epsilon)$} (r);
  \node[anchor=base] at (-4.6,-2.45) {left subtree};
  \node[anchor=base] at (4.6,-2.45) {right subtree};
  \node[align=center] at (-4.6,-3.0)
    {$m$ continuation requests};
  \node[align=center] at (4.6,-3.0)
    {$n-1-m$ continuation requests};
\end{tikzpicture}
  \caption{Budget decomposition visualization for Proposition~\ref{prop:predictive-dp-optimality}.}
  \label{fig:dp-budget-split}
\end{figure}
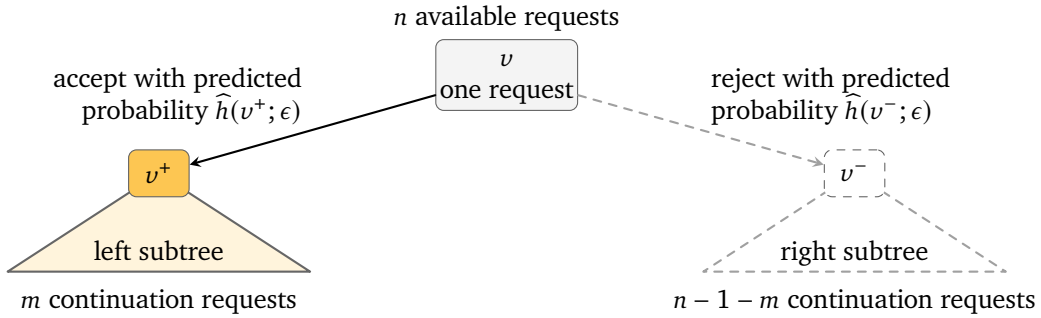

\begin{proof}
\emph{Boundary cases.}
A nonempty selection feasible below $v$ contains $v$, so it needs both
$v\in\mathcal I_{\max}$ and $b\geq1$. Either boundary case therefore
leaves only the empty selection, giving $V(v,b;\epsilon)=0$.
Fix $v\in\mathcal I_{\max}$ and $b\geq1$ below.

\emph{Upper bound.}
Expanding $v$ alone has value one by Lemma~\ref{lem:prefetch-recursion}, so
an optimal $\mathcal I_v$ is nonempty and hence contains $v$, which costs
one request. Its remaining nodes lie in the two child
subtrees, and its restriction to each subtree is feasible below that child. If
$m$ of them lie below $v^{+}$ and at most $b-1-m$ below $v^{-}$, where
$0\leq m\leq b-1$, then Lemma~\ref{lem:prefetch-recursion} with $\widehat h$ in place of $h$ and
the optimality of each child value
give
\[
\widehat F_{\mathcal I_v}(v;\epsilon)
\leq1+\widehat h(v^{+};\epsilon)V(v^{+},m;\epsilon)
     +\widehat h(v^{-};\epsilon)V(v^{-},b-1-m;\epsilon).
\]
Maximizing over $m$ gives the upper bound in
\eqref{eq:predictive-dp-bellman}.

\emph{Attainability.}
Conversely, combine a maximizing $m$ with optimal selections below $v^{+}$
and $v^{-}$ at capacities $m$ and $b-1-m$. Together with the expansion of
$v$ they are feasible below $v$ and use at most $b$ requests, and
Lemma~\ref{lem:prefetch-recursion} gives them exactly the maximized value.
The bound is attained, which proves \eqref{eq:predictive-dp-bellman};
induction from the leaves of the finite tree $\mathcal T_{K'}$ then
determines every $V(v,b;\epsilon)$.

Finally, feasibility below $v_0$ is ancestor closure within
$\mathcal I_{\max}$, so $V(v_0,N_{\mathrm{pf}};\epsilon)$ is the
optimum of \eqref{eq:prefetch-selection-objective}.
\end{proof}

\textbf{Remark.}
Proposition~\ref{prop:predictive-dp-optimality} requires the predicted branch probabilities $\widehat h$ at every candidate node.
Estimating them reduces to classification (Appendix~\ref{apx:acceptance-predictors}): by~\eqref{eq:rounded-branch-probability}, $h(v^{+};\epsilon)$ is exactly zero or one on the two $\epsilon$-certain classes, so a classifier only needs to predict the class probabilities of each candidate node, and Lemma~\ref{lem:midpoint-branch-probability} converts them into $\widehat h(v^{+};\epsilon)$ without modeling $A$ within the uncertain class.

\newpage

\section{Experimental Settings}
\label{apx:exp-set}

\subsection{Baselines, Datasets, and Base LLMs}

\noindent\textbf{MH-based samplers.}
We apply \framework to three MH-based samplers that target the same power distribution, and compare each sampler with its \framework variant:
\begin{itemize}[leftmargin=*, nosep]
    \item \textbf{PowerMH}~\citep{karan2026reasoning} builds a continuation block by block. At each block, it extends the current prefix with the proposal LLM and then runs several MH updates that target the power-sharpened distribution over the extended prefix. Each update draws a cut position uniformly at random, keeps the tokens before it, resamples the suffix from the proposal LLM, and accepts or rejects the candidate with the MH correction.

    \item \textbf{EntropyCut}~\citep{entropycut} modifies PowerMH by choosing cut positions where the base model's next-token entropy increases; these entropy jumps serve as a proxy for consequential reasoning decisions. It regenerates the suffix from the chosen position and applies the MH correction, so the power target is preserved.

    \item \textbf{MultiTryMH} replaces each single-proposal PowerMH update with a Multiple-Try Metropolis update~\citep{liu2000multipletry}. We use a simplified version of Entropy-Guided Power Sampling (EGPS)~\citep{powermh-mtm}, namely the ``EGPS (MultiTryMH, with random pos)'' variant in their Figure~3: we omit entropy-guided cut selection and draw the cut position uniformly at random.
    This baseline preserves the power target while considering several candidate suffixes in each MH transition.
\end{itemize}

\begin{table}[!ht]
    \centering
    \caption{Datasets and their full sizes; every run uses only the first 20 prompts of each dataset.}
    \label{tab:dataset-summary}
    \small
    \newcommand{\hf}[2][datasets/]{{\footnotesize\href{https://huggingface.co/#1#2}{\nolinkurl{#2}}}}
    \begin{tabular}{@{}lllr@{}}
        \toprule
        \textbf{Category} & \textbf{Dataset} & \textbf{Source (Hugging Face)} & \textbf{Instances} \\
        \midrule
        \multirow{2}{*}{\makecell[l]{Mathematical reasoning}} & MATH500 & \hf{HuggingFaceH4/MATH-500} & 500 \\
         & AIME 24\&25 & \makecell[l]{\hf{Maxwell-Jia/AIME_2024}\\\hf{MathArena/aime_2025}} & 45 \\
        \midrule
        \multirow{3}{*}{\makecell[l]{Code generation}} & HumanEval & \hf{openai/openai_humaneval} & 164 \\
         & MBPP & \hf{google-research-datasets/mbpp} & 257 \\
         & LCB V6 & \hf{livecodebench/code_generation_lite} & 175 \\
        \midrule
        \multirow{2}{*}{\makecell[l]{Scientific and general knowledge}} & GPQA & \hf{Idavidrein/gpqa} & 198 \\
         & MMLU & \hf{cais/mmlu} & 14{,}042 \\
        \bottomrule
    \end{tabular}
\end{table}

\noindent\textbf{Datasets.}
We evaluate on seven datasets in three task categories (Table~\ref{tab:dataset-summary}): mathematical reasoning (MATH500~\citep{lightman2024lets}, a 500-problem subset of MATH~\citep{DBLP:conf/nips/HendrycksBKABTS21}, and AIME 24\&25), code generation (HumanEval~\citep{chen2021evaluatingllmcode}, MBPP~\citep{mbpp_dataset}, and LiveCodeBench V6 (LCB V6)~\citep{DBLP:conf/iclr/JainHGLYZWSSS25}), and scientific and general knowledge (GPQA Diamond (GPQA)~\citep{conf/colm/rein2024gpqa} and MMLU~\citep{DBLP:conf/iclr/HendrycksBBZMSS21}).

\begin{table}[!ht]
    \centering
    \caption{Base LLMs and their sources.}
    \label{tab:base-llm-summary}
    \small
    \newcommand{\hf}[1]{{\footnotesize\href{https://huggingface.co/#1}{\nolinkurl{#1}}}}
    \begin{tabular}{@{}llll@{}}
        \toprule
        \textbf{Family} & \textbf{Base LLM} & \textbf{Source (Hugging Face)} & \textbf{Released} \\
        \midrule
        \multirow{2}{*}{Qwen3} & Qwen3-4B & \hf{Qwen/Qwen3-4B} & Apr 2025 \\
         & Qwen3-8B & \hf{Qwen/Qwen3-8B} & Apr 2025 \\
        \midrule
        \multirow{2}{*}{Qwen3.5} & Qwen3.5-4B & \hf{Qwen/Qwen3.5-4B} & Feb 2026 \\
         & Qwen3.5-9B & \hf{Qwen/Qwen3.5-9B} & Feb 2026 \\
        \midrule
        Gemma4 & Gemma4-12B-it & \hf{google/gemma-4-12B-it} & May 2026 \\
        \bottomrule
    \end{tabular}
\end{table}

\noindent\textbf{Base LLMs.}
We evaluate five base LLMs (Table~\ref{tab:base-llm-summary}), where ``base LLM'' denotes the backbone \(p_0\).
Each is a recent open-weight release from a widely adopted model family: two parameter scales from each of the Qwen3 and Qwen3.5 generations, and the larger Gemma4-12B-it, so the models span 4B to 12B parameters.
None is a pretrained-only checkpoint: the four Qwen models are the default releases rather than the \texttt{-Base} variants and ship with chat templates that support a thinking mode, and Gemma4-12B-it is instruction-tuned.

\subsection{Evaluation Metrics}
\label{apx:metrics}

\textbf{Execution Efficiency}  measures the cost of running the same MH sampler for a fixed number of transitions.
\begin{itemize}[leftmargin=*]
\item \textit{MH transitions per model call.}
The number of MH transitions realized per model call; higher means more chain progress per call.
On each prompt, a method that completes \(K\) realized MH transitions realizes \(K/\mathrm{\#\ of\ model\ calls}\) transitions per call.

\item \textit{Wall-clock time and speedup.}
We record the wall-clock time of every model call.
On each prompt, a method's \(\mathrm{time\ per\ call}\) is the mean duration of its timed model calls, and \(\mathrm{call\ time}=\mathrm{\#\ of\ model\ calls}\times\mathrm{time\ per\ call}\) is its cumulative model-call time.
Every reported speedup is a ratio of means over the prompts \(i=1,\ldots,N\): the mean cumulative model-call time of the sequential sampler divided by that of its \framework variant,
\begin{equation}
    \operatorname{Speedup}
    = \frac{\frac{1}{N}\sum_{i} \mathrm{call\ time}_i(\text{sampler})}
           {\frac{1}{N}\sum_{i} \mathrm{call\ time}_i(\text{\framework variant})}.
    \label{eq:paired-call-speedup}
\end{equation}
We decompose it with ratios of means as well: the \emph{increase in time per call} is the mean time per call of the \framework variant divided by that of the sampler, and the \emph{reduction in model calls} is the speedup multiplied by this increase, so that the speedup equals the reduction in calls divided by the increase in time per call (Figure~\ref{fig:teaser}).
\end{itemize}

\textbf{Sample Quality} measures whether \framework changes the samples after the same number of transitions.
\begin{itemize}[leftmargin=*]
    \item \textit{Likelihood and confidence.}
    Prior work evaluates MCMC samplers for text by comparing their sampled distributions with reference distributions~\citep{DBLP:conf/conll/ForristalMDB23,DBLP:conf/icml/DuAHYLEC24}.
    For power sharpening, \citet{karan2026reasoning} characterize sampled responses by their length-normalized base-model likelihood and confidence.
    Let \(\mathbf{x}^{(K)}_{0:n}=(\mathbf{x}_0,\mathbf{x}^{(K)}_{1:n})\) be the terminal state after exactly \(K\) realized MH transitions: the prompt \(\mathbf{x}_0\) followed by a response of length \(n\). Its length-normalized log-likelihood is
    \begin{equation}
    \operatorname{Log-Likelihood}(\mathbf{x}^{(K)})
    =
    \frac{1}{n}\log p_0(\mathbf{x}^{(K)}_{1:n}\mid\mathbf{x}_0),
    \label{eq:normalized-loglik}
    \end{equation}
    and its confidence is the mean negative entropy of the base model's next-token distributions along the sequence:
    \begin{equation}
    \operatorname{Confidence}(\mathbf{x}^{(K)})
    =
    \frac{1}{n}\sum_{t=1}^{n}\sum_{u\in\mathcal{V}}
    p_0(u\mid\mathbf{x}^{(K)}_{0:t-1})
    \log p_0(u\mid\mathbf{x}^{(K)}_{0:t-1}),
    \label{eq:confidence}
    \end{equation}
    where \(\mathcal{V}\) is the vocabulary.
    Both statistics are computed under the same base model \(p_0\), so they are comparable across samplers.
    Higher log-likelihood means a higher mean token log-probability under \(p_0\), and higher confidence means lower predictive uncertainty along the sequence.

    \item \textit{Paired TOST equivalence test.}
    This test checks whether \framework changes the mean of each statistic by less than a small margin \(\delta\).
    For each sampler and its \framework variant, we match prompts and seeds, run both for the same \(K\), and take the per-prompt difference \(d_i\) of each statistic (\framework variant minus sequential sampler), with mean \(\mu_d\).
    The two one-sided tests (TOST) procedure~\citep{schuirmann1987tost,lakens2017equivalence} tests
    \begin{equation}
    H_0:|\mu_d|\geq\delta
    \qquad\text{against}\qquad
    H_1:|\mu_d|<\delta
    \end{equation}
    with two one-sided paired \(t\)-tests; \(p_{\mathrm{eq}}\) is the larger of their two \(p\)-values.
    We set \(\delta\) to 0.2 standard deviations of the statistic across prompts under the sequential sampler, a conventional small effect, and declare equivalence when \(p_{\mathrm{eq}}<0.05\), that is, when the 90\% confidence interval for \(\mu_d\) lies within \((-\delta,\delta)\).
\end{itemize}

\textbf{Hardware-Resource Efficiency} measures how much of the KV-cache pool a sampler occupies and how much prefix computation it reuses.
\begin{itemize}[leftmargin=*]
    \item \textit{Peak KV-pool use.}
    We report the peak fraction of occupied blocks in the pool, where higher use indicates greater cache demand.

    \item \textit{Prefix-cache hit rate.}
    We measure prefix reuse as the cumulative token hit rate: the fraction of queried prefix tokens whose KV states are already cached.
    A higher hit rate indicates greater reuse of prefix computation.

    \item \textit{Prefix-cache miss tokens per transition.}
    We divide the number of prefix-cache miss tokens by the number of realized MH transitions.
    This complements the hit rate by measuring the absolute uncached workload per transition.
\end{itemize}

\subsection{Implementation Environment}
Table~\ref{tab:hard-soft-configuration} summarizes the experimental software and hardware environment.
We implement PowerMH, EntropyCut, MultiTryMH, and their \framework variants on top of vLLM, using custom engine patches.
For Case Study~6 (Appendix~\ref{apx:serving-engines}), we also implement PowerMH and \method on SGLang.

\textbf{vLLM engine patches.}
Each MH transition needs the log-probability of a proposed suffix under both the proposal $q_{\tau}$ and the base model $p_0$, and EntropyCut also needs the next-token entropy of $p_0$ at every position.
The unmodified vLLM engine returns log-probabilities from only one version of the logits, either before or after temperature scaling; this choice is an engine-wide setting (\texttt{logprobs\_mode}), not a per-request option.
We therefore patch the vLLM~V1 engine, run with \texttt{logprobs\_mode=processed\_logprobs}, to return all of these quantities from the forward pass that generates each token.

Before temperature scaling, penalties, or any other logit processing, the patched sampler keeps a float32 copy of the logits.
For each sampled token $x_t$, it returns
(i) the proposal log-probability $\log q_{\tau}(x_t\mid\mathbf{x}_{0:t-1})$ from the processed logits, as in unmodified vLLM;
(ii) the base-model log-probability $\log p_0(x_t\mid\mathbf{x}_{0:t-1})$ from the unprocessed copy; and
(iii) the full-vocabulary entropy of $p_0$, computed from the unprocessed copy with the natural logarithm,
\begin{equation*}
H_t
=
-\sum_{u\in\mathcal{V}} p_0(u\mid\mathbf{x}_{0:t-1})\log p_0(u\mid\mathbf{x}_{0:t-1}),
\end{equation*}

\subsection{Hyperparameter Configurations}
\label{apx:hyperparameters}
Table~\ref{tab:hyperparameter-configuration} lists the default power-sampling settings and the settings of our case studies.
Following \citet{karan2026reasoning}, we set the proposal temperature to $\tau=1/\alpha$ for every sampler except MultiTryMH and \methodMTM, whose candidates use the temperatures in Table~\ref{tab:hyperparameter-configuration}(a); all runs use seed 10086.

\begin{table}[!h]
    \centering
        \caption{Software and hardware configuration.}
        \label{tab:hard-soft-configuration}
        \small
        \begin{tabular}{ll}
            \toprule
            \textbf{Software} & \textbf{Version / configuration} \\
            \midrule
    Python & 3.13.0 \\
    PyTorch & 2.13.0  \\
    vLLM & 0.27.1 (V1 engine) with custom engine patches \\
    Transformers & 5.14.1  \\
            \bottomrule
            \multicolumn{2}{c}{\textbf{(a)} Software configuration} \\
            \addlinespace[4pt]
            \toprule
    \textbf{Hardware} & \textbf{Specification} \\
            \midrule
    GPU & 1$\times$ NVIDIA H100 96\,GB (GH200 Grace Hopper Superchip) \\
    CPU & NVIDIA Grace CPU, 72 cores (3.456 GHz) \\
    Memory & 480 GB LPDDR5X unified memory \\
            \bottomrule
            \multicolumn{2}{c}{\textbf{(b)} Hardware configuration} \\
        \end{tabular}
\end{table}

\begin{table}[!ht]
    \centering
    \caption{Hyperparameter configurations.}
    \label{tab:hyperparameter-configuration}
    \small
    \begin{tabular}{@{}ll@{}}
        \toprule
        \textbf{Parameter} & \textbf{Setting} \\
        \midrule
Response length & $n=1024$ tokens (fixed; no early stop at EOS) \\
Block size & a single block, $B=n=1024$ \\
Power exponent & $\alpha=4.0$ \\
Proposal LLM & $q_{\tau}=$ base model sampled at temperature $\tau$ \\
Proposal temperature & $\tau=1/\alpha=0.25$ \\
Realized MH transitions & $K=100$ \\
Prompts per dataset & first 20 prompts \\
Seed & 10086 \\
Traversal rule & default BFS  \\
Prefetch budget & $N_{\mathrm{pf}}=20$ \\
        \midrule
EntropyCut cut-law exponent & $\beta=4.0$ (dimensionless) \\
MultiTryMH candidates per transition & 4, each at a temperature drawn from $\{0.25,0.5,1.0\}$ \\
        \bottomrule
        \multicolumn{2}{c}{\textbf{(a)} Default power-sampling configuration} \\
        \addlinespace[4pt]
        \toprule
        \textbf{Parameter} & \textbf{Setting} \\
        \midrule
Base model and dataset & $p_0=$ Qwen3.5-9B on LCB (Case Studies~1, 2, 5--8) \\
Realized MH transitions & $K\in\{10,20,30\}$ (Case Study~8) \\
Certainty threshold & $\epsilon=0.01$ (Case Studies~4, 5, 7) \\
Prefetch budget (Case Study~2,5) & $N_{\mathrm{pf}}\in\{8,10,12,16,20\}$ \\
        \bottomrule
        \multicolumn{2}{c}{\textbf{(b)} Case study configuration} \\
    \end{tabular}
\end{table}

\section{Additional Experimental Results}
\label{apx:extra-exp}

\FloatBarrier
\subsection{Full results: Acceleration across base LLMs and datasets}
\label{apx:steps-time-compare}

We pair PowerMH, EntropyCut, and MultiTryMH with their \framework variants on all five base LLMs and seven datasets.
Tables~\ref{tab:all-paired-mean-metrics}, \ref{tab:entropycut-mean-metrics}, and~\ref{tab:multitry-mean-metrics} report the mean and standard deviation of each metric for the three samplers.

\begin{table}[!htbp]
    \caption{PowerMH execution metrics across base LLMs.}
    \label{tab:all-paired-mean-metrics}

    \adjustbox{width=\linewidth}{%
    \begin{tabular}{llcccc}
        \toprule
        Dataset & Base LLM
        & \makecell{Total Time (s) \(\downarrow\)}
        & \makecell{Transitions\\per call \(\uparrow\)}
        & \makecell{Time/call (s) \(\downarrow\)}
        & Speedup \\
        \midrule
        \multirow{5}{*}{\rotatebox[origin=c]{90}{\makecell{\textbf{MATH500}\\(Fig.~\ref{fig:paired-bfs-math500})}}}
        & {Qwen3-4B}
        & \(327.83_{\pm20.11} / 155.68_{\pm21.84}\)
        & \(1 / 4.42_{\pm0.60}\)
        & \(3.28_{\pm0.20} / 6.77_{\pm0.09}\)
        & \(2.11\times\) \\
        & {Qwen3-8B}
        & \(448.28_{\pm27.52} / 200.65_{\pm31.21}\)
        & \(1 / 4.37_{\pm0.68}\)
        & \(4.48_{\pm0.28} / 8.58_{\pm0.14}\)
        & \(2.23\times\) \\
        & {Qwen3.5-4B}
        & \(371.97_{\pm22.70} / 249.10_{\pm70.00}\)
        & \(1 / 3.57_{\pm0.99}\)
        & \(3.72_{\pm0.23} / 8.29_{\pm0.11}\)
        & \(1.49\times\) \\
        & {Qwen3.5-9B}
        & \(499.24_{\pm30.53} / 242.93_{\pm45.34}\)
        & \(1 / 4.50_{\pm0.75}\)
        & \(4.99_{\pm0.31} / 10.63_{\pm0.20}\)
        & \(2.06\times\) \\
        & {Gemma4-12B-it}
        & \(687.48_{\pm41.01} / 395.55_{\pm77.86}\)
        & \(1 / 3.96_{\pm0.68}\)
        & \(6.87_{\pm0.41} / 15.21_{\pm0.29}\)
        & \(1.74\times\) \\
        \midrule
        \multirow{5}{*}{\rotatebox[origin=c]{90}{\textbf{AIME 24\&25}}}
        & {Qwen3-4B}
        & \(318.76_{\pm19.59} / 150.66_{\pm32.25}\)
        & \(1 / 4.66_{\pm0.86}\)
        & \(3.19_{\pm0.20} / 6.78_{\pm0.11}\)
        & \(2.12\times\) \\
        & {Qwen3-8B}
        & \(447.87_{\pm27.50} / 195.48_{\pm34.63}\)
        & \(1 / 4.48_{\pm0.71}\)
        & \(4.48_{\pm0.28} / 8.53_{\pm0.15}\)
        & \(2.29\times\) \\
        & {Qwen3.5-4B}
        & \(360.89_{\pm22.43} / 217.57_{\pm37.98}\)
        & \(1 / 3.99_{\pm0.62}\)
        & \(3.61_{\pm0.22} / 8.47_{\pm0.16}\)
        & \(1.66\times\) \\
        & {Qwen3.5-9B}
        & \(498.37_{\pm30.48} / 239.31_{\pm44.47}\)
        & \(1 / 4.56_{\pm0.72}\)
        & \(4.98_{\pm0.30} / 10.63_{\pm0.31}\)
        & \(2.08\times\) \\
        & {Gemma4-12B-it}
        & \(686.87_{\pm38.96} / 389.43_{\pm57.27}\)
        & \(1 / 3.93_{\pm0.53}\)
        & \(6.87_{\pm0.39} / 15.04_{\pm0.23}\)
        & \(1.76\times\) \\
        \midrule
        \multirow{5}{*}{\rotatebox[origin=c]{90}{\textbf{MBPP}}}
        & {Qwen3-4B}
        & \(327.45_{\pm20.07} / 296.53_{\pm65.99}\)
        & \(1 / 2.42_{\pm0.65}\)
        & \(3.27_{\pm0.20} / 6.81_{\pm0.08}\)
        & \(1.10\times\) \\
        & {Qwen3-8B}
        & \(447.85_{\pm27.46} / 446.25_{\pm41.30}\)
        & \(1 / 1.96_{\pm0.19}\)
        & \(4.48_{\pm0.27} / 8.68_{\pm0.07}\)
        & \(1.00\times\) \\
        & {Qwen3.5-4B}
        & \(370.03_{\pm22.56} / 221.28_{\pm69.63}\)
        & \(1 / 4.13_{\pm0.94}\)
        & \(3.70_{\pm0.23} / 8.56_{\pm0.30}\)
        & \(1.67\times\) \\
        & {Qwen3.5-9B}
        & \(497.47_{\pm30.41} / 309.52_{\pm88.73}\)
        & \(1 / 3.74_{\pm0.84}\)
        & \(4.97_{\pm0.30} / 10.92_{\pm0.32}\)
        & \(1.61\times\) \\
        & {Gemma4-12B-it}
        & \(684.80_{\pm38.84} / 475.98_{\pm177.87}\)
        & \(1 / 3.64_{\pm1.29}\)
        & \(6.85_{\pm0.39} / 15.29_{\pm0.15}\)
        & \(1.44\times\) \\
        \midrule
        \multirow{5}{*}{\rotatebox[origin=c]{90}{\textbf{GPQA}}}
        & {Qwen3-4B}
        & \(323.68_{\pm21.46} / 198.61_{\pm42.79}\)
        & \(1 / 3.55_{\pm0.59}\)
        & \(3.24_{\pm0.21} / 6.83_{\pm0.08}\)
        & \(1.63\times\) \\
        & {Qwen3-8B}
        & \(442.49_{\pm29.28} / 208.59_{\pm77.96}\)
        & \(1 / 4.40_{\pm0.91}\)
        & \(4.42_{\pm0.29} / 8.55_{\pm0.12}\)
        & \(2.12\times\) \\
        & {Qwen3.5-4B}
        & \(365.82_{\pm24.10} / 217.07_{\pm79.59}\)
        & \(1 / 4.16_{\pm0.92}\)
        & \(3.66_{\pm0.24} / 8.38_{\pm0.20}\)
        & \(1.69\times\) \\
        & {Qwen3.5-9B}
        & \(491.59_{\pm32.48} / 239.92_{\pm79.83}\)
        & \(1 / 4.73_{\pm1.12}\)
        & \(4.92_{\pm0.32} / 10.57_{\pm0.25}\)
        & \(2.05\times\) \\
        & {Gemma4-12B-it}
        & \(687.19_{\pm39.51} / 411.00_{\pm127.19}\)
        & \(1 / 3.90_{\pm0.81}\)
        & \(6.87_{\pm0.40} / 15.12_{\pm0.21}\)
        & \(1.67\times\) \\
        \midrule
        \multirow{5}{*}{\rotatebox[origin=c]{90}{\textbf{HumanEval}}}
        & {Qwen3-4B}
        & \(327.25_{\pm20.07} / 184.93_{\pm54.81}\)
        & \(1 / 3.89_{\pm1.04}\)
        & \(3.27_{\pm0.20} / 6.70_{\pm0.11}\)
        & \(1.77\times\) \\
        & {Qwen3-8B}
        & \(447.42_{\pm27.44} / 247.81_{\pm73.78}\)
        & \(1 / 3.72_{\pm0.95}\)
        & \(4.47_{\pm0.27} / 8.59_{\pm0.09}\)
        & \(1.81\times\) \\
        & {Qwen3.5-4B}
        & \(369.89_{\pm22.61} / 231.97_{\pm54.39}\)
        & \(1 / 3.79_{\pm0.92}\)
        & \(3.70_{\pm0.23} / 8.36_{\pm0.24}\)
        & \(1.59\times\) \\
        & {Qwen3.5-9B}
        & \(498.37_{\pm30.48} / 412.74_{\pm78.17}\)
        & \(1 / 2.67_{\pm0.46}\)
        & \(4.98_{\pm0.30} / 10.72_{\pm0.21}\)
        & \(1.21\times\) \\
        & {Gemma4-12B-it}
        & \(687.16_{\pm39.54} / 400.87_{\pm82.35}\)
        & \(1 / 3.93_{\pm0.74}\)
        & \(6.87_{\pm0.40} / 15.19_{\pm0.17}\)
        & \(1.71\times\) \\
        \midrule
        \multirow{5}{*}{\rotatebox[origin=c]{90}{\textbf{MMLU}}}
        & {Qwen3-4B}
        & \(327.51_{\pm20.06} / 205.84_{\pm62.37}\)
        & \(1 / 3.59_{\pm1.06}\)
        & \(3.28_{\pm0.20} / 6.82_{\pm0.11}\)
        & \(1.59\times\) \\
        & {Qwen3-8B}
        & \(447.91_{\pm27.44} / 298.07_{\pm123.34}\)
        & \(1 / 3.29_{\pm1.17}\)
        & \(4.48_{\pm0.27} / 8.57_{\pm0.13}\)
        & \(1.50\times\) \\
        & {Qwen3.5-4B}
        & \(370.51_{\pm22.60} / 210.37_{\pm47.85}\)
        & \(1 / 4.12_{\pm0.81}\)
        & \(3.71_{\pm0.23} / 8.32_{\pm0.17}\)
        & \(1.76\times\) \\
        & {Qwen3.5-9B}
        & \(493.44_{\pm31.14} / 418.01_{\pm122.83}\)
        & \(1 / 2.72_{\pm0.78}\)
        & \(4.93_{\pm0.31} / 10.52_{\pm0.15}\)
        & \(1.18\times\) \\
        & {Gemma4-12B-it}
        & \(686.89_{\pm39.47} / 452.21_{\pm61.30}\)
        & \(1 / 3.40_{\pm0.61}\)
        & \(6.87_{\pm0.39} / 15.03_{\pm0.21}\)
        & \(1.52\times\) \\
        \midrule
        \multirow{5}{*}{\rotatebox[origin=c]{90}{\textbf{LCB V6}}}
        & {Qwen3-4B}
        & \(318.79_{\pm19.56} / 156.14_{\pm35.60}\)
        & \(1 / 4.92_{\pm0.76}\)
        & \(3.19_{\pm0.20} / 7.44_{\pm0.14}\)
        & \(2.04\times\) \\
        & {Qwen3-8B}
        & \(432.26_{\pm25.14} / 210.55_{\pm51.09}\)
        & \(1 / 4.85_{\pm0.90}\)
        & \(4.32_{\pm0.25} / 9.77_{\pm0.15}\)
        & \(2.05\times\) \\
        & {Qwen3.5-4B}
        & \(368.07_{\pm20.61} / 182.61_{\pm21.80}\)
        & \(1 / 4.95_{\pm0.58}\)
        & \(3.68_{\pm0.21} / 8.92_{\pm0.28}\)
        & \(2.02\times\) \\
        & {Qwen3.5-9B}
        & \(539.24_{\pm30.92} / 375.12_{\pm35.08}\)
        & \(1 / 4.96_{\pm0.47}\)
        & \(5.39_{\pm0.31} / 18.45_{\pm0.51}\)
        & \(1.44\times\) \\
        & {Gemma4-12B-it}
        & \(685.34_{\pm39.07} / 351.02_{\pm53.96}\)
        & \(1 / 4.47_{\pm0.68}\)
        & \(6.85_{\pm0.39} / 15.35_{\pm0.29}\)
        & \(1.95\times\) \\
        \bottomrule
    \end{tabular}%
    }
\end{table}

\begin{figure}[!htbp]
    \centering
    \includegraphics[width=.92\linewidth]{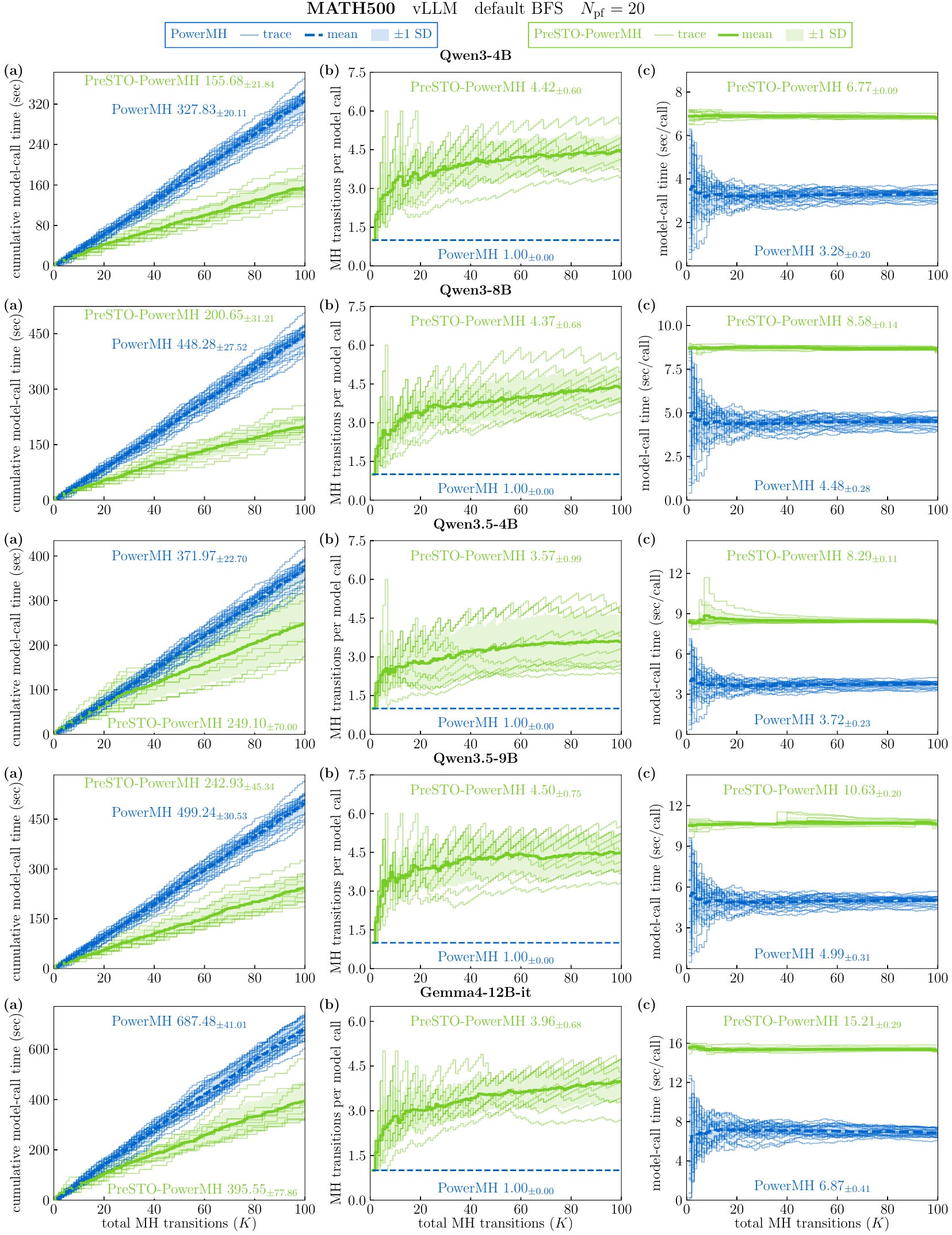}
    \caption{PowerMH and \method on MATH500.}
    \label{fig:paired-bfs-math500}
\end{figure}

\begin{table}[!htbp]
    \caption{EntropyCut execution metrics across base LLMs.}   
     \label{tab:entropycut-mean-metrics}
    \resizebox{\linewidth}{!}{%
    \begin{tabular}{llcccc}
        \toprule
        Dataset & Base LLM
        & \makecell{Total Time (s) \(\downarrow\)}
        & \makecell{Transitions\\per call \(\uparrow\)}
        & \makecell{Time/call (s) \(\downarrow\)}
        & Speedup \\
        \midrule
        \multirow{5}{*}{\rotatebox[origin=c]{90}{\makecell{\textbf{MATH500}\\(Fig.~\ref{fig:paired-entropycut-math500})}}}
        & {Qwen3-4B}
        & \(277.41_{\pm45.48} / 157.85_{\pm20.51}\)
        & \(1 / 4.26_{\pm0.54}\)
        & \(2.77_{\pm0.45} / 6.62_{\pm0.23}\)
        & \(1.76\times\) \\
        & {Qwen3-8B}
        & \(336.64_{\pm58.87} / 192.00_{\pm24.30}\)
        & \(1 / 4.34_{\pm0.47}\)
        & \(3.37_{\pm0.59} / 8.23_{\pm0.27}\)
        & \(1.75\times\) \\
        & {Qwen3.5-4B}
        & \(365.78_{\pm80.61} / 233.47_{\pm73.77}\)
        & \(1 / 3.81_{\pm0.99}\)
        & \(3.66_{\pm0.81} / 8.22_{\pm0.47}\)
        & \(1.57\times\) \\
        & {Qwen3.5-9B}
        & \(478.02_{\pm74.36} / 272.29_{\pm59.03}\)
        & \(1 / 3.84_{\pm0.65}\)
        & \(4.78_{\pm0.74} / 10.11_{\pm0.52}\)
        & \(1.76\times\) \\
        & {Gemma4-12B-it}
        & \(980.09_{\pm240.16} / 431.85_{\pm175.94}\)
        & \(1 / 4.07_{\pm2.00}\)
        & \(9.80_{\pm2.40} / 15.06_{\pm1.37}\)
        & \(2.27\times\) \\
        \midrule
        \multirow{5}{*}{\rotatebox[origin=c]{90}{\textbf{AIME 24\&25}}}
        & {Qwen3-4B}
        & \(255.04_{\pm37.63} / 150.54_{\pm16.85}\)
        & \(1 / 4.57_{\pm0.49}\)
        & \(2.55_{\pm0.38} / 6.81_{\pm0.18}\)
        & \(1.69\times\) \\
        & {Qwen3-8B}
        & \(345.86_{\pm62.77} / 192.57_{\pm26.84}\)
        & \(1 / 4.37_{\pm0.56}\)
        & \(3.46_{\pm0.63} / 8.28_{\pm0.30}\)
        & \(1.80\times\) \\
        & {Qwen3.5-4B}
        & \(405.21_{\pm104.47} / 196.78_{\pm41.26}\)
        & \(1 / 4.29_{\pm0.64}\)
        & \(4.05_{\pm1.04} / 8.20_{\pm0.66}\)
        & \(2.06\times\) \\
        & {Qwen3.5-9B}
        & \(490.26_{\pm82.75} / 277.80_{\pm65.23}\)
        & \(1 / 3.83_{\pm0.75}\)
        & \(4.90_{\pm0.83} / 10.21_{\pm0.51}\)
        & \(1.76\times\) \\
        & {Gemma4-12B-it}
        & \(684.74_{\pm165.29} / 449.67_{\pm125.04}\)
        & \(1 / 3.55_{\pm0.85}\)
        & \(6.85_{\pm1.65} / 15.01_{\pm0.69}\)
        & \(1.52\times\) \\
        \midrule
        \multirow{5}{*}{\rotatebox[origin=c]{90}{\textbf{MBPP}}}
        & {Qwen3-4B}
        & \(260.73_{\pm60.29} / 218.41_{\pm48.68}\)
        & \(1 / 3.51_{\pm1.25}\)
        & \(2.61_{\pm0.60} / 7.15_{\pm0.20}\)
        & \(1.19\times\) \\
        & {Qwen3-8B}
        & \(367.12_{\pm82.44} / 256.66_{\pm96.78}\)
        & \(1 / 4.00_{\pm1.42}\)
        & \(3.67_{\pm0.82} / 9.02_{\pm0.14}\)
        & \(1.43\times\) \\
        & {Qwen3.5-4B}
        & \(509.39_{\pm79.35} / 207.31_{\pm22.63}\)
        & \(1 / 4.13_{\pm0.37}\)
        & \(5.09_{\pm0.79} / 8.49_{\pm0.27}\)
        & \(2.46\times\) \\
        & {Qwen3.5-9B}
        & \(762.07_{\pm70.82} / 267.03_{\pm54.05}\)
        & \(1 / 4.30_{\pm0.92}\)
        & \(7.62_{\pm0.71} / 11.03_{\pm0.17}\)
        & \(2.85\times\) \\
        & {Gemma4-12B-it}
        & \(1022.57_{\pm156.01} / 486.63_{\pm107.74}\)
        & \(1 / 3.25_{\pm0.73}\)
        & \(10.23_{\pm1.56} / 15.12_{\pm1.06}\)
        & \(2.10\times\) \\
        \midrule
        \multirow{5}{*}{\rotatebox[origin=c]{90}{\textbf{GPQA}}}
        & {Qwen3-4B}
        & \(372.97_{\pm69.04} / 178.64_{\pm32.15}\)
        & \(1 / 4.30_{\pm0.67}\)
        & \(3.73_{\pm0.69} / 7.49_{\pm0.42}\)
        & \(2.09\times\) \\
        & {Qwen3-8B}
        & \(377.49_{\pm64.81} / 210.70_{\pm24.84}\)
        & \(1 / 4.39_{\pm0.50}\)
        & \(3.77_{\pm0.65} / 9.14_{\pm0.38}\)
        & \(1.79\times\) \\
        & {Qwen3.5-4B}
        & \(397.44_{\pm97.55} / 203.25_{\pm52.71}\)
        & \(1 / 4.36_{\pm0.77}\)
        & \(3.97_{\pm0.98} / 8.50_{\pm0.49}\)
        & \(1.96\times\) \\
        & {Qwen3.5-9B}
        & \(529.97_{\pm101.82} / 269.92_{\pm75.09}\)
        & \(1 / 4.21_{\pm0.75}\)
        & \(5.30_{\pm1.02} / 10.87_{\pm0.70}\)
        & \(1.96\times\) \\
        & {Gemma4-12B-it}
        & \(943.24_{\pm225.03} / 435.88_{\pm116.35}\)
        & \(1 / 3.64_{\pm0.72}\)
        & \(9.43_{\pm2.25} / 15.17_{\pm1.47}\)
        & \(2.16\times\) \\
        \midrule
        \multirow{5}{*}{\rotatebox[origin=c]{90}{\textbf{HumanEval}}}
        & {Qwen3-4B}
        & \(380.83_{\pm108.02} / 203.08_{\pm60.14}\)
        & \(1 / 3.68_{\pm1.37}\)
        & \(3.81_{\pm1.08} / 6.83_{\pm0.77}\)
        & \(1.88\times\) \\
        & {Qwen3-8B}
        & \(441.98_{\pm154.66} / 215.14_{\pm88.83}\)
        & \(1 / 4.44_{\pm1.79}\)
        & \(4.42_{\pm1.55} / 8.30_{\pm1.04}\)
        & \(2.05\times\) \\
        & {Qwen3.5-4B}
        & \(435.56_{\pm105.27} / 211.17_{\pm46.11}\)
        & \(1 / 3.71_{\pm0.76}\)
        & \(4.36_{\pm1.05} / 7.59_{\pm1.06}\)
        & \(2.06\times\) \\
        & {Qwen3.5-9B}
        & \(767.76_{\pm90.41} / 268.95_{\pm108.36}\)
        & \(1 / 4.78_{\pm2.47}\)
        & \(7.68_{\pm0.90} / 10.56_{\pm0.44}\)
        & \(2.85\times\) \\
        & {Gemma4-12B-it}
        & \(930.27_{\pm168.82} / 401.23_{\pm108.00}\)
        & \(1 / 3.76_{\pm1.02}\)
        & \(9.30_{\pm1.69} / 14.13_{\pm1.84}\)
        & \(2.32\times\) \\
        \midrule
        \multirow{5}{*}{\rotatebox[origin=c]{90}{\textbf{MMLU}}}
        & {Qwen3-4B}
        & \(390.39_{\pm105.60} / 192.26_{\pm42.67}\)
        & \(1 / 4.01_{\pm0.73}\)
        & \(3.90_{\pm1.06} / 7.44_{\pm0.34}\)
        & \(2.03\times\) \\
        & {Qwen3-8B}
        & \(612.03_{\pm184.09} / 278.39_{\pm93.41}\)
        & \(1 / 3.90_{\pm1.15}\)
        & \(6.12_{\pm1.84} / 9.88_{\pm0.42}\)
        & \(2.20\times\) \\
        & {Qwen3.5-4B}
        & \(395.79_{\pm86.46} / 240.24_{\pm86.12}\)
        & \(1 / 3.75_{\pm0.83}\)
        & \(3.96_{\pm0.86} / 8.39_{\pm0.83}\)
        & \(1.65\times\) \\
        & {Qwen3.5-9B}
        & \(841.38_{\pm140.68} / 352.80_{\pm95.07}\)
        & \(1 / 3.49_{\pm0.87}\)
        & \(8.41_{\pm1.41} / 11.54_{\pm0.40}\)
        & \(2.38\times\) \\
        & {Gemma4-12B-it}
        & \(887.63_{\pm220.26} / 376.96_{\pm167.32}\)
        & \(1 / 4.50_{\pm2.02}\)
        & \(8.88_{\pm2.20} / 14.11_{\pm1.67}\)
        & \(2.35\times\) \\
        \midrule
        \multirow{5}{*}{\rotatebox[origin=c]{90}{\textbf{LCB V6}}}
        & {Qwen3-4B}
        & \(250.62_{\pm56.09} / 151.82_{\pm26.30}\)
        & \(1 / 4.19_{\pm0.62}\)
        & \(2.51_{\pm0.56} / 6.20_{\pm0.30}\)
        & \(1.65\times\) \\
        & {Qwen3-8B}
        & \(340.08_{\pm82.22} / 220.38_{\pm65.17}\)
        & \(1 / 3.92_{\pm0.90}\)
        & \(3.40_{\pm0.82} / 8.15_{\pm0.73}\)
        & \(1.54\times\) \\
        & {Qwen3.5-4B}
        & \(337.14_{\pm37.91} / 198.35_{\pm48.78}\)
        & \(1 / 4.31_{\pm0.59}\)
        & \(3.37_{\pm0.38} / 8.36_{\pm1.40}\)
        & \(1.70\times\) \\
        & {Qwen3.5-9B}
        & \(463.80_{\pm52.32} / 242.89_{\pm35.93}\)
        & \(1 / 4.22_{\pm0.58}\)
        & \(4.64_{\pm0.52} / 10.08_{\pm0.47}\)
        & \(1.91\times\) \\
        & {Gemma4-12B-it}
        & \(844.27_{\pm204.00} / 472.94_{\pm164.09}\)
        & \(1 / 3.56_{\pm0.99}\)
        & \(8.44_{\pm2.04} / 15.34_{\pm0.77}\)
        & \(1.79\times\) \\
        \bottomrule
    \end{tabular}%
    }
\end{table}

\begin{figure}[!htbp]
    \centering
    \includegraphics[width=.92\linewidth]{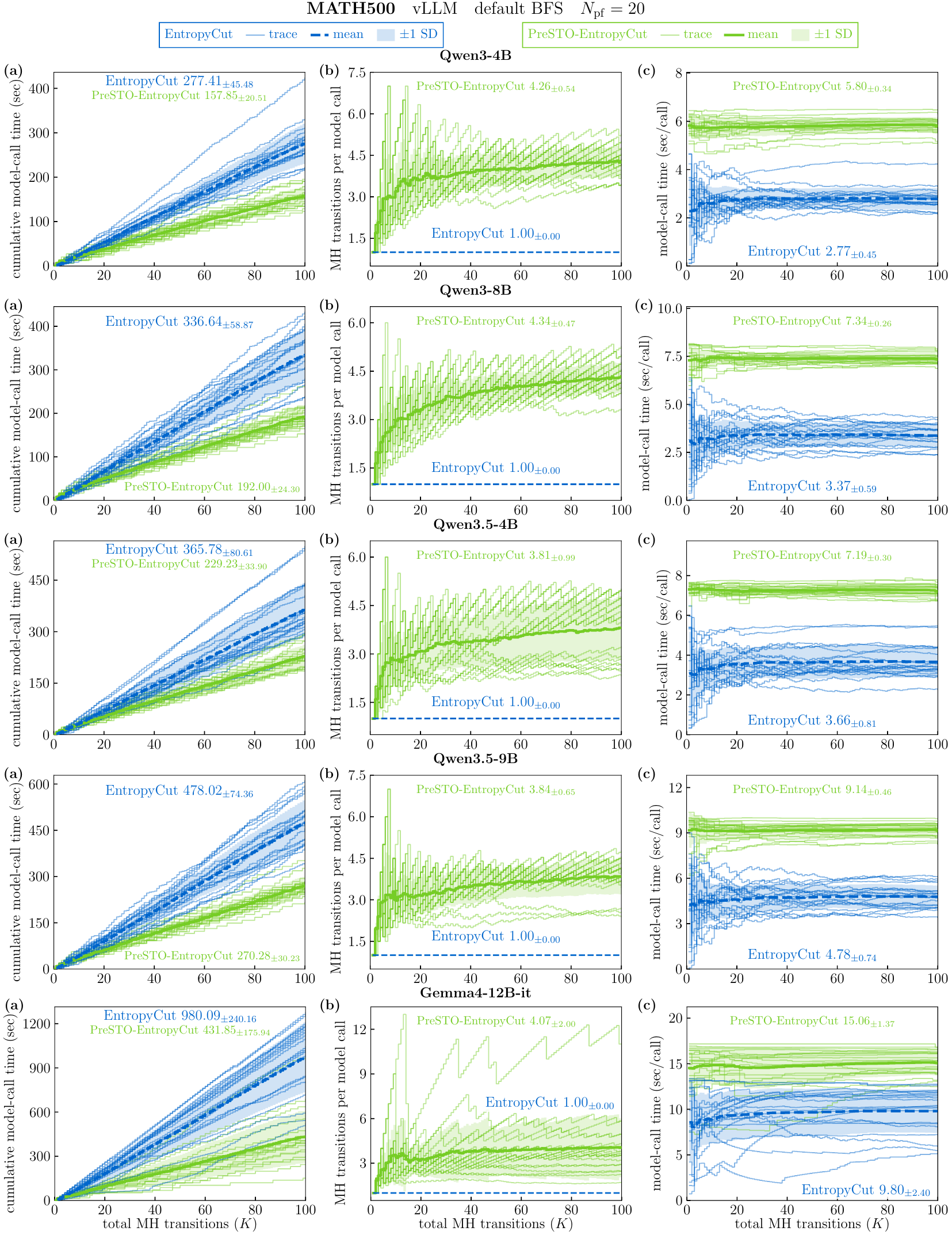}
    \caption{EntropyCut and \methodEnt\ on MATH500.}
    \label{fig:paired-entropycut-math500}
\end{figure}

\begin{table}[!p]
    \caption{MultiTryMH execution metrics across base LLMs.}
    \label{tab:multitry-mean-metrics}

    \resizebox{\linewidth}{!}{%
    \begin{tabular}{llcccc}
        \toprule
        Dataset & Base LLM
        & \makecell{Total Time (s) \(\downarrow\)}
        & \makecell{Transitions\\per call \(\uparrow\)}
        & \makecell{Time/call (s) \(\downarrow\)}
        & Speedup \\
        \midrule
        \multirow{5}{*}{\rotatebox[origin=c]{90}{\textbf{MATH500}}}
        & {Qwen3-4B}
        & \(741.55_{\pm35.32} / 322.90_{\pm17.49}\)
        & \(1 / 2.02_{\pm0.11}\)
        & \(7.42_{\pm0.35} / 6.49_{\pm0.17}\)
        & \(2.30\times\) \\
        & {Qwen3-8B}
        & \(865.84_{\pm42.90} / 417.08_{\pm21.21}\)
        & \(1 / 2.04_{\pm0.10}\)
        & \(8.66_{\pm0.43} / 8.50_{\pm0.18}\)
        & \(2.08\times\) \\
        & {Qwen3.5-4B}
        & \(2174.72_{\pm142.35} / 375.23_{\pm14.43}\)
        & \(1 / 2.05_{\pm0.08}\)
        & \(21.75_{\pm1.42} / 7.68_{\pm0.16}\)
        & \(5.80\times\) \\
        & {Qwen3.5-9B}
        & \(3246.66_{\pm262.85} / 463.26_{\pm25.28}\)
        & \(1 / 2.05_{\pm0.11}\)
        & \(32.47_{\pm2.63} / 9.48_{\pm0.15}\)
        & \(7.01\times\) \\
        & {Gemma4-12B-it}
        & \(1558.70_{\pm79.33} / 686.49_{\pm45.49}\)
        & \(1 / 1.86_{\pm0.13}\)
        & \(15.59_{\pm0.79} / 12.75_{\pm0.20}\)
        & \(2.27\times\) \\
        \midrule
        \multirow{5}{*}{\rotatebox[origin=c]{90}{\textbf{AIME 24\&25}}}
        & {Qwen3-4B}
        & \(633.31_{\pm31.36} / 313.16_{\pm17.48}\)
        & \(1 / 2.03_{\pm0.11}\)
        & \(6.33_{\pm0.31} / 6.33_{\pm0.18}\)
        & \(2.02\times\) \\
        & {Qwen3-8B}
        & \(867.08_{\pm41.03} / 429.22_{\pm17.70}\)
        & \(1 / 1.99_{\pm0.06}\)
        & \(8.67_{\pm0.41} / 8.52_{\pm0.20}\)
        & \(2.02\times\) \\
        & {Qwen3.5-4B}
        & \(2337.06_{\pm148.90} / 381.74_{\pm16.29}\)
        & \(1 / 2.04_{\pm0.09}\)
        & \(23.37_{\pm1.49} / 7.79_{\pm0.19}\)
        & \(6.12\times\) \\
        & {Qwen3.5-9B}
        & \(3522.15_{\pm239.42} / 510.54_{\pm27.05}\)
        & \(1 / 2.03_{\pm0.11}\)
        & \(35.22_{\pm2.39} / 10.35_{\pm0.24}\)
        & \(6.90\times\) \\
        & {Gemma4-12B-it}
        & \(1705.50_{\pm87.72} / 672.72_{\pm34.82}\)
        & \(1 / 2.05_{\pm0.11}\)
        & \(17.06_{\pm0.88} / 13.79_{\pm0.28}\)
        & \(2.54\times\) \\
        \midrule
        \multirow{5}{*}{\rotatebox[origin=c]{90}{\textbf{MBPP}}}
        & {Qwen3-4B}
        & \(747.32_{\pm38.12} / 317.69_{\pm16.22}\)
        & \(1 / 2.00_{\pm0.10}\)
        & \(7.47_{\pm0.38} / 6.35_{\pm0.16}\)
        & \(2.35\times\) \\
        & {Qwen3-8B}
        & \(857.30_{\pm42.10} / 434.58_{\pm29.62}\)
        & \(1 / 1.96_{\pm0.12}\)
        & \(8.57_{\pm0.42} / 8.50_{\pm0.23}\)
        & \(1.97\times\) \\
        & {Qwen3.5-4B}
        & \(2083.62_{\pm128.92} / 389.74_{\pm22.83}\)
        & \(1 / 1.97_{\pm0.13}\)
        & \(20.84_{\pm1.29} / 7.66_{\pm0.15}\)
        & \(5.35\times\) \\
        & {Qwen3.5-9B}
        & \(3088.45_{\pm181.21} / 520.56_{\pm38.97}\)
        & \(1 / 1.83_{\pm0.13}\)
        & \(30.88_{\pm1.81} / 9.46_{\pm0.16}\)
        & \(5.93\times\) \\
        & {Gemma4-12B-it}
        & \(1485.77_{\pm74.38} / 682.29_{\pm31.67}\)
        & \(1 / 1.82_{\pm0.09}\)
        & \(14.86_{\pm0.74} / 12.41_{\pm0.26}\)
        & \(2.18\times\) \\
        \midrule
        \multirow{5}{*}{\rotatebox[origin=c]{90}{\textbf{GPQA}}}
        & {Qwen3-4B}
        & \(760.17_{\pm39.39} / 324.25_{\pm18.52}\)
        & \(1 / 1.97_{\pm0.13}\)
        & \(7.60_{\pm0.39} / 6.38_{\pm0.18}\)
        & \(2.34\times\) \\
        & {Qwen3-8B}
        & \(873.76_{\pm43.90} / 427.11_{\pm25.97}\)
        & \(1 / 2.01_{\pm0.14}\)
        & \(8.74_{\pm0.44} / 8.54_{\pm0.19}\)
        & \(2.05\times\) \\
        & {Qwen3.5-4B}
        & \(2135.71_{\pm126.60} / 378.95_{\pm16.08}\)
        & \(1 / 2.03_{\pm0.10}\)
        & \(21.36_{\pm1.27} / 7.69_{\pm0.14}\)
        & \(5.64\times\) \\
        & {Qwen3.5-9B}
        & \(3197.39_{\pm221.97} / 467.95_{\pm19.99}\)
        & \(1 / 2.04_{\pm0.09}\)
        & \(31.97_{\pm2.22} / 9.51_{\pm0.17}\)
        & \(6.83\times\) \\
        & {Gemma4-12B-it}
        & \(1444.83_{\pm69.75} / 678.05_{\pm50.57}\)
        & \(1 / 1.87_{\pm0.15}\)
        & \(14.45_{\pm0.70} / 12.61_{\pm0.27}\)
        & \(2.13\times\) \\
        \midrule
        \multirow{5}{*}{\rotatebox[origin=c]{90}{\textbf{HumanEval}}}
        & {Qwen3-4B}
        & \(753.01_{\pm39.06} / 341.15_{\pm23.26}\)
        & \(1 / 1.94_{\pm0.14}\)
        & \(7.53_{\pm0.39} / 6.59_{\pm0.18}\)
        & \(2.21\times\) \\
        & {Qwen3-8B}
        & \(867.99_{\pm45.68} / 455.50_{\pm29.38}\)
        & \(1 / 1.89_{\pm0.14}\)
        & \(8.68_{\pm0.46} / 8.56_{\pm0.20}\)
        & \(1.91\times\) \\
        & {Qwen3.5-4B}
        & \(2179.86_{\pm113.33} / 390.14_{\pm23.08}\)
        & \(1 / 1.97_{\pm0.13}\)
        & \(21.80_{\pm1.13} / 7.66_{\pm0.19}\)
        & \(5.59\times\) \\
        & {Qwen3.5-9B}
        & \(3248.00_{\pm177.40} / 525.50_{\pm26.42}\)
        & \(1 / 1.78_{\pm0.09}\)
        & \(32.48_{\pm1.77} / 9.32_{\pm0.18}\)
        & \(6.18\times\) \\
        & {Gemma4-12B-it}
        & \(1473.37_{\pm73.60} / 680.25_{\pm59.15}\)
        & \(1 / 1.89_{\pm0.11}\)
        & \(14.73_{\pm0.74} / 12.81_{\pm1.17}\)
        & \(2.17\times\) \\
        \midrule
        \multirow{5}{*}{\rotatebox[origin=c]{90}{\textbf{MMLU}}}
        & {Qwen3-4B}
        & \(736.92_{\pm35.99} / 345.12_{\pm28.39}\)
        & \(1 / 1.85_{\pm0.17}\)
        & \(7.37_{\pm0.36} / 6.35_{\pm0.22}\)
        & \(2.14\times\) \\
        & {Qwen3-8B}
        & \(858.38_{\pm44.28} / 446.25_{\pm28.97}\)
        & \(1 / 1.90_{\pm0.15}\)
        & \(8.58_{\pm0.44} / 8.44_{\pm0.19}\)
        & \(1.92\times\) \\
        & {Qwen3.5-4B}
        & \(2445.00_{\pm125.15} / 384.31_{\pm22.86}\)
        & \(1 / 2.01_{\pm0.12}\)
        & \(24.45_{\pm1.25} / 7.70_{\pm0.15}\)
        & \(6.36\times\) \\
        & {Qwen3.5-9B}
        & \(3764.61_{\pm197.19} / 544.38_{\pm54.16}\)
        & \(1 / 1.82_{\pm0.21}\)
        & \(37.65_{\pm1.97} / 9.81_{\pm0.21}\)
        & \(6.92\times\) \\
        & {Gemma4-12B-it}
        & \(1673.76_{\pm80.51} / 717.93_{\pm43.53}\)
        & \(1 / 1.92_{\pm0.12}\)
        & \(16.74_{\pm0.81} / 13.74_{\pm0.21}\)
        & \(2.33\times\) \\
        \midrule
        \multirow{5}{*}{\rotatebox[origin=c]{90}{\textbf{LCB V6}}}
        & {Qwen3-4B}
        & \(781.77_{\pm48.77} / 337.69_{\pm8.76}\)
        & \(1 / 1.96_{\pm0.06}\)
        & \(7.82_{\pm0.49} / 6.61_{\pm0.16}\)
        & \(2.32\times\) \\
        & {Qwen3-8B}
        & \(908.99_{\pm62.75} / 436.38_{\pm22.48}\)
        & \(1 / 2.00_{\pm0.13}\)
        & \(9.09_{\pm0.63} / 8.69_{\pm0.25}\)
        & \(2.08\times\) \\
        & {Qwen3.5-4B}
        & \(2197.08_{\pm189.19} / 381.58_{\pm12.13}\)
        & \(1 / 2.04_{\pm0.09}\)
        & \(21.97_{\pm1.89} / 7.79_{\pm0.22}\)
        & \(5.76\times\) \\
        & {Qwen3.5-9B}
        & \(3260.87_{\pm315.74} / 474.97_{\pm21.88}\)
        & \(1 / 2.03_{\pm0.09}\)
        & \(32.61_{\pm3.16} / 9.62_{\pm0.22}\)
        & \(6.87\times\) \\
        & {Gemma4-12B-it}
        & \(1494.62_{\pm85.17} / 673.93_{\pm96.53}\)
        & \(1 / 1.93_{\pm0.19}\)
        & \(14.95_{\pm0.85} / 12.88_{\pm0.87}\)
        & \(2.22\times\) \\
        \bottomrule
    \end{tabular}%
    }
\end{table}

\FloatBarrier
\subsection{Case Study 2: Effect of prefetch budget}
\label{apx:prefetch-budget-traces}

\textbf{Setting.}
On LCB V6 with Qwen3.5-9B under the default BFS traversal at \(K=100\), we compare \method at \(N_{\mathrm{pf}}\in\{8,10,12,16,20\}\) with PowerMH.
The \method runs are the BFS runs of Case Study~5 and the PowerMH run is the vLLM run of Case Study~6; all share one host and were recorded within three days (Appendix~\ref{apx:lcb-provenance}).
These runs did not record KV-pool use or prefix-cache statistics, and no LCB V6 sweep of \methodEnt is available.

\textbf{Results.}
Figure~\ref{fig:prefetch-budget-traces} shows the traces of every budget over the \(K=100\) realized MH transitions, and confirms the findings of Section~\ref{subsec:case-study}.
\begin{itemize}[leftmargin=*, nosep]
    \item \textit{Transitions per call grow with the budget.} They rise from \(3.59\) at \(N_{\mathrm{pf}}=8\) to \(4.91\) at the default \(N_{\mathrm{pf}}=20\) (Table~\ref{tab:prefetch-budget-impact}).
    \item \textit{Each call slows with the budget, so the speedup grows more slowly.} A \method call takes \(16.2\)--\(18.4\) seconds, versus \(5.4\) seconds for a PowerMH step, so the speedup over PowerMH rises from \(1.20\times\) at \(N_{\mathrm{pf}}=8\) to \(1.44\times\) at \(N_{\mathrm{pf}}=20\), the largest budget tested.
\end{itemize}

\begin{figure}[!htbp]
    \centering
    \includegraphics[width=\linewidth,trim=0 0 0 14bp,clip]{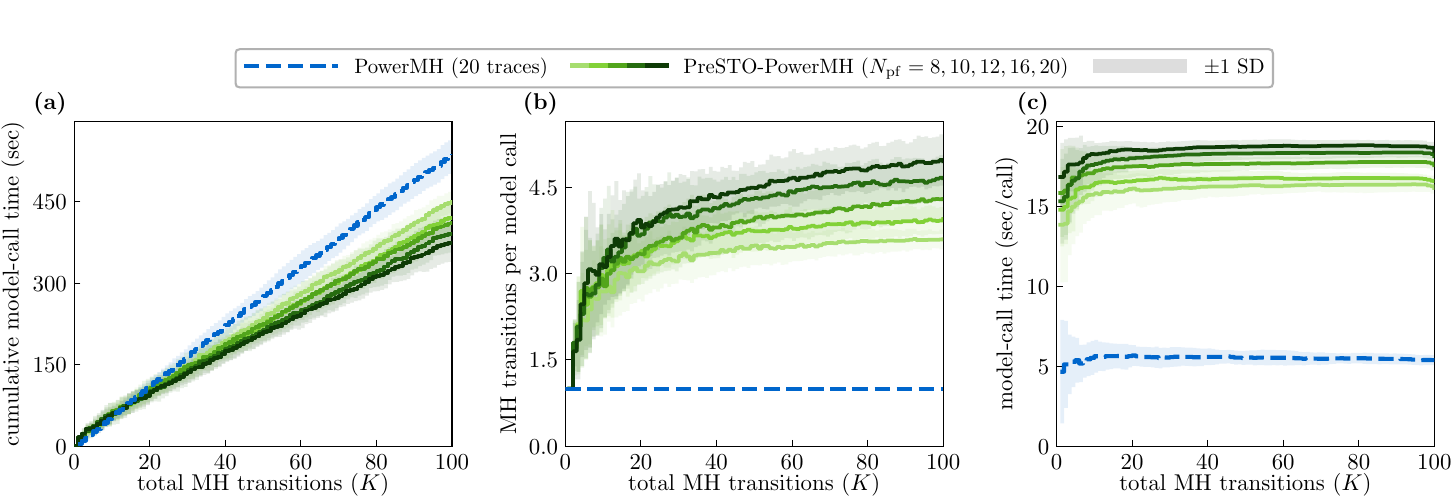}
    \caption{\textbf{Per-transition traces across prefetch budgets on LCB V6 with Qwen3.5-9B.}
    Solid curves show the \framework mean for each budget, from light (small \(N_{\mathrm{pf}}\)) to dark (large \(N_{\mathrm{pf}}\)).
    Dashed curves show PowerMH, and shaded bands show \(\pm 1\) SD over traces.}
    \label{fig:prefetch-budget-traces}
\end{figure}

\subsection{Case Study 3: Likelihood and confidence}
\label{apx:passes}

\textbf{Setting.}
Following \citet{karan2026reasoning}, we describe each terminal trace by its token-normalized log-likelihood (Eq.~\ref{eq:normalized-loglik}) and confidence (Eq.~\ref{eq:confidence}) under \(p_0\).
We compare each sampler with its \framework variant at \(N_{\mathrm{pf}}=20\) on matched dataset--base-LLM pairs: PowerMH against \method, and EntropyCut against \methodEnt.

For each pair and statistic, we apply the paired TOST of Appendix~\ref{apx:exp-set} to the per-prompt difference, \framework minus baseline, with margin \(\delta=0.2\) baseline standard deviations across prompts. For each sampler, we also pool the standardized differences (each difference divided by that standard deviation) over its pairs with a fixed-effect average and test the pooled difference against the same margin.

We compare both statistics only within each base LLM, because changing \(p_0\) changes the scoring distribution itself.

\begin{figure}[!htbp]
    \centering
    \includegraphics[width=\linewidth]{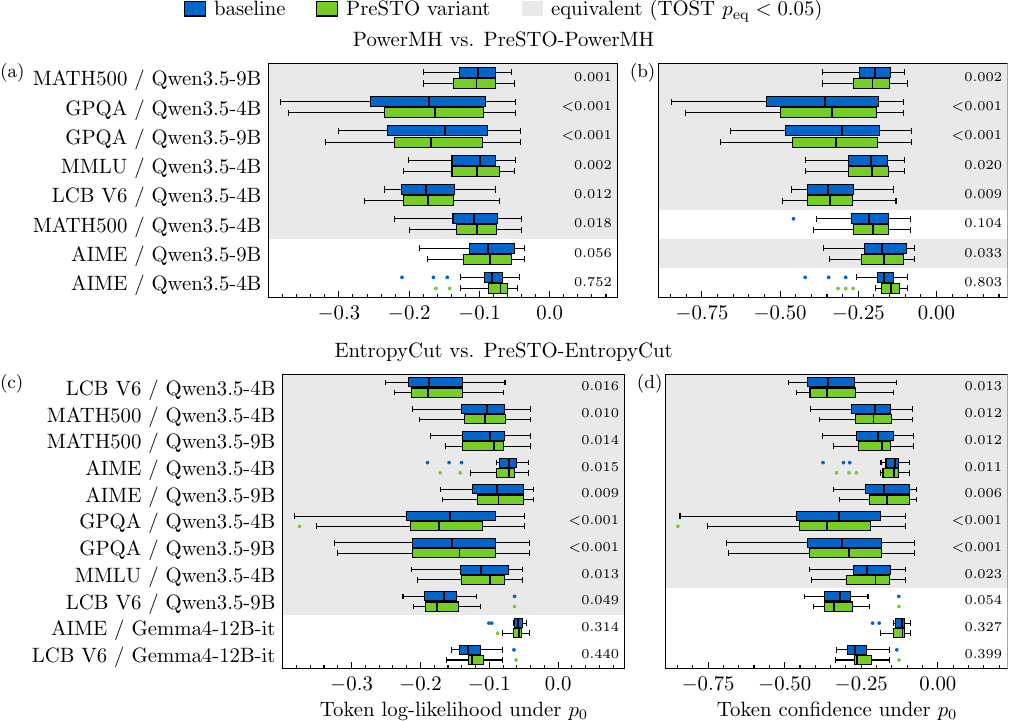}
    \caption{\textbf{Likelihood and confidence distributions with and without \framework.}
    Boxes show the per-prompt token-normalized log-likelihood (a, c) and confidence (b, d) of terminal traces under \(p_0\) for the eight PowerMH pairs (top) and the eleven EntropyCut pairs (bottom), each against its \framework variant.
    Shaded rows are equivalent under the paired TOST (\(p_{\mathrm{eq}}<0.05\)).}
    \label{fig:likelihood-confidence-distributions}
\end{figure}

\textbf{Results.}
Figure~\ref{fig:likelihood-confidence-distributions} shows all retained pairs and confirms the findings of Section~\ref{subsec:case-study}.
\begin{itemize}[leftmargin=*, nosep]
    \item \textit{Equivalence holds for most tests.} 29 of the 38 tests establish equivalence: 12 of 16 for PowerMH and 17 of 22 for EntropyCut (shaded rows). For PowerMH with Qwen3.5-4B, these include 7 of the 10 tests of Figure~\ref{fig:likelihood-confidence}.
    \item \textit{Means nearly coincide.} The median absolute gap between the \framework and baseline means is \(0.002\) in log-likelihood for both samplers, and \(0.007\) for PowerMH and \(0.004\) for EntropyCut in confidence, against across-prompt standard deviations of about \(0.05\) and \(0.10\).
\end{itemize}

\subsection{Case Study 4: Distribution of \texorpdfstring{\(\epsilon\)}{epsilon}-certainty}
\label{subsec:edge-certainty}

\textbf{Setting.}
This case study extends the certainty analysis of Section~\ref{subsec:efficiency} to more base LLMs, more datasets, and the EntropyCut proposal rule.
With \(\epsilon=0.01\), an MH decision with acceptance probability \(A\) is an \(\epsilon\)-certain reject if \(A\leq\epsilon\), an \(\epsilon\)-certain accept if \(A\geq1-\epsilon\), and \(\epsilon\)-uncertain otherwise.

The histograms count both outgoing edges of every scored decision, with transition probabilities \(A\) and \(1-A\), and call an edge \(\epsilon\)-certain when its transition probability is at most \(\epsilon\) or at least \(1-\epsilon\); the two outer categories therefore have equal mass, and their combined share equals the fraction of \(\epsilon\)-certain decisions.

\begin{figure}[!htbp]
    \centering
    \includegraphics[width=\linewidth]{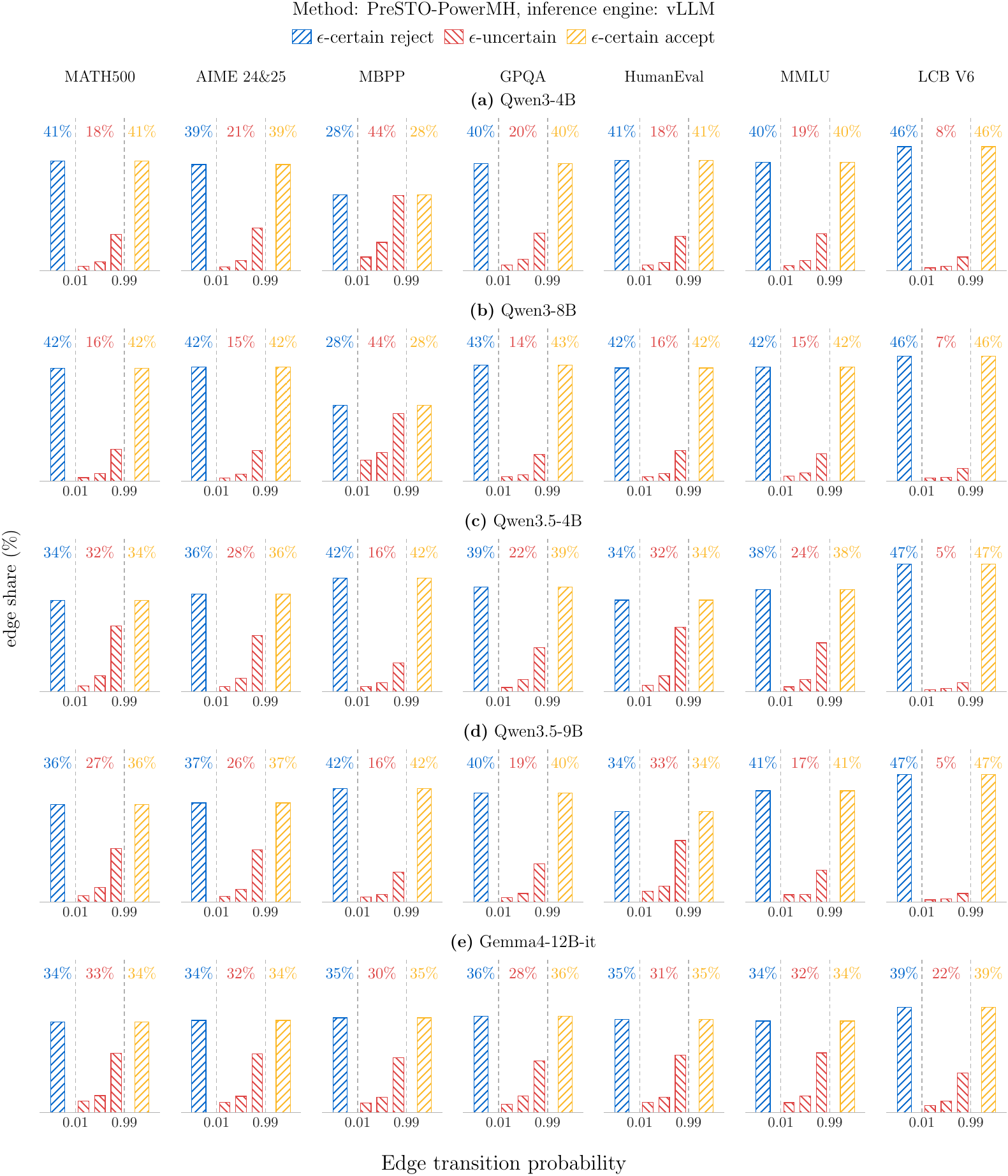}
    \caption{Edge certainty under \method across models and datasets.}
    \label{fig:epsilon-certainty-N20}
\end{figure}
\begin{figure}[!htbp]
  \centering
  \includegraphics[width=\linewidth]{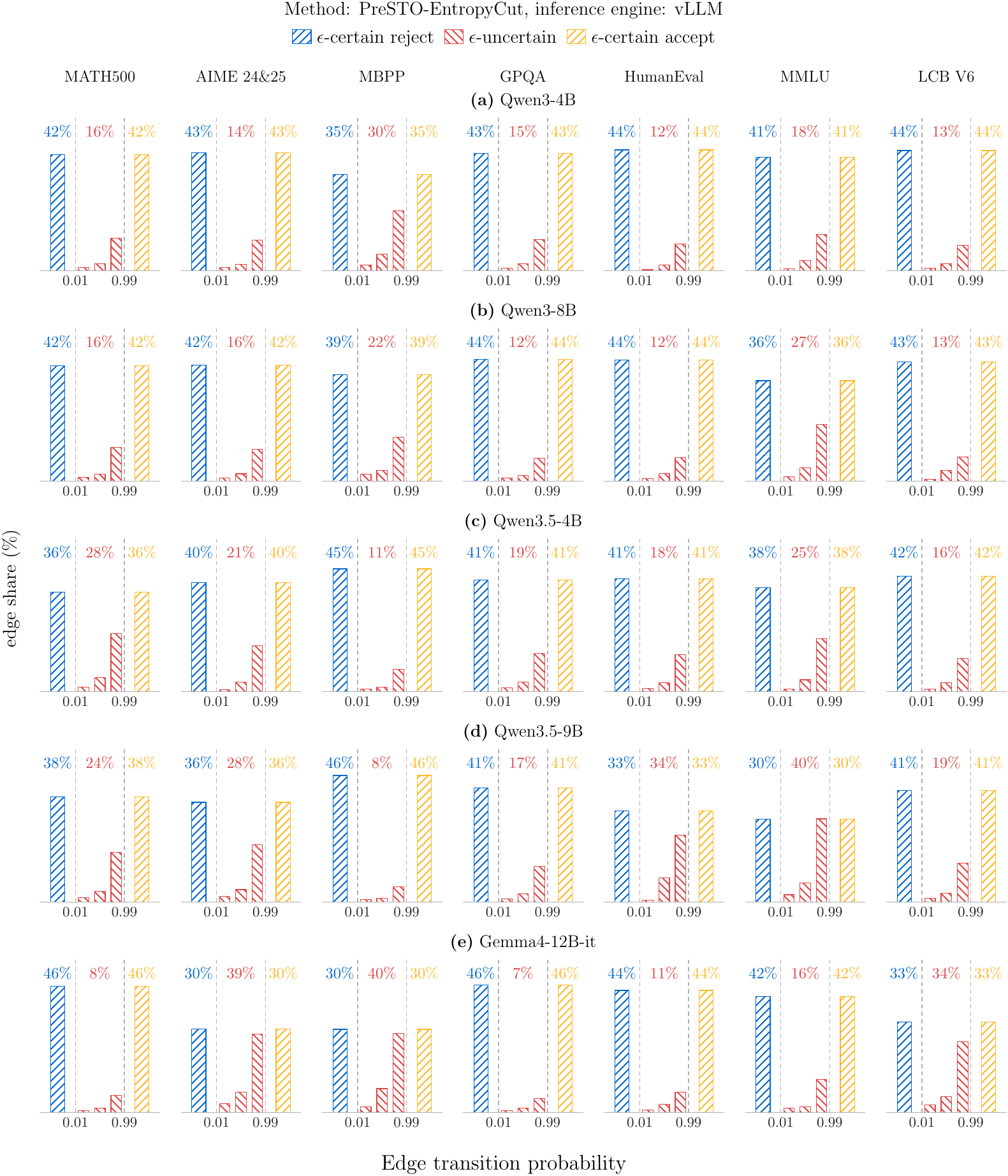}
  \caption{Edge certainty under \methodEnt across models and datasets.}
  \label{fig:entropycut-certainty}
\end{figure}

\textbf{Results.}
Figures~\ref{fig:epsilon-certainty-N20} and~\ref{fig:entropycut-certainty} confirm the findings of Section~\ref{subsec:case-study}. We find that {Most decisions are \(\epsilon\)-certain.} Under \method, \(56\%\)--\(95\%\) of edges are \(\epsilon\)-certain for the Qwen models and \(67\%\)--\(78\%\) for Gemma4-12B-it; under \methodEnt, \(60\%\)--\(92\%\) and \(60\%\)--\(93\%\).

\FloatBarrier
\subsection{Case Study 5: Effect of traversal rules}
\label{apx:traversal-rules}

We compare the predictive traversal rule (Appendix~\ref{subsec:predictive-dp}) with the default BFS traversal; Appendix~\ref{apx:acceptance-predictors} evaluates the classifier that supplies its class probabilities.

\textbf{Setting.}
We run \method with the default BFS traversal and with the predictive traversal rule on LCB V6 with Qwen3.5-9B, using \(N_{\mathrm{pf}}\in\{8,10,12,16,20\}\).
The classifier of the predictive rule is the cut-fraction logistic regression of Appendix~\ref{apx:acceptance-predictors}.

\begin{table}[!htbp]
    \centering
    \caption{\textbf{Predictive traversal rule against the default BFS traversal.} Speedup is the total model-call time of BFS divided by that of the predictive rule.}
    \label{tab:traversal-rules-dp-vs-bfs}
    \resizebox{\linewidth}{!}{%
    \begin{tabular}{@{}c cc cc cc cc c@{}}
        \toprule
        & \multicolumn{2}{c}{Model calls} & \multicolumn{2}{c}{Transitions per call} & \multicolumn{2}{c}{Time per call (s)} & \multicolumn{2}{c}{Model-call time (s)} & \\
        \cmidrule(lr){2-3}\cmidrule(lr){4-5}\cmidrule(lr){6-7}\cmidrule(lr){8-9}
        \(N_{\mathrm{pf}}\) & BFS & Predictive & BFS & Predictive & BFS & Predictive & BFS & Predictive & Speedup \\
        \midrule
        8 & 27.85 & \textbf{23.40} & 3.59 & \textbf{4.27} & 16.15 & \textbf{15.35} & \(449.7_{\pm19.1}\) & \(\mathbf{359.2}_{\pm28.0}\) & \(1.25\times\) \\
        10 & 25.40 & \textbf{20.15} & 3.94 & \textbf{4.96} & 16.56 & \textbf{16.00} & \(420.7_{\pm25.5}\) & \(\mathbf{322.3}_{\pm28.5}\) & \(1.31\times\) \\
        12 & 23.40 & \textbf{18.90} & 4.27 & \textbf{5.29} & 17.58 & \textbf{17.13} & \(411.4_{\pm38.7}\) & \(\mathbf{323.8}_{\pm39.5}\) & \(1.27\times\) \\
        16 & 21.60 & \textbf{17.10} & 4.63 & \textbf{5.85} & 18.15 & \textbf{17.34} & \(392.1_{\pm33.8}\) & \(\mathbf{296.6}_{\pm45.8}\) & \(1.32\times\) \\
        20 & 20.35 & \textbf{15.30} & 4.91 & \textbf{6.54} & 18.43 & \textbf{17.95} & \(375.1_{\pm35.1}\) & \(\mathbf{274.7}_{\pm40.1}\) & \(1.37\times\) \\
        \bottomrule
    \end{tabular}%
    }
\end{table}

\textbf{Results.}
\begin{itemize}[leftmargin=*, nosep]
    \item \textit{Faster at every budget.} The predictive rule is \(1.25\times\)--\(1.37\times\) faster than BFS in model-call time (Figure~\ref{fig:traversal-rules} and Table~\ref{tab:traversal-rules-dp-vs-bfs}) and faster on 98 of the 100 prompt--budget pairs.
    \item \textit{The gain comes from fewer model calls.} The predictive rule completes \(19\%\)--\(33\%\) more MH transitions per call, so it needs \(1.19\times\)--\(1.33\times\) fewer calls for the same \(K\) (Figure~\ref{fig:traversal-rules}(a)); its time per call is also \(2.6\%\)--\(4.9\%\) shorter (Figure~\ref{fig:traversal-rules}(b)).
    \item \textit{The gain is largest at \(N_{\mathrm{pf}}=20\)} (\(1.37\times\)), where the predictive rule raises the transitions per call from \(4.91\) to \(6.54\).
\end{itemize}

Figures~\ref{fig:traversal-rules}(d) and~\ref{fig:traversal-rules-first-trees} show the first prefetched subtree of prompt 0 and of prompts 1--3 under each rule at every budget, one prompt and rule per row. The predictive rule realizes more transitions in 11 of the 20 prompt–budget pairs and fewer in 8 (59 against 52 in total), so its speedup in Table~\ref{tab:traversal-rules-dp-vs-bfs} reflects an average over all trees rather than a gain on every tree.

\begin{figure}[!htbp]
    \centering
    \resizebox{\textwidth}{!}{\begin{minipage}{396.0pt}
    \makebox[24.0pt][l]{}%
    \makebox[48.1pt][c]{\footnotesize \(N_{\mathrm{pf}}=8\)}\hspace{4.0pt}%
    \makebox[65.9pt][c]{\footnotesize \(N_{\mathrm{pf}}=10\)}\hspace{4.0pt}%
    \makebox[73.9pt][c]{\footnotesize \(N_{\mathrm{pf}}=12\)}\hspace{4.0pt}%
    \makebox[83.9pt][c]{\footnotesize \(N_{\mathrm{pf}}=16\)}\hspace{4.0pt}%
    \makebox[83.9pt][c]{\footnotesize \(N_{\mathrm{pf}}=20\)}
    \par\smallskip
    \begin{minipage}[b]{20.0pt}\raggedleft\parbox[b][45.6pt][c]{\linewidth}{\raggedleft\rotatebox[origin=c]{90}{\begin{tabular}[c]{@{}c@{}}\footnotesize Prompt 1\\\footnotesize BFS\end{tabular}}}\par{\scriptsize\strut\par}\end{minipage}\hspace{4.0pt}%
    \begin{minipage}[b]{48.1pt}
        \centering
        \includegraphics[scale=0.234]{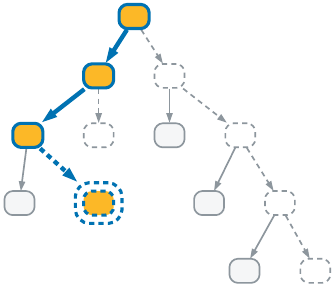}\par
        {\scriptsize MH \(0\rightarrow3\)\par}
    \end{minipage}\hspace{4.0pt}%
    \begin{minipage}[b]{65.9pt}
        \centering
        \includegraphics[scale=0.234]{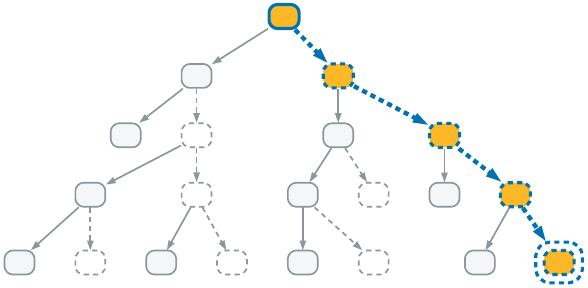}\par
        {\scriptsize MH \(0\rightarrow4\)\par}
    \end{minipage}\hspace{4.0pt}%
    \begin{minipage}[b]{73.9pt}
        \centering
        \includegraphics[scale=0.234]{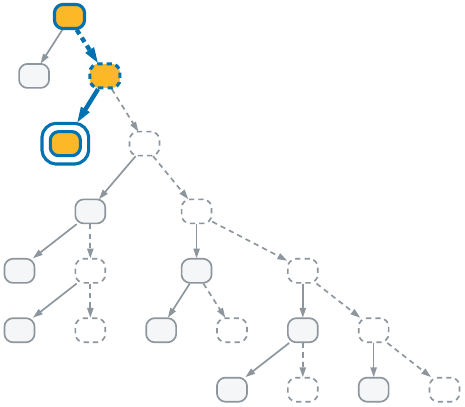}\par
        {\scriptsize MH \(0\rightarrow2\)\par}
    \end{minipage}\hspace{4.0pt}%
    \begin{minipage}[b]{83.9pt}
        \centering
        \includegraphics[scale=0.234]{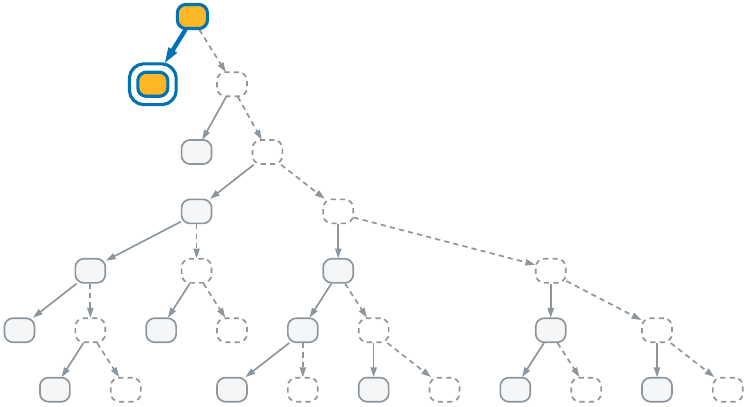}\par
        {\scriptsize MH \(0\rightarrow1\)\par}
    \end{minipage}\hspace{4.0pt}%
    \begin{minipage}[b]{83.9pt}
        \centering
        \includegraphics[scale=0.234]{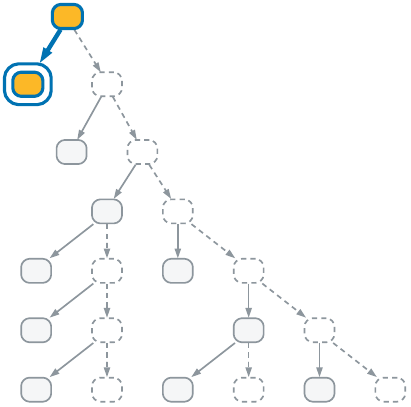}\par
        {\scriptsize MH \(0\rightarrow1\)\par}
    \end{minipage}
    \par\smallskip
    \begin{minipage}[b]{20.0pt}\raggedleft\parbox[b][72.3pt][c]{\linewidth}{\raggedleft\rotatebox[origin=c]{90}{\begin{tabular}[c]{@{}c@{}}\footnotesize Prompt 1\\\footnotesize Predictive\end{tabular}}}\par{\scriptsize\strut\par}\end{minipage}\hspace{4.0pt}%
    \begin{minipage}[b]{48.1pt}
        \centering
        \includegraphics[scale=0.234]{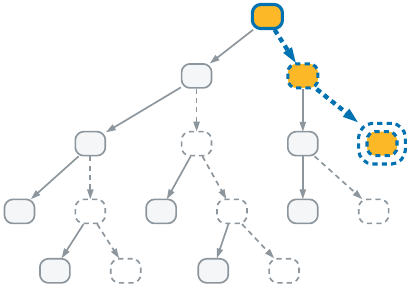}\par
        {\scriptsize MH \(0\rightarrow2\)\par}
    \end{minipage}\hspace{4.0pt}%
    \begin{minipage}[b]{65.9pt}
        \centering
        \includegraphics[scale=0.234]{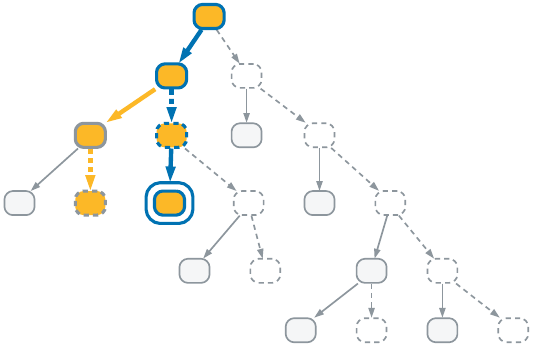}\par
        {\scriptsize MH \(0\rightarrow3\)\par}
    \end{minipage}\hspace{4.0pt}%
    \begin{minipage}[b]{73.9pt}
        \centering
        \includegraphics[scale=0.234]{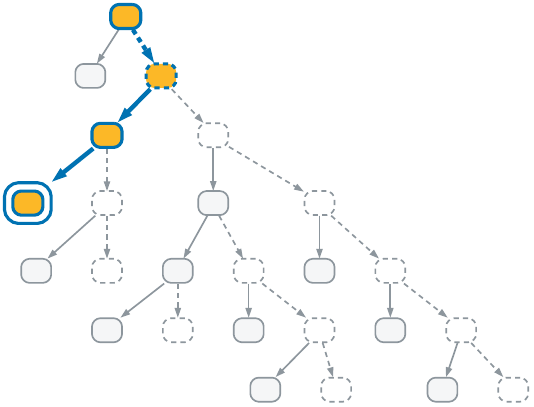}\par
        {\scriptsize MH \(0\rightarrow3\)\par}
    \end{minipage}\hspace{4.0pt}%
    \begin{minipage}[b]{83.9pt}
        \centering
        \includegraphics[scale=0.234]{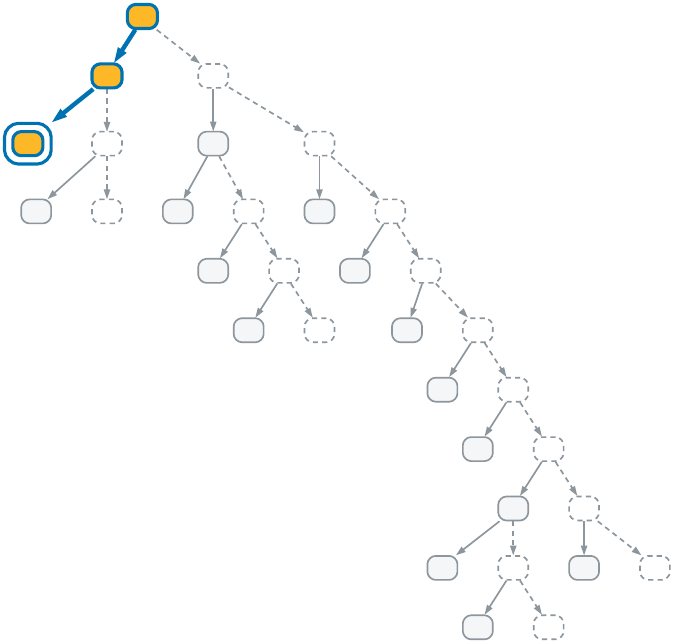}\par
        {\scriptsize MH \(0\rightarrow2\)\par}
    \end{minipage}\hspace{4.0pt}%
    \begin{minipage}[b]{83.9pt}
        \centering
        \includegraphics[scale=0.234]{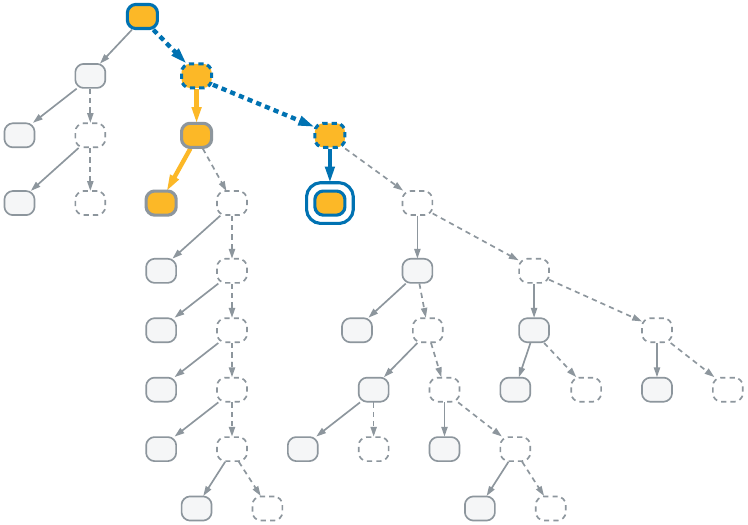}\par
        {\scriptsize MH \(0\rightarrow3\)\par}
    \end{minipage}
    \par\smallskip
    \noindent\tikz\draw[gray!60, dashed, line width=0.4pt] (0,0) -- (\linewidth,0);
    \par\smallskip
    \begin{minipage}[b]{20.0pt}\raggedleft\parbox[b][52.2pt][c]{\linewidth}{\raggedleft\rotatebox[origin=c]{90}{\begin{tabular}[c]{@{}c@{}}\footnotesize Prompt 2\\\footnotesize BFS\end{tabular}}}\par{\scriptsize\strut\par}\end{minipage}\hspace{4.0pt}%
    \begin{minipage}[b]{48.1pt}
        \centering
        \includegraphics[scale=0.234]{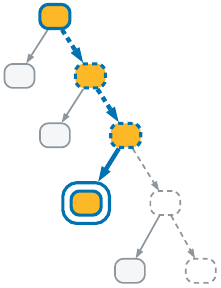}\par
        {\scriptsize MH \(0\rightarrow3\)\par}
    \end{minipage}\hspace{4.0pt}%
    \begin{minipage}[b]{65.9pt}
        \centering
        \includegraphics[scale=0.234]{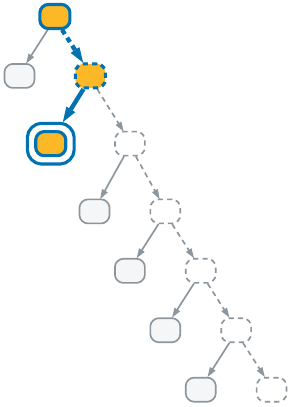}\par
        {\scriptsize MH \(0\rightarrow2\)\par}
    \end{minipage}\hspace{4.0pt}%
    \begin{minipage}[b]{73.9pt}
        \centering
        \includegraphics[scale=0.234]{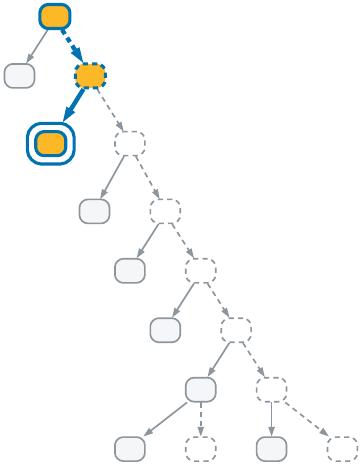}\par
        {\scriptsize MH \(0\rightarrow2\)\par}
    \end{minipage}\hspace{4.0pt}%
    \begin{minipage}[b]{83.9pt}
        \centering
        \includegraphics[scale=0.234]{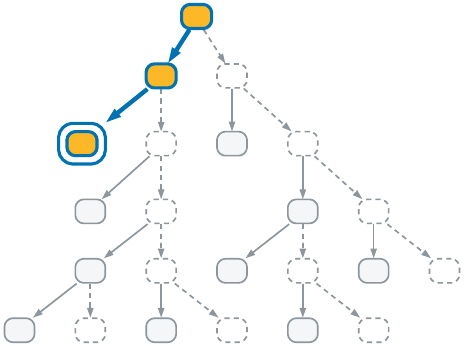}\par
        {\scriptsize MH \(0\rightarrow2\)\par}
    \end{minipage}\hspace{4.0pt}%
    \begin{minipage}[b]{83.9pt}
        \centering
        \includegraphics[scale=0.234]{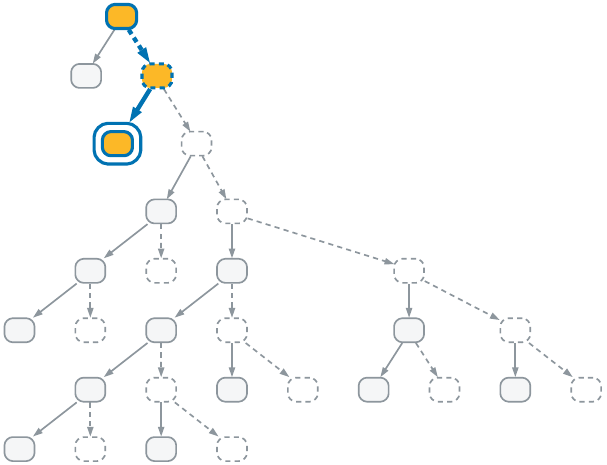}\par
        {\scriptsize MH \(0\rightarrow2\)\par}
    \end{minipage}
    \par\smallskip
    \begin{minipage}[b]{20.0pt}\raggedleft\parbox[b][85.6pt][c]{\linewidth}{\raggedleft\rotatebox[origin=c]{90}{\begin{tabular}[c]{@{}c@{}}\footnotesize Prompt 2\\\footnotesize Predictive\end{tabular}}}\par{\scriptsize\strut\par}\end{minipage}\hspace{4.0pt}%
    \begin{minipage}[b]{48.1pt}
        \centering
        \includegraphics[scale=0.234]{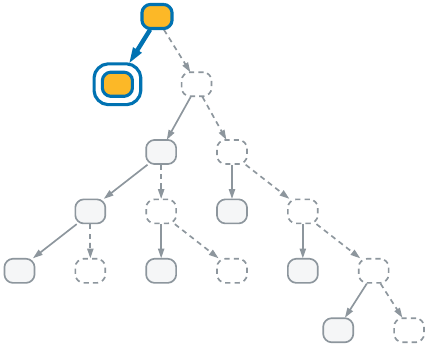}\par
        {\scriptsize MH \(0\rightarrow1\)\par}
    \end{minipage}\hspace{4.0pt}%
    \begin{minipage}[b]{65.9pt}
        \centering
        \includegraphics[scale=0.234]{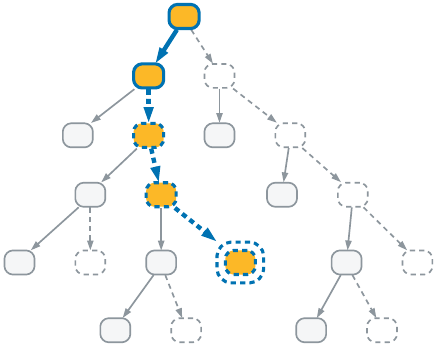}\par
        {\scriptsize MH \(0\rightarrow4\)\par}
    \end{minipage}\hspace{4.0pt}%
    \begin{minipage}[b]{73.9pt}
        \centering
        \includegraphics[scale=0.234]{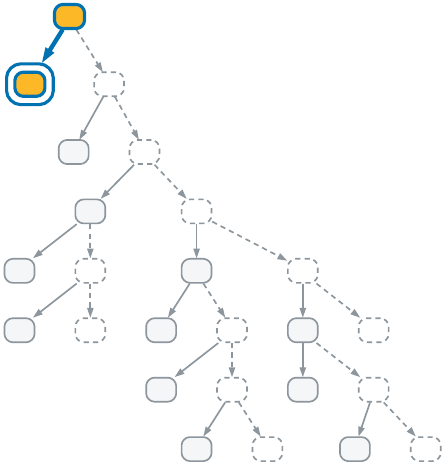}\par
        {\scriptsize MH \(0\rightarrow1\)\par}
    \end{minipage}\hspace{4.0pt}%
    \begin{minipage}[b]{83.9pt}
        \centering
        \includegraphics[scale=0.234]{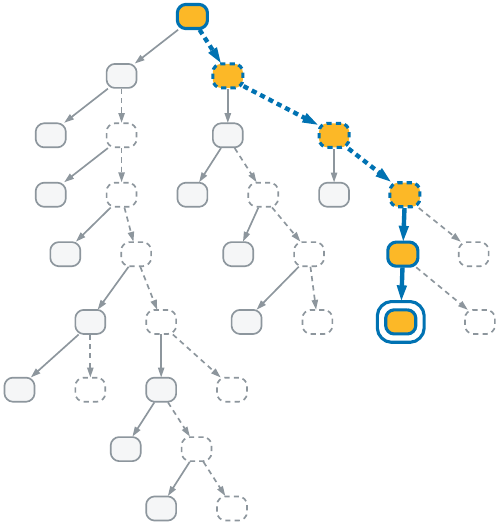}\par
        {\scriptsize MH \(0\rightarrow5\)\par}
    \end{minipage}\hspace{4.0pt}%
    \begin{minipage}[b]{83.9pt}
        \centering
        \includegraphics[scale=0.234]{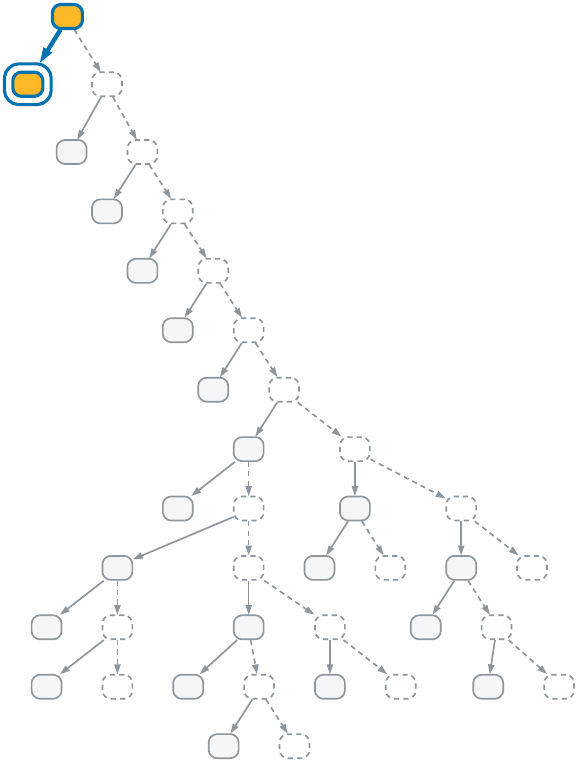}\par
        {\scriptsize MH \(0\rightarrow1\)\par}
    \end{minipage}
    \par\smallskip
    \noindent\tikz\draw[gray!60, dashed, line width=0.4pt] (0,0) -- (\linewidth,0);
    \par\smallskip
    \begin{minipage}[b]{20.0pt}\raggedleft\parbox[b][52.2pt][c]{\linewidth}{\raggedleft\rotatebox[origin=c]{90}{\begin{tabular}[c]{@{}c@{}}\footnotesize Prompt 3\\\footnotesize BFS\end{tabular}}}\par{\scriptsize\strut\par}\end{minipage}\hspace{4.0pt}%
    \begin{minipage}[b]{48.1pt}
        \centering
        \includegraphics[scale=0.234]{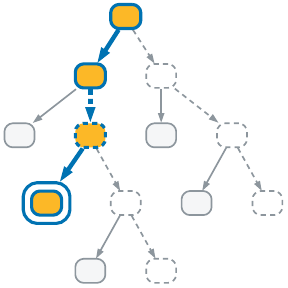}\par
        {\scriptsize MH \(0\rightarrow3\)\par}
    \end{minipage}\hspace{4.0pt}%
    \begin{minipage}[b]{65.9pt}
        \centering
        \includegraphics[scale=0.234]{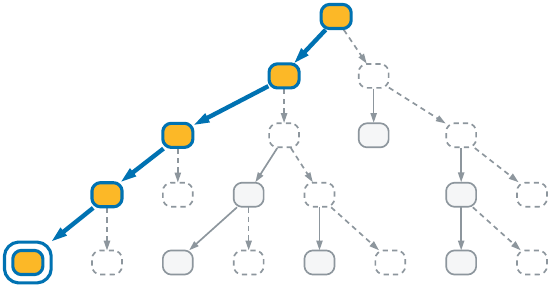}\par
        {\scriptsize MH \(0\rightarrow4\)\par}
    \end{minipage}\hspace{4.0pt}%
    \begin{minipage}[b]{73.9pt}
        \centering
        \includegraphics[scale=0.234]{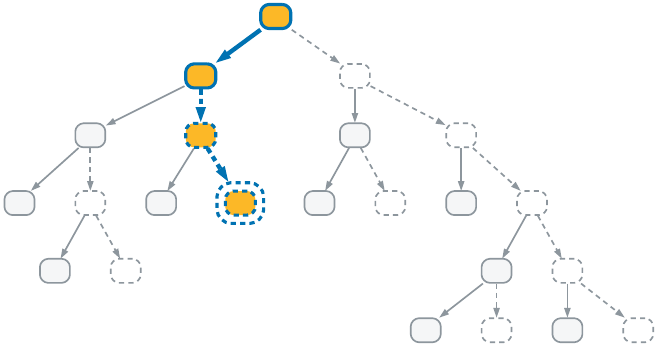}\par
        {\scriptsize MH \(0\rightarrow3\)\par}
    \end{minipage}\hspace{4.0pt}%
    \begin{minipage}[b]{83.9pt}
        \centering
        \includegraphics[scale=0.234]{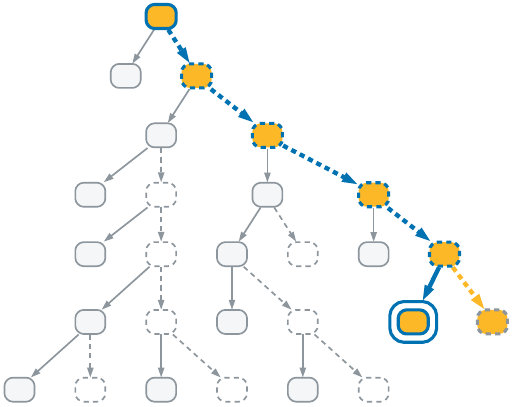}\par
        {\scriptsize MH \(0\rightarrow5\)\par}
    \end{minipage}\hspace{4.0pt}%
    \begin{minipage}[b]{83.9pt}
        \centering
        \includegraphics[scale=0.234]{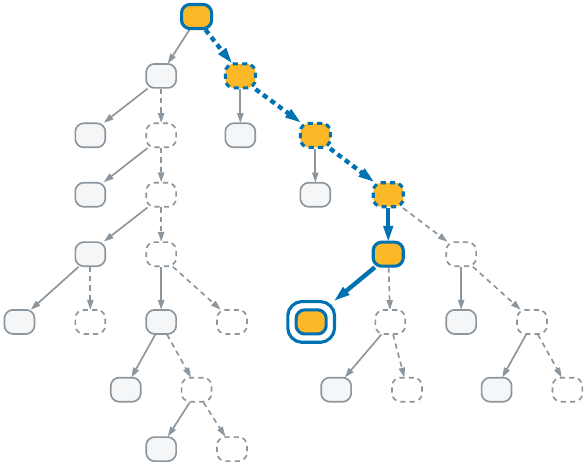}\par
        {\scriptsize MH \(0\rightarrow5\)\par}
    \end{minipage}
    \par\smallskip
    \begin{minipage}[b]{20.0pt}\raggedleft\parbox[b][85.6pt][c]{\linewidth}{\raggedleft\rotatebox[origin=c]{90}{\begin{tabular}[c]{@{}c@{}}\footnotesize Prompt 3\\\footnotesize Predictive\end{tabular}}}\par{\scriptsize\strut\par}\end{minipage}\hspace{4.0pt}%
    \begin{minipage}[b]{48.1pt}
        \centering
        \includegraphics[scale=0.234]{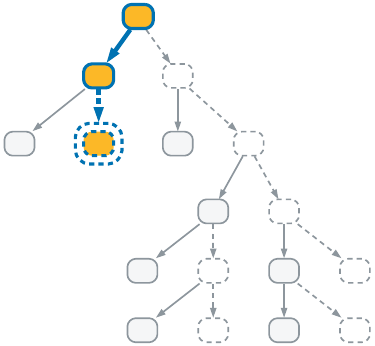}\par
        {\scriptsize MH \(0\rightarrow2\)\par}
    \end{minipage}\hspace{4.0pt}%
    \begin{minipage}[b]{65.9pt}
        \centering
        \includegraphics[scale=0.234]{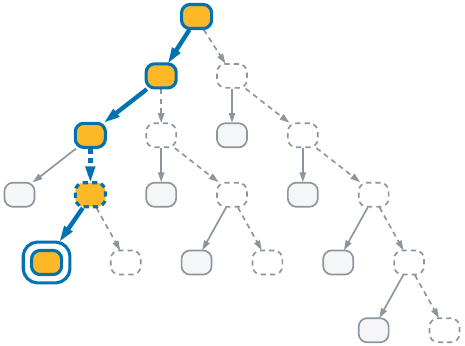}\par
        {\scriptsize MH \(0\rightarrow4\)\par}
    \end{minipage}\hspace{4.0pt}%
    \begin{minipage}[b]{73.9pt}
        \centering
        \includegraphics[scale=0.234]{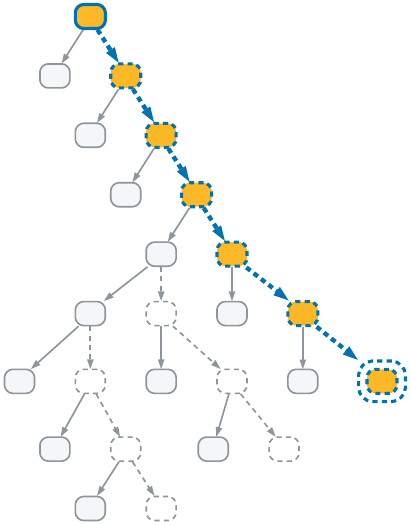}\par
        {\scriptsize MH \(0\rightarrow6\)\par}
    \end{minipage}\hspace{4.0pt}%
    \begin{minipage}[b]{83.9pt}
        \centering
        \includegraphics[scale=0.234]{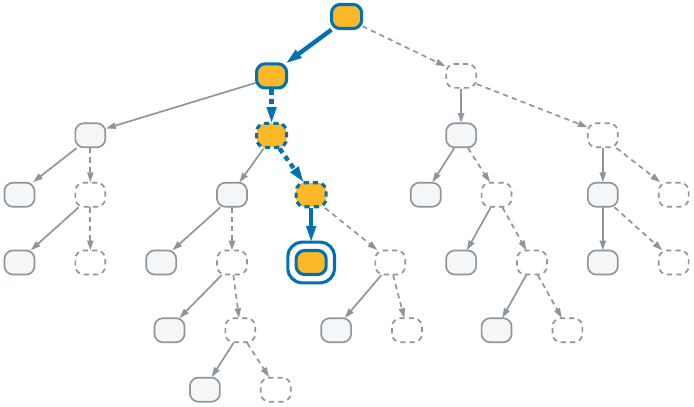}\par
        {\scriptsize MH \(0\rightarrow4\)\par}
    \end{minipage}\hspace{4.0pt}%
    \begin{minipage}[b]{83.9pt}
        \centering
        \includegraphics[scale=0.234]{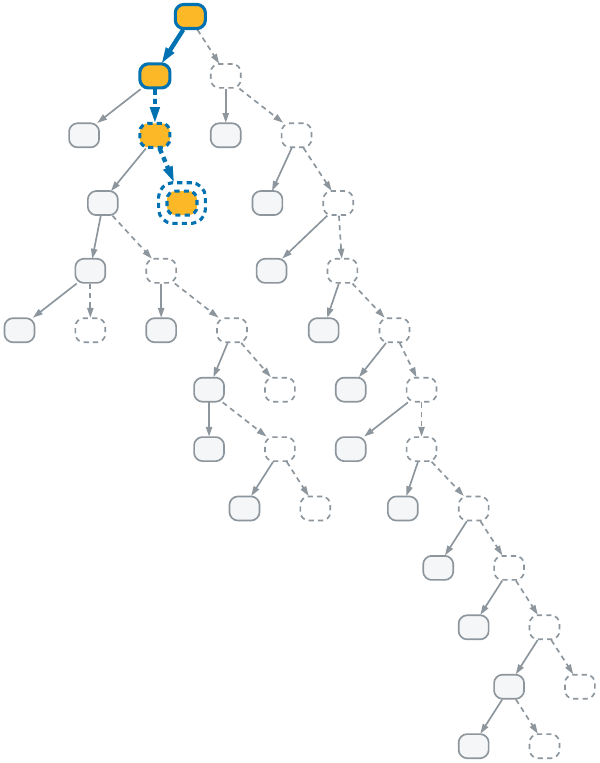}\par
        {\scriptsize MH \(0\rightarrow3\)\par}
    \end{minipage}
    \end{minipage}}
    \caption{\textbf{First prefetched subtree under each traversal rule and prefetch budget.} Prompts 1--3 of LCB V6 with Qwen3.5-9B under \method, one row per prompt and rule; labels and colors follow Figure~\ref{fig:traversal-rules}d.}
    \label{fig:traversal-rules-first-trees}
\end{figure}

\FloatBarrier
\subsection{Case Study 6: Effect of different serving engines}
\label{apx:serving-engines}

\textbf{Setting.}
All main results use vLLM; to check that the acceleration is not specific to it, we also run PowerMH and \method on SGLang.
The two engines draw different random samples, so their chains and numbers of calls differ.

\begin{table}[!htbp]
    \centering
    \caption{\textbf{PowerMH and \method on vLLM and SGLang.}}
    \label{tab:serving-engines}
    \begin{tabular}{llcccc}
        \toprule
        Engine & Sampler & Calls & Time per call (s) & Model-call time (s) & Speedup \\
        \midrule
        \multirow{2}{*}{vLLM} & PowerMH & \(100\) & \(5.39_{\pm0.31}\) & \(539.24_{\pm30.92}\) & \multirow{2}{*}{\(1.40\times\)} \\
        & \method & \(20.60_{\pm1.31}\) & \(18.75_{\pm0.43}\) & \(385.97_{\pm21.25}\) & \\
        \midrule
        \multirow{2}{*}{SGLang} & PowerMH & \(100\) & \(3.57_{\pm0.24}\) & \(357.40_{\pm23.67}\) & \multirow{2}{*}{\(1.97\times\)} \\
        & \method & \(21.25_{\pm1.86}\) & \(8.56_{\pm0.10}\) & \(181.74_{\pm15.25}\) & \\
        \bottomrule
    \end{tabular}%
\end{table}

\textbf{Results.}
\method is faster than PowerMH under both engines, so the acceleration is not specific to vLLM (Table~\ref{tab:serving-engines}).
\begin{itemize}[leftmargin=*, nosep]
    \item \textit{The speedup is larger under SGLang} (\(1.97\times\) against \(1.40\times\) under vLLM). Both engines cut the number of calls by a similar factor (\(4.7\times\) and \(4.9\times\)), but a \method call costs \(2.4\times\) a PowerMH step under SGLang against \(3.5\times\) under vLLM.
    \item \textit{SGLang is faster in absolute time:} \(1.5\times\) faster than vLLM for PowerMH and \(2.1\times\) for \method.
\end{itemize}

\FloatBarrier
\subsection{Case Study 7: Predicting \texorpdfstring{\(\epsilon\)}{epsilon}-certainty before generation}
\label{apx:acceptance-predictors}
The predictive traversal rule (Appendix~\ref{subsec:predictive-dp}) needs the class probabilities \(\widehat P(v\text{ is $\epsilon$-certain reject})\) and \(\widehat P(v\text{ is $\epsilon$-certain accept})\) of every candidate node \(v\) before the proposal at \(v^{+}\) is generated.
We therefore predict, from information available before generation, which of the three classes of Section~\ref{subsec:efficiency} each candidate node falls into: certain reject (\(A\leq\epsilon\)), certain accept (\(A\geq1-\epsilon\)), or uncertain (\(\epsilon<A<1-\epsilon\)).
We use the results to choose the input features and the model of the rule.

\textbf{Setting.}
Each run records the acceptance probability \(A=A(v^{+},v)\) of every prefetched node, and these nodes are the examples.
We use one run per sampler, on the first 20 prompts of LCB V6 with Qwen3.5-9B.
The \method dataset has \(4{,}621\) nodes from \(579\) model calls: \(2{,}814\) \(\epsilon\)-certain rejects, \(945\) \(\epsilon\)-certain accepts, and \(862\) \(\epsilon\)-uncertain.
The \methodEnt dataset has \(4{,}533\) nodes from \(568\) model calls: \(2{,}701\), \(880\), and \(952\), respectively.
These are about \(8\) nodes per call, fewer than \(N_{\mathrm{pf}}=10\), for two reasons.
First, the calls issue \(5{,}690\) and \(5{,}546\) requests; we drop the \(1{,}069\) and \(1{,}013\) requests that BFS places below the rejection child of a node whose acceptance-child request it deferred under Condition~\ref{cond:prefetch}, because the MH walk of that call stops at that node and does not reach them.
Second, \(21\) and \(22\) calls near the end of a prompt, when at most \(6\) of the \(K\) transitions remain, issue fewer than \(10\) requests.
Because the examples are the nodes that BFS actually prefetched, they follow the distribution of the prefetched subtree, not that of all candidate nodes.
We group the features, all available when the batch is assembled, into four sets:
\begin{itemize}[leftmargin=*, nosep]
    \item \textit{Feature 1 (cut fraction):} the cut fraction \(c_j/T\) of the request, where \(v_j\) is the acceptance child of the node \(v_i\) (that is, \(v_i=v\) and \(v_j=v^{+}\)).
    \item \textit{Feature 2 (tree position):} the cut index \(c_i\) of \(v_i\), the gap \(c_i-c_j\), the depth of \(v_i\), and the number of acceptances on the path from the root to \(v_i\).
    \item \textit{Feature 3 (likelihood):} four statistics of the token-wise log-likelihood of the root state under the proposal LLM: its value at the first resampled token, its changes into and out of that token, and its mean over the suffix after the cut.
    \item \textit{Feature 4 (entropy, \methodEnt only):} the same four statistics of the token-wise entropy of the root state; the uniform-cut run did not log token entropies.
\end{itemize}

\textbf{Classifier.}
Each node \(v\) has a standardized feature vector \(x\in\mathbb{R}^{d}\) and a true class \(y(v)\in\{1,2,3\}\): \(\epsilon\)-certain reject, \(\epsilon\)-certain accept, or \(\epsilon\)-uncertain, in the order of Appendix~\ref{subsec:predictive-dp}.
We use a three-class logistic regression, that is, a linear layer with weights \(W\) followed by a softmax over the three classes.
It thus outputs \(\widehat P(v\text{ is $\epsilon$-certain reject})\) and \(\widehat P(v\text{ is $\epsilon$-certain accept})\) directly, and their sum is at most one by construction.
To test whether a nonlinear model helps, we also train multilayer perceptrons that insert hidden layers with the \(\operatorname{ReLU}\) activation before the softmax; with one hidden layer,
\begin{equation}
\label{eq:certainty-mlp}
\widehat P(y=k\mid v)=\operatorname{softmax}_k\bigl(W_2\operatorname{ReLU}(W_1x+b_1)+b_2\bigr),\qquad k=1,2,3,
\end{equation}
and a second hidden layer adds another linear layer followed by \(\operatorname{ReLU}\); we use one hidden layer of \(16\), \(32\), or \(64\) units, or two hidden layers of \(32\) or \(64\) units each.
All models are trained to minimize the log loss (cross-entropy) on a training dataset \(\mathcal{D}\) of nodes,
\begin{equation}
\label{eq:certainty-log-loss}
\mathcal{L}(\mathcal{D})=-\frac{1}{|\mathcal{D}|}\sum_{v\in\mathcal{D}}\log\widehat P\bigl(y=y(v)\mid v\bigr),
\end{equation}
with the natural logarithm, plus the \(L_2\) penalty \(\frac{\lambda}{2|\mathcal{D}|}\) times the sum of squared weights, biases excluded: \(\lambda=1\) for the logistic regression (scikit-learn's \texttt{C}\(\,=1/\lambda=1\)), and \(\lambda=10^{-3}\) for the multilayer perceptrons (scikit-learn's \texttt{alpha}), which are trained with Adam and stopped early on a \(10\%\) validation split of the training fold.

\textbf{Evaluation metrics.}
We evaluate each classifier by 5-fold cross-validation with whole prompts held out and pool the out-of-fold predictions, so every node \(v\) in the set \(\mathcal{U}\) of all nodes is scored by a model that never saw its prompt.
Let \(s_k(v)=\widehat P(y=k\mid v)\) be the held-out probability of class \(k\).
We report two metrics, each against a constant baseline that ignores the features:
\begin{itemize}[leftmargin=*, nosep]
    \item \textit{Accuracy:} the fraction of nodes whose most likely class \(\hat y(v)=\arg\max_{k}s_k(v)\) equals \(y(v)\); the baseline always predicts the most frequent class.
    \item \textit{Log loss:} \(\mathcal{L}(\mathcal{U})\) in~\eqref{eq:certainty-log-loss}, computed with the held-out probabilities \(s_k(v)\); the baseline predicts the class frequencies for every node.
\end{itemize}
Log loss scores the predicted probabilities, which the predictive traversal rule uses; accuracy scores only the most likely class \(\hat y(v)\).

\begin{figure}[!htbp]
    \centering
    \includegraphics[width=\linewidth]{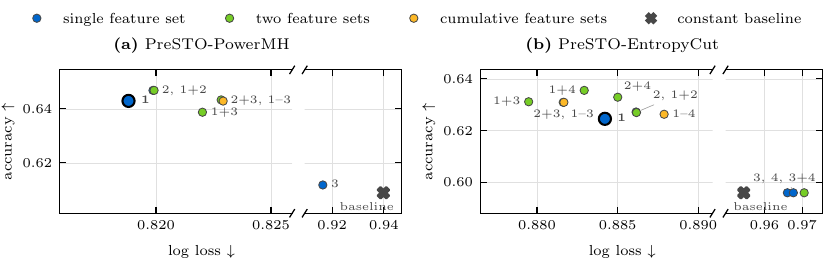}
    \caption{\textbf{Prediction of the \(\epsilon\)-certainty class by feature set.}
    Held-out accuracy against log loss of the logistic regression for \method (a) and \methodEnt (b); upper left is better.
    Labels give the feature sets used: a single digit is one set, ``1+3'' combines Features~1 and~3, and ``1--3'' combines Features~1 through~3; comma-separated labels mark overlapping points.
    The cross marks the constant baseline, and the bold label is the configuration we use.}
    \label{fig:acceptance-predictors}
\end{figure}

\begin{figure}[!htbp]
    \centering
    \includegraphics[width=\linewidth]{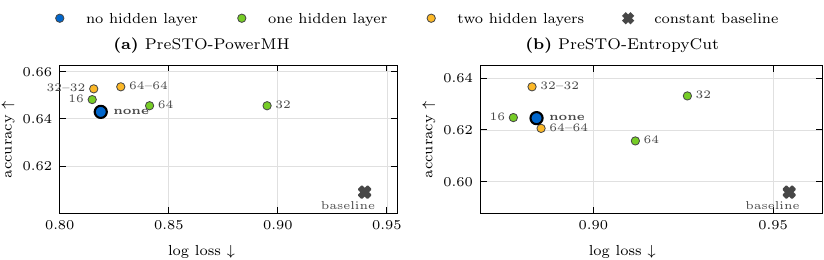}
    \caption{\textbf{Effect of hidden layers on predicting the \(\epsilon\)-certainty class.}
    Held-out accuracy against log loss with the cut fraction as the only feature; labels give the units per hidden layer, and ``none'' is the logistic regression.
    Markers as in Figure~\ref{fig:acceptance-predictors}.}
    \label{fig:acceptance-predictors-mlp}
\end{figure}

\textbf{Results.}
We use the logistic regression with the cut fraction \(c_j/T\) as the only feature.
It reaches a held-out accuracy of \(0.643\) under \method and \(0.625\) under \methodEnt, above the majority-class rates of \(0.609\) and \(0.596\), and lowers the log loss from \(0.940\) to \(0.819\) and from \(0.954\) to \(0.884\) relative to the class prior (Figure~\ref{fig:acceptance-predictors}).
The other choices are not useful.

\FloatBarrier
\subsection{Case Study 8: Comparison with test-time baselines}
\label{apx:scaling-baselines}

\method speeds up PowerMH without changing its transition kernel, but PowerMH is only one way to spend test-time compute.
This case study places both among three baselines that also sharpen the base distribution or select among samples: low-temperature sampling, best-of-\(N\)~\citep{DBLP:conf/nips/StiennonO0ZLVRA20,DBLP:conf/nips/SunHZYQYWBZ24}, and PowerSMC~\citep{DBLP:arxiv/power-smc}.
Low-temperature sampling and best-of-\(N\) do not target \(p_0^{\alpha}\): scaling each token's logits by \(\alpha\) renormalizes every conditional separately, which differs from the sequence-level power distribution~\citep{karan2026reasoning}.
We compare their generation time, peak KV-pool use, and accuracy on the same prompts.

\textbf{Setting.}
All runs use the first \(20\) prompts of LCB V6 with Qwen3.5-9B, process one prompt at a time on vLLM, and target \(\alpha=4\).
PowerMH and \method ignore the end-of-sequence token and always generate all \(1{,}024\) new tokens, whereas low-temperature sampling, best-of-\(N\), and PowerSMC stop a sample at the end-of-sequence token, so they generate up to \(1{,}024\) new tokens.
PowerMH and \method run a single block of \(K\in\{10,20,30\}\) MH transitions, fewer than the \(K=100\) of the main runs, and \method uses the default BFS traversal at \(N_{\mathrm{pf}}=10\).
Low-temperature sampling draws one sample at temperature \(1/\alpha=0.25\); best-of-\(N\) draws \(N\in\{4,8,16,32\}\) samples at the same temperature and keeps the one with the highest unnormalized sequence log-probability, the sum of its token log-probabilities under this temperature-\(0.25\) sampling distribution rather than under \(p_0\).
PowerSMC grows \(P\in\{2,4,8,16,32\}\) particles in blocks of \(128\) tokens, resamples when the effective sample size falls below half, and ramps the exponent to \(\alpha\) over the first \(400\) tokens.
Time is the generation time per prompt, excluding model loading and grading, and peak KV-pool use is the largest fraction of the preallocated KV pool in use.
Figure~\ref{fig:scaling-baselines} plots peak KV-pool use against time.

\begin{figure}[!htbp]
    \centering
    \includegraphics[width=0.95\linewidth]{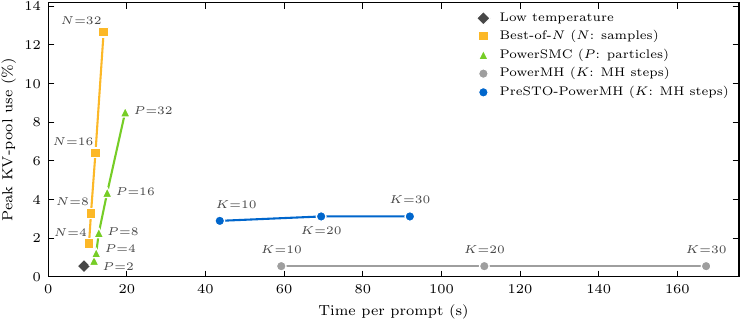}
    \caption{\textbf{Speed and memory against test-time baselines.}
    Peak KV-pool use against time per prompt on the first 20 prompts of LCB V6 with Qwen3.5-9B; lower left is better.}
    \label{fig:scaling-baselines}
\end{figure}

\textbf{Results.}
Figure~\ref{fig:scaling-baselines} shows that prefetching raises peak KV-pool use, but by less than the parallel baselines at their larger settings.
Best-of-\(N\) and PowerSMC decode all samples or particles at once, so their peak KV-pool use grows with \(N\) or \(P\), reaching \(12.7\%\) of the pool at \(N=32\) and \(8.5\%\) at \(P=32\).
\method batches \(10\) proposals and peaks at \(2.9\%\)--\(3.1\%\) at every \(K\), close to best-of-\(8\) (\(3.3\%\)), compared with \(0.6\%\) for PowerMH.
\method therefore narrows, but does not close, the cost gap between PowerMH and these baselines.

\newpage

\subsection{Example prefetched-subtree trajectory}
\label{subsec:tree-visualization}

We print the raw log of one prefetched subtree.
Each line of the log is a node \(v_i\), printed as
\texttt{Node(id=\(i\), cut\_idx=\(c_i\))}.
Node \(v_i\) has acceptance child \(v_{2i+1}\) and rejection child
\(v_{2i+2}\); the rejection child keeps its parent's state and cut index.
Each acceptance edge carries \texttt{log\_mh\_ratio}, the unclamped log MH
ratio of the proposal at \(v_{2i+1}\), so the edge is taken with probability
\(\min\{1,\exp(\texttt{log\_mh\_ratio})\}\); rejection edges carry no value.
The tree holds 16 of the 20 requests in the batch.
The other 4 lie below the rejection child of a node whose own request
violates Condition~\ref{cond:prefetch}, such as \(v_9\); the chain cannot
pass such a node in this batch, so the sampler does not attach these
requests, and the budget runs out before BFS reaches frontier nodes such as \(v_{30}\).

\refstepcounter{prefetchbox}\label{box:prefetched-original-1}
\TreeBox[
  basicstyle=\fontsize{8}{9}\selectfont\ttfamily
]{\small Box~\theprefetchbox: Prefetched tree}{exps/tree_visualization/lcb-bfs-prefetched-subtree-tree-1.txt}

\refstepcounter{prefetchbox}\label{box:prefetched-path-1}
  \begin{tcolorbox}[
    enhanced,
    breakable,
    boxrule=0.4pt,
    titlerule=0.4pt,
    colframe=TreeBoxFrameColor,
    colback=TreeBoxBackgroundColor,
    colbacktitle=TreeBoxTitleBackgroundColor,
    coltitle=TreeBoxTitleColor,
    fonttitle=\normalsize\ttfamily\bfseries,
    fontupper=\small,
    before upper={%
      \setlength{\abovedisplayskip}{3pt}%
      \setlength{\belowdisplayskip}{3pt}%
      \setlength{\abovedisplayshortskip}{0pt}%
      \setlength{\belowdisplayshortskip}{3pt}%
    },
    title={\small Box~\theprefetchbox: MH progress 0/100 $\rightarrow$ 2/100 (2 transitions)},
    title after break={\small Box~\theprefetchbox: MH progress 0/100 $\rightarrow$ 2/100 (2 transitions) (continued)},
    before skip=4pt,
    after skip=4pt,
    boxsep=1pt,
    left=1pt,
    right=1pt,
    top=1pt,
    bottom=1pt,
    colframe=black,colback=white,colbacktitle=white,coltitle=black,colupper=black
  ]
\begin{lstlisting}[
  style=prefetchtree
]
sampled path in the prefetched subtree: 0 → 1 → 3
\end{lstlisting}
  \end{tcolorbox}%

The sampled path (Box~\ref{box:prefetched-path-1}) accepts at \(v_0\) and
\(v_1\), advancing the chain from step 0 to step 2.
Both decisions are \(\epsilon\)-certain at \(\epsilon=0.01\): the ratio at
\(v_0\) accepts with
probability above \(0.9999\), and \(v_1\) accepts with probability 1.
The \(\epsilon\)-rounded tree (Box~\ref{box:prefetched-rounded-1};
Definition~\ref{def:finite-tree-continuation}) drops each branch that
rounding assigns zero probability, together with its descendants, and keeps
both branches at uncertain decisions, so here it reduces to the single path
\(v_0\to v_1\to v_3\), which is the sampled path.

\colorlet{TreeBoxFrameColor}{sglOrange!65!black}
\colorlet{TreeBoxBackgroundColor}{sglOrange!7}
\colorlet{TreeBoxTitleBackgroundColor}{sglOrange!25}
\colorlet{TreeBoxTitleColor}{sglOrange!35!black}
\refstepcounter{prefetchbox}\label{box:prefetched-rounded-1}
\TreeBox[
  basicstyle=\fontsize{8}{9}\selectfont\ttfamily
]{\small Box~\theprefetchbox: $\epsilon$-rounded tree ($\epsilon=0.01$)}{exps/tree_visualization/lcb-bfs-prefetched-subtree-rounded.txt}

\paragraph{Binary-tree views of prefetched subtrees.}

We plot all 19 prefetched subtrees for prompt index 0 from the
Qwen3.5-9B run on LCB V6 with
prefetch budget \(N_{\mathrm{pf}}=20\).
Figure~\ref{fig:lcb-binary-00} shows them in sampling order; together
they cover the chain's full trajectory of 100 MH transitions.

\begin{figure}[!htbp]
\refstepcounter{figure}
\label{fig:lcb-binary-00}
\addcontentsline{lof}{figure}{\protect\numberline{\thefigure}{Prefetched subtrees for prompt 0}}
\centering
\begin{minipage}[b]{0.24\linewidth}
\centering
\refstepcounter{prefetchpanel}
\parbox[b][0.1\textheight][b]{\linewidth}{\centering\includegraphics[width=\linewidth,height=0.1\textheight,keepaspectratio]{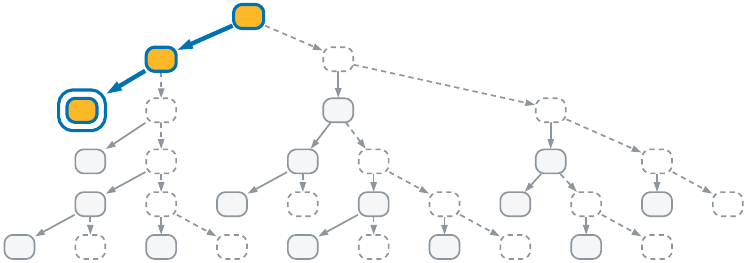}}
\par\smallskip
{\footnotesize (\arabic{prefetchpanel}) MH \(0\rightarrow2\)\par}
\end{minipage}%
\hfill
\begin{minipage}[b]{0.24\linewidth}
\centering
\refstepcounter{prefetchpanel}
\parbox[b][0.1\textheight][b]{\linewidth}{\centering\includegraphics[width=\linewidth,height=0.1\textheight,keepaspectratio]{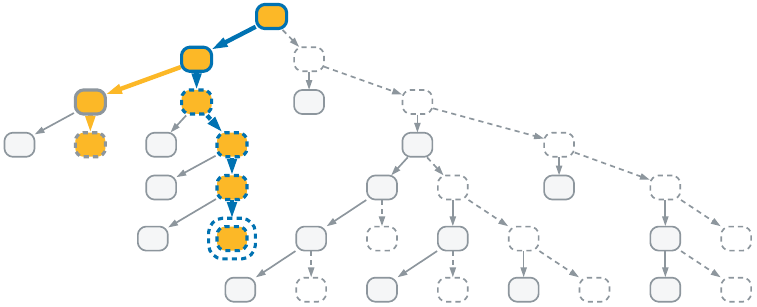}}
\par\smallskip
{\footnotesize (\arabic{prefetchpanel}) MH \(2\rightarrow7\)\par}
\end{minipage}%
\hfill
\begin{minipage}[b]{0.24\linewidth}
\centering
\refstepcounter{prefetchpanel}
\parbox[b][0.1\textheight][b]{\linewidth}{\centering\includegraphics[width=\linewidth,height=0.1\textheight,keepaspectratio]{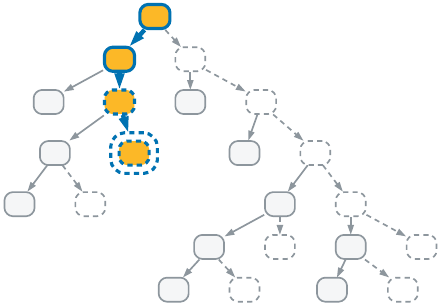}}
\par\smallskip
{\footnotesize (\arabic{prefetchpanel}) MH \(7\rightarrow10\)\par}
\end{minipage}%
\hfill
\begin{minipage}[b]{0.24\linewidth}
\centering
\refstepcounter{prefetchpanel}
\parbox[b][0.1\textheight][b]{\linewidth}{\centering\includegraphics[width=\linewidth,height=0.1\textheight,keepaspectratio]{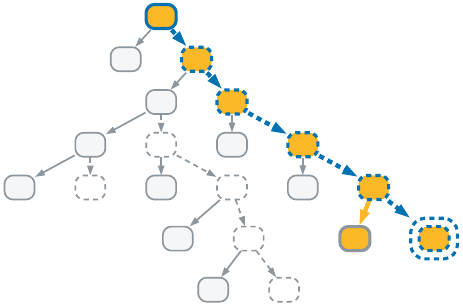}}
\par\smallskip
{\footnotesize (\arabic{prefetchpanel}) MH \(10\rightarrow15\)\par}
\end{minipage}%
\par
\vspace{6pt}
\begin{minipage}[b]{0.24\linewidth}
\centering
\refstepcounter{prefetchpanel}
\parbox[b][0.1\textheight][b]{\linewidth}{\centering\includegraphics[width=\linewidth,height=0.1\textheight,keepaspectratio]{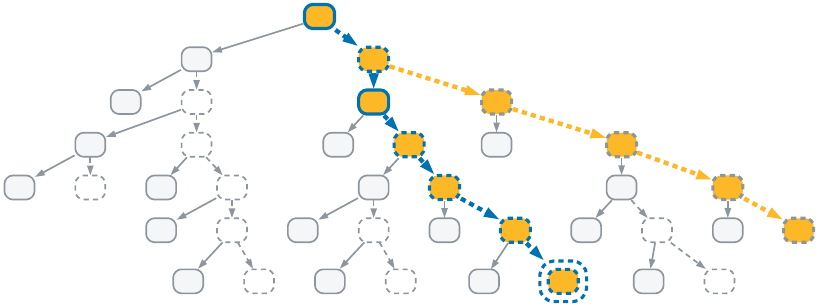}}
\par\smallskip
{\footnotesize (\arabic{prefetchpanel}) MH \(15\rightarrow21\)\par}
\end{minipage}%
\hfill
\begin{minipage}[b]{0.24\linewidth}
\centering
\refstepcounter{prefetchpanel}
\parbox[b][0.1\textheight][b]{\linewidth}{\centering\includegraphics[width=\linewidth,height=0.1\textheight,keepaspectratio]{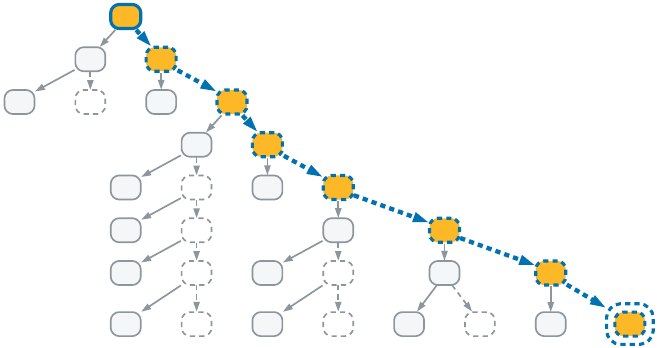}}
\par\smallskip
{\footnotesize (\arabic{prefetchpanel}) MH \(21\rightarrow28\)\par}
\end{minipage}%
\hfill
\begin{minipage}[b]{0.24\linewidth}
\centering
\refstepcounter{prefetchpanel}
\parbox[b][0.1\textheight][b]{\linewidth}{\centering\includegraphics[width=\linewidth,height=0.1\textheight,keepaspectratio]{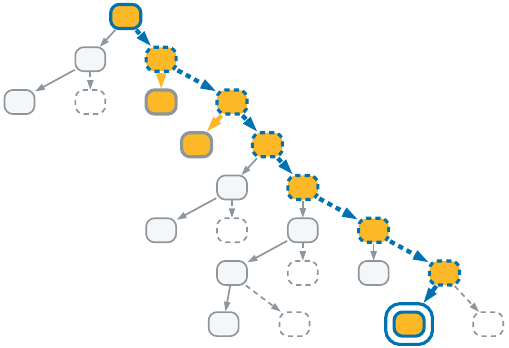}}
\par\smallskip
{\footnotesize (\arabic{prefetchpanel}) MH \(28\rightarrow35\)\par}
\end{minipage}%
\hfill
\begin{minipage}[b]{0.24\linewidth}
\centering
\refstepcounter{prefetchpanel}
\parbox[b][0.1\textheight][b]{\linewidth}{\centering\includegraphics[width=\linewidth,height=0.1\textheight,keepaspectratio]{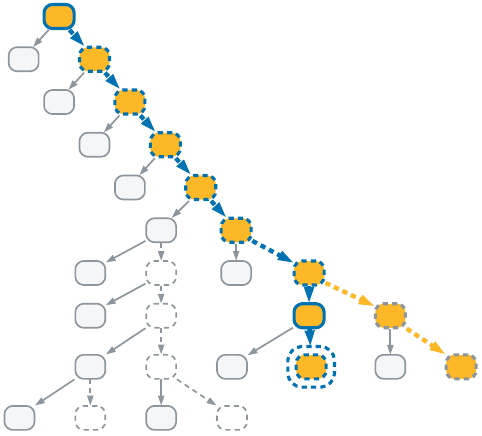}}
\par\smallskip
{\footnotesize (\arabic{prefetchpanel}) MH \(35\rightarrow43\)\par}
\end{minipage}%
\par
\vspace{6pt}
\begin{minipage}[b]{0.24\linewidth}
\centering
\refstepcounter{prefetchpanel}
\parbox[b][0.1\textheight][b]{\linewidth}{\centering\includegraphics[width=\linewidth,height=0.1\textheight,keepaspectratio]{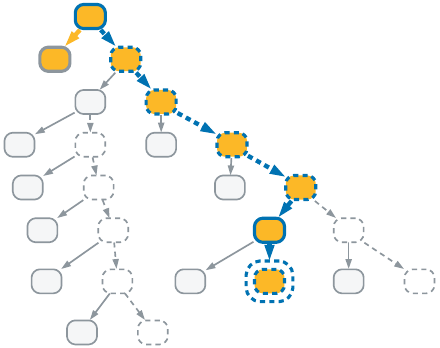}}
\par\smallskip
{\footnotesize (\arabic{prefetchpanel}) MH \(43\rightarrow49\)\par}
\end{minipage}%
\hfill
\begin{minipage}[b]{0.24\linewidth}
\centering
\refstepcounter{prefetchpanel}
\parbox[b][0.1\textheight][b]{\linewidth}{\centering\includegraphics[width=\linewidth,height=0.1\textheight,keepaspectratio]{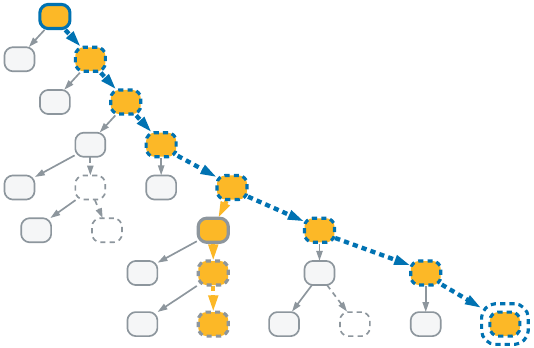}}
\par\smallskip
{\footnotesize (\arabic{prefetchpanel}) MH \(49\rightarrow56\)\par}
\end{minipage}%
\hfill
\begin{minipage}[b]{0.24\linewidth}
\centering
\refstepcounter{prefetchpanel}
\parbox[b][0.1\textheight][b]{\linewidth}{\centering\includegraphics[width=\linewidth,height=0.1\textheight,keepaspectratio]{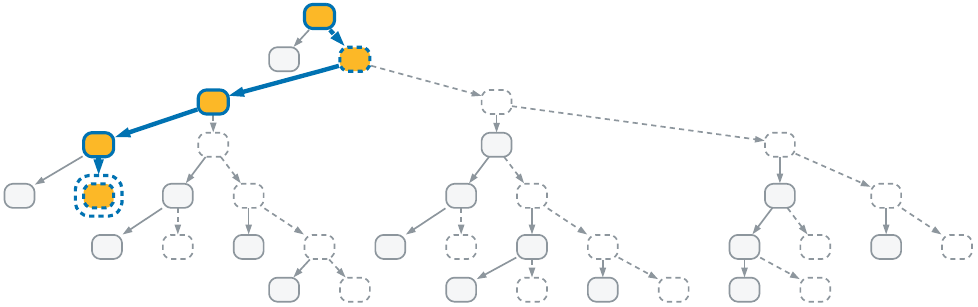}}
\par\smallskip
{\footnotesize (\arabic{prefetchpanel}) MH \(56\rightarrow60\)\par}
\end{minipage}%
\hfill
\begin{minipage}[b]{0.24\linewidth}
\centering
\refstepcounter{prefetchpanel}
\parbox[b][0.1\textheight][b]{\linewidth}{\centering\includegraphics[width=\linewidth,height=0.1\textheight,keepaspectratio]{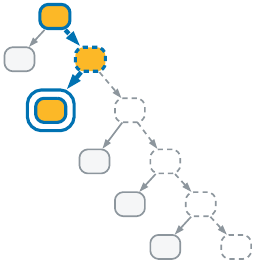}}
\par\smallskip
{\footnotesize (\arabic{prefetchpanel}) MH \(60\rightarrow62\)\par}
\end{minipage}%
\par
\vspace{6pt}
\begin{minipage}[b]{0.24\linewidth}
\centering
\refstepcounter{prefetchpanel}
\parbox[b][0.1\textheight][b]{\linewidth}{\centering\includegraphics[width=\linewidth,height=0.1\textheight,keepaspectratio]{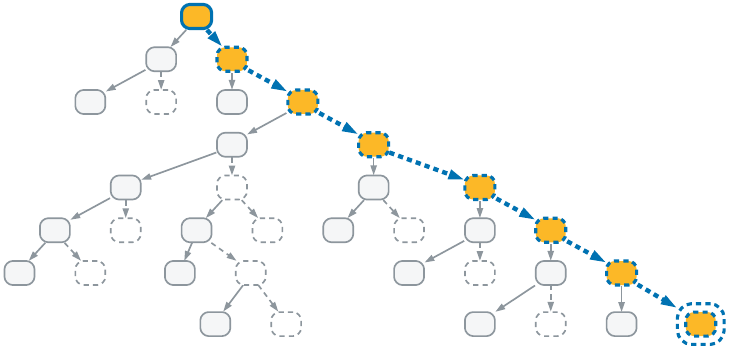}}
\par\smallskip
{\footnotesize (\arabic{prefetchpanel}) MH \(62\rightarrow69\)\par}
\end{minipage}%
\hfill
\begin{minipage}[b]{0.24\linewidth}
\centering
\refstepcounter{prefetchpanel}
\parbox[b][0.1\textheight][b]{\linewidth}{\centering\includegraphics[width=\linewidth,height=0.1\textheight,keepaspectratio]{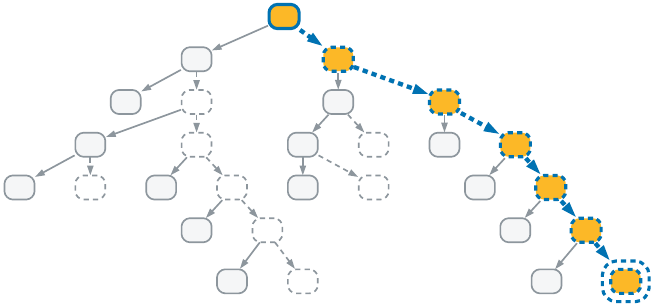}}
\par\smallskip
{\footnotesize (\arabic{prefetchpanel}) MH \(69\rightarrow75\)\par}
\end{minipage}%
\hfill
\begin{minipage}[b]{0.24\linewidth}
\centering
\refstepcounter{prefetchpanel}
\parbox[b][0.1\textheight][b]{\linewidth}{\centering\includegraphics[width=\linewidth,height=0.1\textheight,keepaspectratio]{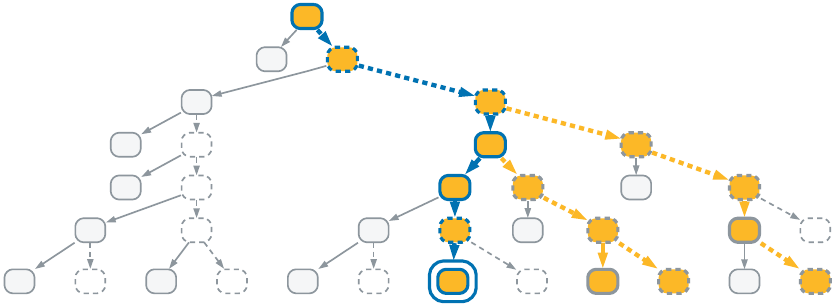}}
\par\smallskip
{\footnotesize (\arabic{prefetchpanel}) MH \(75\rightarrow81\)\par}
\end{minipage}%
\hfill
\begin{minipage}[b]{0.24\linewidth}
\centering
\refstepcounter{prefetchpanel}
\parbox[b][0.1\textheight][b]{\linewidth}{\centering\includegraphics[width=\linewidth,height=0.1\textheight,keepaspectratio]{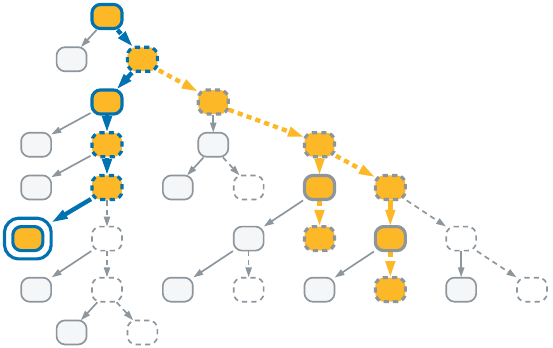}}
\par\smallskip
{\footnotesize (\arabic{prefetchpanel}) MH \(81\rightarrow86\)\par}
\end{minipage}%
\par
\vspace{6pt}
\begin{minipage}[b]{0.24\linewidth}
\centering
\refstepcounter{prefetchpanel}
\parbox[b][0.1\textheight][b]{\linewidth}{\centering\includegraphics[width=\linewidth,height=0.1\textheight,keepaspectratio]{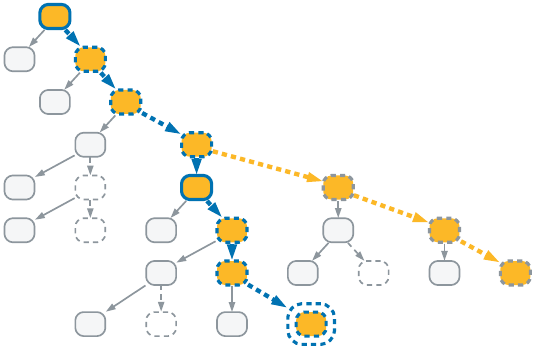}}
\par\smallskip
{\footnotesize (\arabic{prefetchpanel}) MH \(86\rightarrow93\)\par}
\end{minipage}%
\hfill
\begin{minipage}[b]{0.24\linewidth}
\centering
\refstepcounter{prefetchpanel}
\parbox[b][0.1\textheight][b]{\linewidth}{\centering\includegraphics[width=\linewidth,height=0.1\textheight,keepaspectratio]{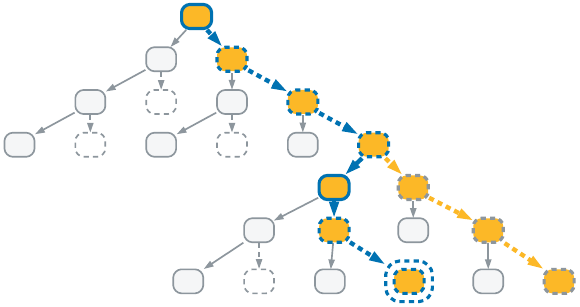}}
\par\smallskip
{\footnotesize (\arabic{prefetchpanel}) MH \(93\rightarrow99\)\par}
\end{minipage}%
\hfill
\begin{minipage}[b]{0.24\linewidth}
\centering
\refstepcounter{prefetchpanel}
\parbox[b][0.1\textheight][b]{\linewidth}{\centering\includegraphics[scale=0.5]{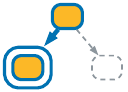}}
\par\smallskip
{\footnotesize (\arabic{prefetchpanel}) MH \(99\rightarrow100\)\par}
\end{minipage}%
\hfill
\begin{minipage}[b]{0.24\linewidth}\mbox{}\end{minipage}%
\par
\PrefetchFigureCaption{\textbf{Prefetched subtrees along one full MH trajectory.} The 19 subtrees of prompt 0 (0-based) in sampling order, covering its 100 MH transitions (LCB V6, Qwen3.5-9B, \method, BFS, \(N_{\mathrm{pf}}=20\)). Each label gives the MH steps the subtree advances.
Solid left edges: acceptance; dashed right edges: rejection; orange: \(\epsilon\)-rounded subtree (\(\epsilon=0.01\), Definition~\ref{def:finite-tree-continuation}); blue: sampled path (endpoint double-bordered); gray: other prefetched branches.}
\end{figure}
\clearpage

\end{document}